\documentclass{article} 
\usepackage{iclr2026_conference,times}

\usepackage{amsmath,amsfonts,bm}

\def\eqref#1{equation~\ref{#1}}

\def\1{\bm{1}}

\DeclareMathAlphabet{\mathsfit}{\encodingdefault}{\sfdefault}{m}{sl}
\SetMathAlphabet{\mathsfit}{bold}{\encodingdefault}{\sfdefault}{bx}{n}

\usepackage[toc,page,header]{appendix}
\usepackage{minitoc}

\usepackage{colortbl}
\usepackage{pifont}

\usepackage{caption}
\usepackage{subcaption}

\usepackage{titletoc}

\usepackage[pagebackref=true]{hyperref}
\definecolor{mydarkblue}{rgb}{0,0.08,0.45}
\hypersetup{colorlinks=true, linkcolor=mydarkblue, citecolor=mydarkblue, urlcolor=mydarkblue}
\renewcommand*{\backref}[1]{}  
\renewcommand*{\backrefalt}[4]{%
  \ifcase #1 %
  \or
    (Page #2)%
  \else
    (Pages #2)%
  \fi}
  
\usepackage{url}
\usepackage{graphicx}
\usepackage{amsmath, amssymb, amsthm}
\usepackage{mathtools, cases}

\usepackage{booktabs}
\usepackage{tcolorbox}
\usepackage{multirow}
\usepackage{algorithm}
\usepackage{algpseudocode}
\usepackage{subcaption}
\usepackage{wrapfig}
\usepackage{float}
\newtheorem{proposition}{Proposition}

\newtheorem{lemma}{Lemma}

\newtheorem{theorem}{Theorem}
\newtheorem{corollary}{Corollary}
\newtheorem{definition}{Definition}
\usepackage{orcidlink}
\usepackage[misc]{ifsym}

\definecolor{indigo}{rgb}{0.0, 0.25, 0.42}

\title{Learning Causal Normalizing Flows from Incomplete Data via Observed-Data Likelihood}

\author{%
Trung-Dung Hoang$^{1,2,4,5(\textrm{\Letter})}$, Alceu Bissoto$^{1,2,5}$, Tim Flühmann$^{1,2,4,5}$, David Herzig$^{2}$, \\
\textbf{Christos Nakas}$^{3,6}$, \textbf{Lia Bally}$^{2}$ \& \textbf{Lisa M. Koch}$^{1,2,5(\textrm{\Letter})}$ \\[0.3em]
\small $^{1}$Department of Digital Medicine, University of Bern, Switzerland \\
\small $^{2}$Department of Diabetes, Endocrinology, Nutritional Medicine and Metabolism (UDEM), \\
\small \phantom{$^{2}$}Inselspital, Bern University Hospital, University of Bern, Switzerland \\
\small $^{3}$Department of Clinical Chemistry, Inselspital, Bern University Hospital, University of Bern, Switzerland \\
\small $^{4}$Graduate School for Cellular and Biomedical Sciences (GCB), University of Bern, Switzerland \\
\small $^{5}$Diabetes Center Berne, Switzerland \\
\small $^{6}$Laboratory of Biometry, School of Agriculture, University of Thessaly, 384 46 Volos, Greece \\
\small \texttt{\{trung.hoang,lisa.koch\}@unibe.ch}
}

\iclrfinalcopy 
\begin{document}

\maketitle

\begin{abstract}

Causal Normalizing Flows (CNFs) enable causal inference from observational data given the causal structure, but they assume fully observed training data. We introduce MissCNF, which trains CNFs directly on incomplete data by maximizing the marginal likelihood of each partially observed sample, without discarding rows or constructing a completed dataset. Thanks to the causal structure encoded in the autoregressive factorization of CNFs, only missing variables in the ancestral closure of the observed set are integrated out, while the others are dropped without computation. We further establish the conditions under which MissCNF recovers the true joint distribution, and introduce \emph{causal-family positivity}, where identification is possible even when no record in the dataset is ever complete. We compare MissCNF with two common strategies for handling missing data: listwise deletion and impute-then-fit pipelines. Across eight synthetic causal benchmarks, three missingness mechanisms, and missing rates up to $90\%$, MissCNF achieves the lowest KL divergence in 23 of 24 nonlinear MCAR and MAR settings and in all nonlinear MNAR settings, as well as the lowest counterfactual error in 20 of 24 settings. On linear SCMs, where linear imputation performs best, MissCNF ranks in the top two in 22 of 24 settings.

\end{abstract}

\section{Introduction}
\label{sec:intro}

Randomized controlled trials remain the reference standard for estimating causal effects, but they are expensive, slow, and in many settings impossible to run. This has motivated a large body of work on causal inference from observational data, including Deep Structural Causal Models \citep{pawlowski2020deep}, Variational Graph Autoencoders \citep{pablo2022vaca}, and Diffusion-based approaches \citep{chao2024modeling}. Causal Normalizing Flows (CNFs) \citep{javaloy2023causal} parameterize the structural equations with an invertible flow masked according to the causal graph. CNFs have two key advantages: they yield exact likelihoods rather than bounds, and their density factorizes along the graph, so a model fitted to observational data answers interventional and counterfactual queries directly.

However, like many machine learning methods, CNFs were developed assuming complete training data. This is a significant limitation in domains such as healthcare, where missing data is unavoidable. One option is to discard incomplete records, also known as \textit{listwise deletion}, which wastes data and, more importantly, introduces bias whenever the probability of being fully observed depends on the values themselves, which Missing At Random (MAR) does not rule out. The other common option is first to fit an imputation model, then fit the causal model on its output as if it were the true data \citep{parra2026jointtreatmenteffectestimation}. Treating imputed values as observed data discards the uncertainty in the imputation stage and lets any error there propagate into the causal modelling stage.

In this work, we introduce \textbf{MissCNF}, which trains CNFs directly on missing data, without discarding rows or constructing any completed dataset, by computing the marginal likelihood of an arbitrary observed subset of variables. The proposed method can be traced back to \textit{Full Information Maximum Likelihood} (FIML), originally developed for the multivariate normal distribution \citep{anderson1957maximum} and linear-Gaussian Structural Equation Models \citep{Lee_1986,Muthén_Kaplan_Hollis_1987}, in which the marginal likelihood is available in closed form. In our setting, the Structural Causal Models represented by CNFs are not necessarily linear and Gaussian, so no closed form is available. We instead compute the required marginal likelihood by a Monte Carlo estimator that exploits the Directed Acyclic Graph (DAG) factorization: it samples only missing variables in the ancestral closure of that observed subset, while variables outside the ancestral closure are marginalized without simulation.

We then study what can be learned from the resulting observed-data likelihood. Under Missing At Random (MAR), the missingness mechanism is ignorable for likelihood-based inference \citep{rubin1976inference}, but ignorability alone does not guarantee that the complete-data distribution is identified: if a set of variables is never jointly observed in some region of their support, the observed data carry no information about their joint distribution there.
Recovering the complete-data distribution therefore requires some form of support-overlap assumption relating what is missing to what is observed \citep{naf2026goodimputationmarmissingness}. Exploiting the causal structure a CNF encodes, we show that identifying the complete joint distribution requires only a \emph{causal-family positivity} condition that holds even if no row is ever fully observed: every variable must be jointly observed with its causal parents in some pattern, with positive probability for almost every value in the relevant support. We further show that MissCNF already recovers, as a byproduct, every conditional an imputation model would need to target.
Identification under MNAR requires additional assumptions and is left to future work, though our experiments on eight synthetic causal graphs, across three missingness mechanisms and multiple missingness rates, show MissCNF still performs well in different regimes, including MNAR.

Our contributions are as follows: 
(1) We derive a marginal-likelihood objective that trains CNFs directly from incomplete data by integrating only over missing ancestral variables (Section~\ref{sec:misscnf});
(2) Under MAR, we separate likelihood ignorability from identification and establish a \emph{causal-family positivity} condition that can be strictly weaker than \emph{full-pattern positivity} and is enough to recover the full joint even when no row is fully observed (Section~\ref{sec:mar-consistent});
(3) We characterize the population target of impute-then-fit pipelines and show that MissCNF recovers every such conditional as a byproduct (Section~\ref{sec:imputation-based});
(4) We empirically validate the method on eight synthetic SCMs, showing consistent improvements in density estimation, counterfactual prediction, and average treatment effect estimation over baselines across missingness mechanisms and rates (Section~\ref{sec:experiments}). 

\section{Background}
\label{sec:background}

\paragraph{Structural causal models.}
A \textit{Structural Causal Model} (SCM) \citep{Pearl_2009} $\mathcal{M} = (\mathbf{f}, P_\mathbf{u})$ describes a data-generating process that transforms $d$ mutually independent exogenous noise variables $\mathbf{u} = \{u_i\}_{i=1}^d \sim P_\mathbf{u}$ into $d$ observed endogenous variables $\mathbf{x}$ through structural equations $\mathbf f = \{f_i\}_{i=1}^d$:
$$
x_i = f_i\big(\mathbf{x}_{\mathrm{pa}(i)}, u_i\big), \qquad i = 1,\dots,d,
$$
where $\mathrm{pa}(i)$ are the parents of $i$. Each SCM induces a causal graph $G=(V,E)$ on $V=\{1,\dots,d\}$, with $(j,i)\in E$ if $j \in \mathrm{pa}(i)$; we fix a topological order $\prec$ compatible with $G$ throughout, so $j\in\mathrm{pa}(i)\Rightarrow j\prec i$. 

\paragraph{Causal normalizing flows.} 

A \textit{Causal Normalizing Flow} (CNF) \citep{javaloy2023causal} instantiates an SCM $\mathcal{M}_\theta = (T_\theta, P_\mathbf{u})$ by representing the structural equations implicitly through a masked autoregressive normalizing flow $T_\theta$ \citep{germain2015made,papamakarios2017masked}. A CNF takes the causal graph $G$ as given and assumes \textit{causal sufficiency} \citep{spirtesgs93} such that there is no unobserved confounding. We inherit these assumptions unchanged in this work. We adopt the abductive direction with a single layer as it is causally consistent in both directions between $\mathbf{u}$ and $\mathbf{x}$, without requiring additional layers or regularization \citep{javaloy2023causal}. 
Let $h_i$ be a conditioner masked according to $G$, taking only the parents of $i$ as input, and $\tau(\cdot\,; \eta)$ be a transformer strictly monotonic and continuously differentiable in its first argument for every conditioner output $\eta$.
Each coordinate is then given by
$$
u_i = \tau\big(x_i\,;\,h_i(\mathbf{x}_{\mathrm{pa}(i)})\big), \qquad x_i = \tau^{-1}\big(u_i\,;\,h_i(\mathbf{x}_{\mathrm{pa}(i)})\big), \qquad i=1,\dots,d,
$$

Since the $u_i$ are mutually independent, $P_{\mathbf u}$ factorizes as $\prod_i p(u_i)$, and the resulting density $p_\theta(\mathbf x) \equiv p_\theta(x_1,\dots,x_d)$ factorizes according to $G$:
$$
p_\theta(x_1,\dots,x_d)=\prod_{i=1}^d p_\theta(x_i\mid \mathbf{x}_{\mathrm{pa}(i)}),\qquad
p_\theta(x_i\mid \mathbf{x}_{\mathrm{pa}(i)})=p(u_i)\left|\frac{\partial \tau}{\partial x_i}(x_i;h_i(\mathbf{x}_{\mathrm{pa}(i)}))\right|.
$$

For any index subset $O \subseteq V$ we write $p_{\theta,O}(\mathbf{x}_O)$ for the corresponding marginal of $p_\theta$.
\citet{javaloy2023causal} show that, given only the causal ordering, the family of CNFs is identifiable: any two CNFs matching the true observational distribution recover the same exogenous variables up to an invertible, component-wise transformation of variables $\mathbf{u}$. Consequently, fitting $T_\theta$ to observational data by maximum likelihood, $\max_\theta \log p_\theta(\mathbf{x})$, recovers the true SCM in this sense. This objective, however, presumes each $\mathbf{x}$ is fully observed, which might not always happen in many applications.

\paragraph{Missing mechanisms.} Let $\mathcal{R} \in \{0,1\}^d$ be a random missingness mask and $\mathbf{X} \in \mathbb{R}^d$ be a random sample, with $\mathcal{R}_i=1$ indicating that variable $X_i$ is observed. 
For a realization $R$ of $\mathcal{R}$, we write $O(R):= \{i: R_i = 1\}$ and $M(R):= V \setminus O(R)$ as the observed and missing index sets. We drop $R$ when it is clear from context, so that $\mathbf{x}_O$ and $\mathbf{x}_M$ denote the observed and missing coordinates of a realization $\mathbf{x}$. Because $R$, $O$, and $M$ map bijectively to one another, we use them interchangeably and refer to this shared representation as a \textbf{\textit{pattern}}.
The missingness mechanism is the conditional distribution $p_\mathrm{miss}(R \mid \mathbf{x})$, which can be one of the following \citep{rubin1976inference}:
\begin{itemize}
    \item \textit{Missing Completely At Random} (MCAR): $p_\mathrm{miss}(R \mid \mathbf{x}) = p_\mathrm{miss}(R)$, i.e., the mask is independent of $\mathbf{x}$ entirely.
    \item \textit{Missing At Random} (MAR): $p_\mathrm{miss}(R \mid \mathbf{x}) = p_\mathrm{miss}(R \mid \mathbf{x}_{O})$, i.e., the mask may depend only on the observed values. MCAR is the special case of MAR.
    \item \textit{Missing Not At Random} (MNAR): $p_\mathrm{miss}(R\mid \mathbf{x})$ depends on $\mathbf{x}_{M}$ as well.
\end{itemize}
We use $p_\theta(\mathbf{x})$ and $p^\ast(\mathbf{x})$ to denote the full joint densities of the fitted CNF and the true SCM, respectively, while $p_{\theta,O}(\mathbf{x}_O)$ and $p^\ast_O(\mathbf{x}_O)$ denote their marginals over a subset $O \subseteq V$. Together $\mathcal{M}^\ast$ and $p_\mathrm{miss}$ induce a joint density over $(\mathbf{X},\mathcal{R})$ via $p_{\mathrm{joint}}(\mathbf{x},R) := p^\ast(\mathbf{x})\,p_\mathrm{miss}(R\mid \mathbf{x})$.
Because complete rows are generally not observed, the distribution that a sample is actually drawn from is not $p_{\mathrm{joint}}$ but its marginal over the unobserved coordinates. The \emph{observed-data} density is thus $p_{\mathrm{obs}}(\mathbf{x}_O,R) := \int p_{\mathrm{joint}}(\mathbf{x}_O,\mathbf{x}_M,R)\,d\mathbf{x}_{M}$. 

\section{Method}
\label{sec:method}

We aim to fit a CNF $p_\theta(\mathbf{x})$ over all variables from partially observed data. We first introduce MissCNF and derive its training objective, then study the corresponding population objective under different missingness assumptions and establish conditions for recovering the complete-data distribution. Finally, we compare MissCNF with alternative approaches for learning from incomplete data.

\subsection{MissCNF}
\label{sec:misscnf}

\begin{tcolorbox}
\begin{proposition}
\label{prop:marginal-mc}
The computation of $p_{\theta,O}(\mathbf{x}_O)$ involves only the variables in the ancestral closure $A = \mathrm{anc}(O)$; equivalently, only the missing ancestral coordinates $M'=A\setminus O$ need to be integrated out:
$$p_{\theta,O}(\mathbf{x}_O) = \mathbb E_{\mathbf{u}_{M'} \sim \prod_{i\in M'} p(u_i)}\big[g(\mathbf{u}_{M'})\big],
\qquad
g(\mathbf{u}_{M'}; \mathbf{x}_O) := \prod_{i\in O} p(u_i)\left|\frac{\partial\tau}{\partial x_i}\right|,
$$
where, for each $i \in A$ processed in topological order, $u_i = \tau(x_i; h_i(\mathbf{x}_{\mathrm{pa}(i)}))$ if $i \in O$, and $x_i = \tau^{-1}(u_i; h_i(\mathbf{x}_{\mathrm{pa}(i)}))$ if $i \in M'$. (Proof in Appendix~\ref{app:proofs-method}.)
\end{proposition}
\end{tcolorbox} 

\newpage

\begin{wrapfigure}{r}{0.55\textwidth}
    \centering
    \includegraphics[width=0.55\textwidth]{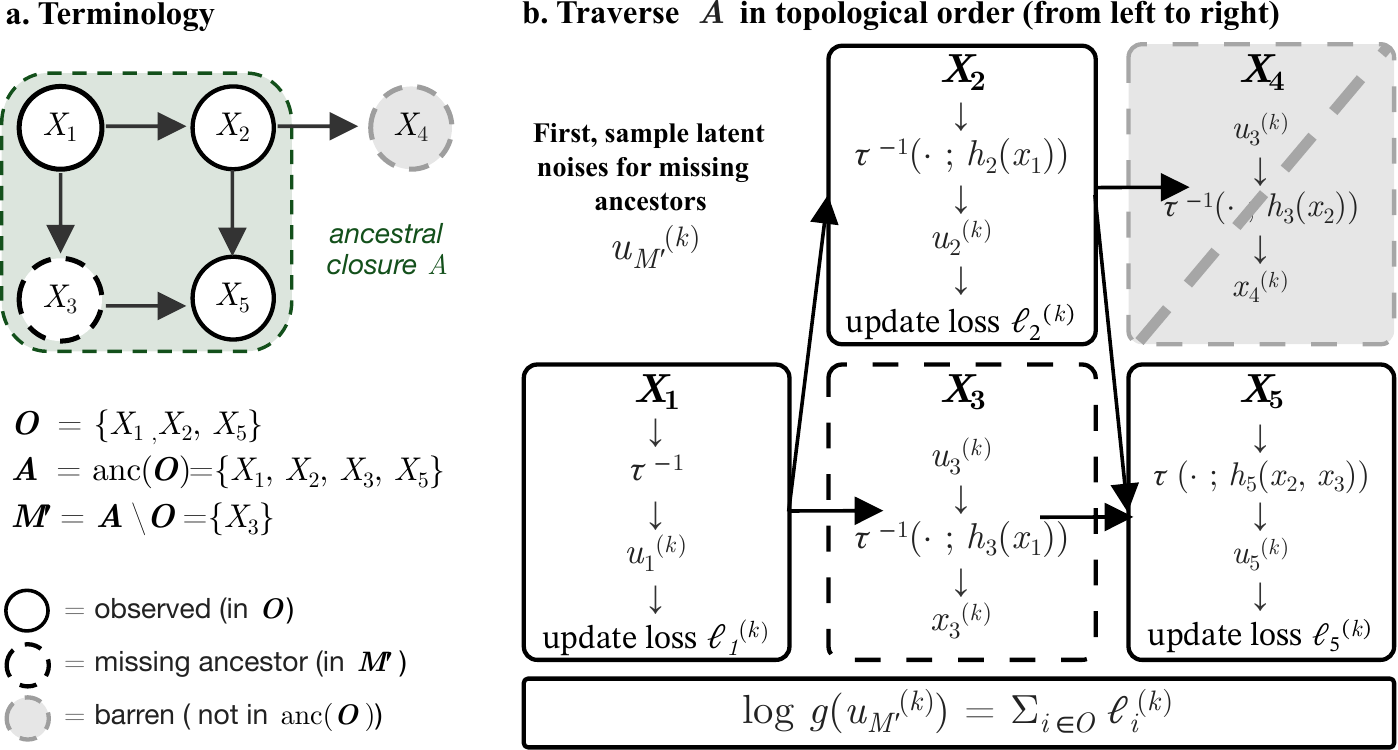}
    \caption{
(a) With $O=\{1,2,5\}$ observed, $A$ is $\{1,2,3,5\}$. $X_3$ is a missing ancestor and must be integrated out. $X_4$ is barren and is dropped without integration.
(b) One MC sample of Algorithm~\ref{alg:marginal}. Nodes in $A$ are processed in topological order. Observed nodes are mapped forward to their noise and added to the log-likelihood. The missing $X_3$ is generated from the noise $u_3^{(k)}$ and passed on to its children. $X_4$ is skipped.}
    \label{fig:misscnf-overview}
\end{wrapfigure}

For an observed/missing split $O\sqcup M=V$ of a given sample, the full-data likelihood $p_\theta(\mathbf{x})$ cannot be evaluated because $\mathbf{x}_M$ is unobserved. The sample therefore contributes the marginal likelihood
$$
p_{\theta,O}(\mathbf{x}_O)=\int p_\theta(x_1,\dots,x_d)\, d\mathbf{x}_M .
$$
Let $A:= \mathrm{anc}(O)$ be $O$ with all its ancestors, and $M':=A\setminus O$ the missing nodes inside that set (Figure~\ref{fig:misscnf-overview}a). Any node outside $A$ is a descendant-or-unrelated node with respect to $O$ (\textit{barren}). Barren nodes vanish from the calculation, since integrating a barren node out of the joint density yields 1. The only remaining integral is over $M'$, the missing ancestors. Rather than integrating over $\mathbf{x}_{M'}$ directly, we integrate over its noise $\mathbf{u}_{M'}$. This is a reparameterization: the noise terms are independent draws from the base distribution, which is easy to sample, and each $x_i$ for $i$ in $M'$ is obtained by pushing $u_i$ through the flow given its (already computed) parents (Figure~\ref{fig:misscnf-overview}b). We formally state this in Proposition~\ref{prop:marginal-mc}.

\begin{algorithm}[htbp]
\small
\caption{Computing $\hat p_{\theta,O}(\mathbf{x}_O)$ for a masked autoregressive CNF}
\label{alg:marginal}
\begin{algorithmic}[1]
\Require Observed values $\mathbf{x}_O$; ancestral set $A = \mathrm{anc}(O)$, missing-within-$A$ set $M' = A \setminus O$; topological order $\prec$ restricted to $A$; number of samples $K$
\State Draw $\mathbf{u}_{M'}^{(k)} \sim \prod_{i\in M'} p(u_i)$ i.i.d.\ for $k = 1,\dots,K$
\For{$k = 1, \dots, K$}
    \For{$i \in A$ in topological order $\prec$} \Comment{parents of $i$ already resolved}
        \If{$i \in O$}
            \State $u_i \gets \tau(x_i\,;\,h_i(\mathbf{x}_{\mathrm{pa}(i)}))$ \Comment{observed: evaluate forward}
            \State $\ell_i^{(k)} \gets \log p(u_i) + \log\left|\partial \tau / \partial x_i\right|$
        \Else \Comment{$i \in M'$: missing, already sampled as $u_i = u_i^{(k)}$}
            \State $x_i \gets \tau^{-1}(u_i^{(k)}\,;\,h_i(\mathbf{x}_{\mathrm{pa}(i)}))$ \Comment{invert to propagate to children}
        \EndIf
    \EndFor
    \State $\log g(\mathbf{u}_{M'}^{(k)}; \mathbf{x}_O) \gets \sum_{i \in O} \ell_i^{(k)}$
\EndFor
\State \Return $\log \hat p_{\theta,O}(\mathbf{x}_O) \gets \operatorname*{logsumexp}_{n=1}^K \log g(\mathbf{u}_{M'}^{(k)}) - \log K$
\end{algorithmic}
\end{algorithm}

We estimate the expectation by Monte Carlo (MC) estimation. Algorithm~\ref{alg:marginal} gives the full procedure: draw $K$ samples of $\mathbf{u}_{M'}$ from the base distribution, propagate each through the flow in topological order to obtain $g$, and average in log-space. If no missing node is an ancestor of an observed node, $M'$ is empty, and $K = 0$ samples are needed.
As we use an abductive autoregressive CNF with an $L=1$ layer, as in ~\cite{javaloy2023causal}, the full-density evaluation costs $\mathcal{O}(L)$, while sampling requires $\mathcal{O}(dL)$ due to autoregressive inversion \citep{javaloy2023causal,papamakarios2021normalizingflowsprobabilisticmodeling}. Computing the marginal density in the algorithm requires $K$ MC samples; therefore, it gives $\mathcal{O}(KdL)$ total computational work.

Given this recipe for computing $\hat p_{\theta,O}(\mathbf{x}_O)$ for any observed subset $O$, we can now train a CNF directly on incomplete data by maximizing the likelihood it assigns to whatever each row actually shows. Given a training set of $n$ i.i.d.\ samples $D^{(j)} = (\mathbf{x}^{(j)}_{O_j}, O_j)$, we define the training objective as
$
\mathcal{L}_{n,K}(\theta) := -\frac{1}{n}\sum_{i=1}^n \log \hat p_{\theta,O_j}\big(\mathbf{x}^{(j)}_{O_j}\big).
$
As both the number of MC samples $K$ in Algorithm~\ref{alg:marginal} and the training set size $n$ grow, $\hat p_{\theta,O} \to p_{\theta,O}$ and the empirical loss converges to the population risk

\[
\mathcal{L}(\theta) = \mathbb{E}_{(\mathbf{X}_{O(\mathcal{R})},\mathcal{R})\sim p_{\mathrm{obs}}} \big[ -\log p_{\theta,O(\mathcal{R})}(\mathbf{X}_{O(\mathcal{R})}) \big] ~.
\]
The objective uses only the observed entries of each row, so it can be applied to any missingness pattern. Whether its minimizer recovers $p^\ast$, however, depends on the missingness mechanism, which we study in Section~\ref{sec:mar-consistent}.

\subsection{MAR: Ignorability and Identification}
\label{sec:mar-consistent}

The population risk in Section~\ref{sec:misscnf} uses only the model marginal $p_{\theta,O}(\mathbf{x}_O)$, rather than explicitly modeling the missingness pattern $R$. We first show that, under MAR, this is equivalent to optimizing the joint likelihood of the observed values and missingness pattern, up to a term independent of $\theta$. 

For any candidate complete-data model $p_\theta$, let $q_\theta$ denote the distribution over $(\mathbf{X}_{O(\mathcal{R})},\mathcal{R})$ obtained by applying the fixed missingness mechanism $p_\mathrm{miss}$. Lemma~\ref{lem:mar-pullout} shows that MAR factorizes this observed-data density into the model marginal $p_{\theta, O}(\mathbf{x}_O)$ and a missingness term independent of $\theta$. This is the classical likelihood-ignorability property under MAR \citep{rubin1976inference}.

\begin{tcolorbox}
\begin{lemma}[MAR pull-out identity, for any candidate $\theta$]
\label{lem:mar-pullout}
Fix a pattern $O \subseteq V$ and let $R = \mathbf 1_O$. For \emph{any} $\theta$, define
$$
q_\theta(\mathbf{x}_O, R) := \int p_\theta(\mathbf{x}_O,\mathbf{x}_{M})\, p_\mathrm{miss}(R\mid \mathbf{x}_O, \mathbf{x}_{M})\, d\mathbf{x}_M,
$$
the observed-data density that would arise if $\mathbf{x}$ were generated by the candidate model $p_\theta$ while $R$ is generated by the true, fixed mechanism $p_\mathrm{miss}$. Under MAR, $q_\theta$ simplifies for \emph{every} $\theta$:
$$
q_\theta(\mathbf{x}_O,R) = p_{\theta,O}(\mathbf{x}_O)\, p_\mathrm{miss}(R \mid \mathbf{x}_O).
$$
In particular, if the model is realizable ($p_{\theta_0} = p^\ast$ for some $\theta_0$), then $p_{obs}(\mathbf{x}_O,R) = q_{\theta_0}(\mathbf{x}_O,R) = p^\ast_O(\mathbf{x}_O)\,p_\mathrm{miss}(R\mid \mathbf{x}_O)$. (Proof in Appendix~\ref{app:proofs-convergence}.)
\end{lemma}
\end{tcolorbox}
Since $\log q_{\theta}(\mathbf{x}_O,R) = \log p_{\theta,O}(\mathbf{x}_O) + \log p_\mathrm{miss}(R\mid \mathbf{x}_O)$ and the second term does not depend on $\theta$, minimizing $\mathcal{L}$ is equivalent to maximum likelihood on the observed values and patterns,
$$
\arg\min_\theta \mathcal L(\theta) = \arg\min_\theta \mathbb E_{p_{\mathrm{obs}}}\big[-\log q_{\theta}(\mathbf{X}_{O(\mathcal{R})},\mathcal{R})\big],
$$
so $p_\mathrm{miss}$ can be dropped from the objective without changing the minimizer.

\begin{tcolorbox}
\begin{theorem}[MAR identification of the population risk, on the identified region]
\label{thm:mar-consistency}
Suppose MAR holds, and the CNF family is realizable ($p_{\theta_0} = p^\ast$ for some $\theta_0$). Define $Z_O := \{\mathbf{x}_O : p_\mathrm{miss}(\mathbf 1_O\mid \mathbf{x}_O) = 0\}$ the set of values of $\mathbf{x}_O$ that pattern $O$ never reveals. Then $\theta_0$ is a global minimizer of $\mathcal L(\theta)$, and \emph{every} global minimizer $\theta^\ast$ satisfies
$$
p_{\theta^\ast,O}(\mathbf{x}_O) = p^\ast_O(\mathbf{x}_O) \quad \text{for $p^\ast_O$-almost every\ } \mathbf{x}_O \notin Z_O,\quad \text{simultaneously for every } O.
$$
(Proof in Appendix~\ref{app:proofs-convergence}.)
\end{theorem}
\end{tcolorbox}

Theorem~\ref{thm:mar-consistency} says that the trained model reproduces the true distribution of each observed pattern, but only where that pattern actually reveals data. It does not tell us what the model does on values a pattern never shows. 
To allow the model to recover the true distribution, we need one of the following conditions to hold.

\begin{definition}[Positivity conditions]
\label{def:positivity}
We say \emph{full-pattern positivity} holds if $p^\ast(Z_V)=0$, i.e.\ some rows are fully observed and $p_\mathrm{miss}(\mathbf 1_V\mid \mathbf{X}=\mathbf{x})>0$ for $p^\ast$-almost every $\mathbf{x}$. We say \emph{causal-family positivity} holds if, for every node $i\in V$, there exists a pattern $O$ with \emph{causal-family} of $ i = \{i\}\cup\mathrm{pa}(i)\subseteq O$, and $p^\ast_{O}(Z_{O})=0$, i.e.\ $p_\mathrm{miss}(\mathbf 1_{O}\mid \mathbf{X}_{O}=\mathbf{x}_{O})>0$ for $p^\ast$-almost every $\mathbf{x}_{O}$ (as formalized in Corollary~\ref{cor:coverage}).
\end{definition}
\newpage
\begin{wrapfigure}{r}{0.45\textwidth}
    \centering

    \includegraphics[width=0.45\textwidth]{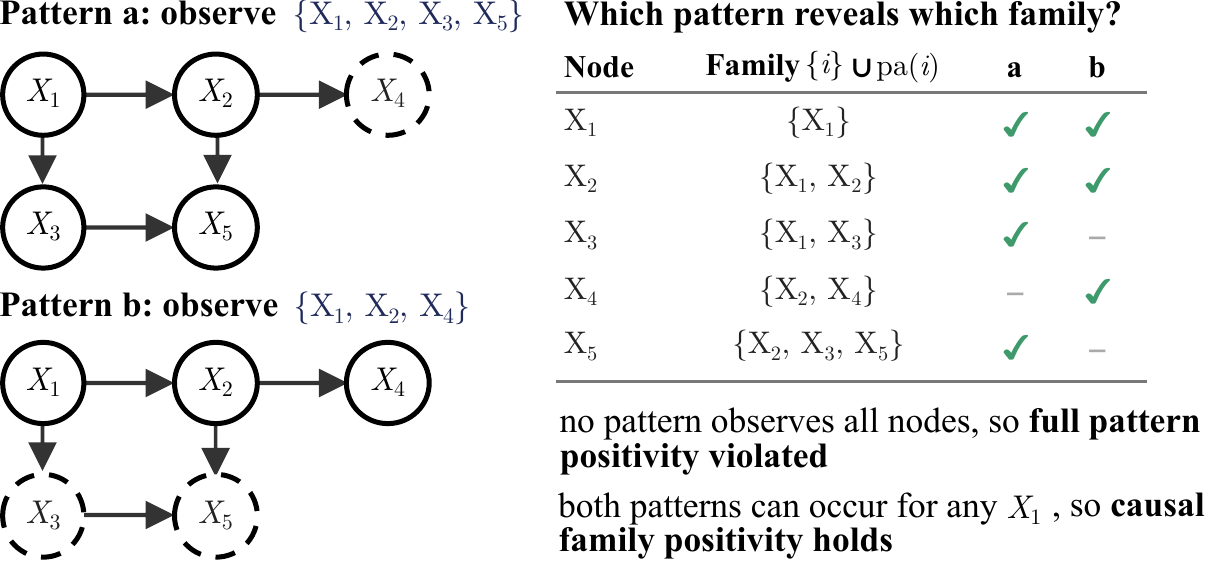}
    \caption{\textbf{Left:} the two observed patterns. \textbf{Right:} the causal family $\{i\}\cup\mathrm{pa}(i)$ of each node and the patterns that contain it. No pattern observes all of $V$, but every family is observed by at least one pattern.}
    \label{fig:positivity}
\end{wrapfigure}
\paragraph{Two routes to full joint recovery.} Theorem~\ref{thm:mar-consistency} identifies each marginal $p_{\theta,O}(\mathbf{x}_O)$ outside its unidentified region $Z_O$, pattern by pattern.
Taking $O=V$ gives the direct route to obtain $p_{\theta^\ast} = p^\ast$, full-pattern positivity: every $\mathbf{x}$ has some chance of being fully observed. The second route exploits the CNF's causal structure. Because the same model marginal over any subset is shared across every observation pattern containing that subset, a causal family may be identified through one pattern even if another pattern provides no support there. Corollary~\ref{cor:coverage} shows that positivity need only be checked locally, once per node against its own causal family, rather than globally against the full vector $\mathbf{x}$. 
Figure~\ref{fig:positivity} gives an example. Only two patterns occur: pattern~a observes $\{X_1,X_2,X_3,X_5\}$ and pattern~b observes $\{X_1,X_2,X_4\}$. No row is ever complete, so full-pattern positivity is violated. Suppose $X_1$ is always observed and the pattern is chosen with probability $\mathrm{sigmoid}(x_1)$ for a and $1-\mathrm{sigmoid}(x_1)$ for b. The mechanism then depends only on $X_1$, which both patterns observe, so it is MAR. Since both probabilities are positive for every $x_1$, $Z_{O_a}=Z_{O_b}=\emptyset$. Each causal family is contained in at least one pattern (right panel), so causal-family positivity holds, and the joint is still recovered. 

\begin{tcolorbox}
\begin{corollary}[Joint recovery via causal-family positivity, without requiring full-pattern positivity in general]
\label{cor:coverage}
Suppose $p_{\theta^\ast}$ and $p^\ast$ both factorize according to $G$ (Section~\ref{sec:background}). Under causal-family positivity (Definition~\ref{def:positivity}) and $p^\ast_{\mathrm{pa}(i)}(\mathbf{x}_{\mathrm{pa}(i)})>0$ on the relevant domain for each $i$, every global minimizer $\theta^\ast$ of $\mathcal L$ satisfies $p_{\theta^\ast}(\mathbf{x}) = p^\ast(\mathbf{x})$ for $p^\ast$-almost every\ $\mathbf{x}$. (Proof in Appendix~\ref{app:proofs-convergence}.)
\end{corollary}
\end{tcolorbox}

\paragraph{Missing Not At Random (MNAR).}
\label{sec:mnar-discussion}

Under MNAR, the missingness probability may depend on values that are themselves unobserved, so the MAR factorization used in Theorem~\ref{thm:mar-consistency} no longer holds, and the distribution of the observed entries alone does not generally identify $p^\ast(\mathbf{x})$. 
Nevertheless, our estimator remains well-defined under MNAR and continues to optimize its marginal-likelihood objective based on $p_{\theta, O}(\mathbf{x}_O)$, without constructing any completed dataset. This does not remove MNAR-induced selection bias, but it avoids compounding it with a second, imputation-based stage. We evaluate the method empirically under MNAR in Section~\ref{sec:mnar}, while leaving formal identification under additional MNAR assumptions to future work.

\subsection{Comparison to Alternative Estimators}
\label{sec:comparison}

\paragraph{Listwise deletion.}
\label{sec:listwise-deletion}

Listwise deletion discards every row with at least one missing coordinate and fits the model only on complete cases. It therefore targets
$$
p^\ast(\mathbf{x} \mid \mathcal{R} = \mathbf 1_V) \propto p^\ast(\mathbf{x})\,p_\mathrm{miss}(\mathbf 1_V\mid \mathbf{x})
$$
rather than \(p^\ast(\mathbf{x})\) itself. Under MCAR, $p_\mathrm{miss}(\mathbf 1_V\mid \mathbf{x})$ is constant, and the two coincide; under general MAR, the complete-case probability may vary with $\mathbf{x}$, so the complete-case distribution is generally reweighted. Moreover, listwise deletion requires complete samples, whereas Corollary~\ref{cor:coverage} shows that MissCNF can recover the full joint even when no row is fully observed under the causal-family positivity  (Corollary~\ref{cor:coverage}).

\paragraph{Imputation-then-fit Approaches}
\label{sec:imputation-based}
A common alternative is impute-then-fit: an imputation model $q(\mathbf{x}_M\mid \mathbf{x}_O, R)$ first completes the missing coordinates, after which a downstream model is fit to the completed data. The completed rows follow
$$
p_{\mathrm{comp}}(\mathbf{x}) = \sum_R p_\mathrm{obs}(\mathbf{x}_{O}, R)\, q\!\left(\mathbf{x}_{M}\mid \mathbf{x}_{O},R\right),
$$
so maximum likelihood on the completed data targets $p_{\mathrm{comp}}$, not directly $p^\ast$. Under MAR, the oracle conditional $q(\mathbf{x}_M\mid \mathbf{x}_O, R)=p^\ast(\mathbf{x}_M \mid \mathbf{x}_O)$ is sufficient for $p_{\mathrm{comp}}=p^\ast$; otherwise, the downstream target generally inherits errors from the imputation model. 

MissCNF avoids the intermediate model entirely: under causal-family positivity, it recovers the full joint and therefore every conditional required for imputation,
$
p_{\theta^\ast}(\mathbf{x}_B\mid \mathbf{x}_A) = p^\ast(\mathbf{x}_B\mid \mathbf{x}_A),
$
for any partition $V=A\sqcup B$, wherever the conditional is defined (Cororally~\ref{cor:conditional-recovery}, Appendix~\ref{app:proofs-imputation}).
Issues specific to single and multiple imputation, including uncertainty understatement, model pooling, and compatibility, are discussed in Appendix~\ref{app:discussion-imputation}.

\section{Experiments}
\label{sec:experiments}

\subsection{Setup}
\label{sec:exp_setup}
We evaluate on eight synthetic SCMs used in prior work  \citep{pablo2022vaca,javaloy2023causal,chao2024modeling}: Chain, Triangle, Collider, and Fork, each with a linear and a nonlinear version (Appendix~\ref{app:datasets}). Following \citet{javaloy2023causal}, we use 20,000 training, 2,500 validation, and 2,500 test samples for each dataset.

We consider three missingness mechanisms, \textbf{MCAR}, \textbf{MAR}, and \textbf{self-masking MNAR} based on previous work \citep{muzellec2020missing}. For MCAR and MAR, we use missingness rates of 30\%, 60\%, and 90\%. For MNAR, we use 30\% and 60\%. The missingness is only applied to the training and validation sets. The root nodes are never masked to make sure MAR always has observed conditioning variables. The exception is Experiment~3, where we use a MAR variant to test causal-family coverage. Under MCAR, each variable is masked independently. Under MAR and self-masking MNAR, the masking probability of a non-root node is the sigmoid of a linear function of its conditions, which are its root parents (or of all roots if it has none) for MAR and the variable's own value for self-masking MNAR. Full definitions of the missingness mechanisms are given in Appendix~\ref{app:missingness}.

We compare MissCNF against \textbf{listwise deletion} and impute-then-fit pipelines using five non-causal imputers: \textbf{Mean imputation}, \textbf{MissForest} \citep{stekhoven2012missforest}, \textbf{MICE} with single imputation \citep{vanbuuren2011mice}, and the more recent methods \textbf{KPI} \citep{wang2025kpi} and \textbf{DiffPuter} \citep{zhang2025diffputer}, which report state-of-the-art imputation accuracy against other deep imputers such as GAIN \citep{yoon2018gain}, MIWAE \citep{mattei2019miwae}, and ReMasker \citep{du2024remasker}. We also use one causal imputer, MIRACLE \citep{kyono2021miracle}. While MIRACLE learns the causal graph itself, we give it the true graph instead and call this variant \textbf{Oracle-graph MIRACLE (O-MIRACLE)}. We additionally train the same CNF on complete data as a reference. All methods use the same downstream causal normalizing flow and training budget; baseline and implementation details are provided in Appendices ~\ref{app:baselines} and ~\ref{app:implementation}.

Each experiment is repeated over five random seeds affecting the data, missingness mask, and model initialization. Following \citet{javaloy2023causal}, we evaluate on the complete test set using symmetric \textbf{KL} divergence for observational fit, \textbf{$\text{RMSE}_{\text{ATE}}$} for interventional accuracy, and \textbf{$\text{RMSE}_{\text{CF}}$} for counterfactual prediction. More details are given in Appendix~\ref{app:metrics}. 
We report KL as the main metric: under the identifiability conditions of \citet{javaloy2023causal}, a causal NF that matches the observational distribution also recovers the interventional and counterfactual distributions. Nevertheless, finite-sample differences in KL need not translate directly into lower ATE or CF error. We therefore report KL in the main text and provide the full results in Appendix~\ref{app:full_results}.
\vspace{-0.1cm}
\subsection{Main result: Performance under MCAR and MAR missingness}

\label{sec:exp1}
\begin{figure}[htbp]
    \centering
    \includegraphics[]{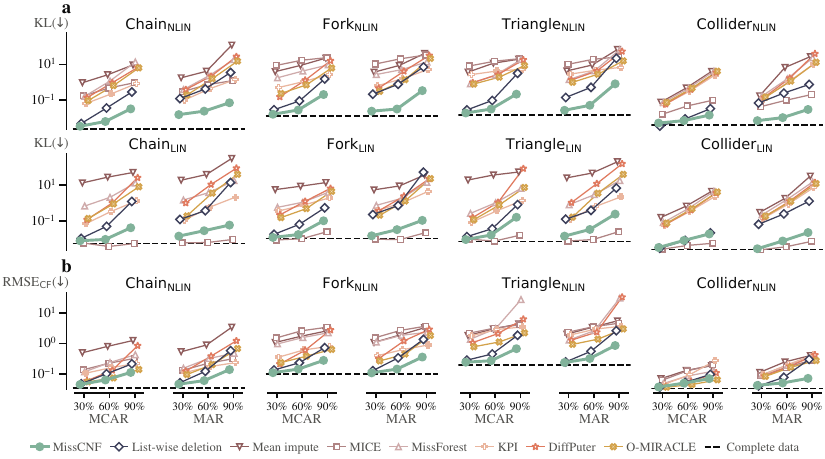}
    \caption{Results under MCAR and MAR at 30\%, 60\%, and 90\% missing rates (mean $\pm$ std over 5 seeds; lower is better). (a) Symmetric KL divergence to the ground-truth distribution on all SCMs. (b) $\text{RMSE}_{\text{CF}}$ on the nonlinear SCMs. The dashed line marks the model trained on complete data.}
    \vspace{-0.1cm}
    \label{fig:tier1_grid_kl_distance}
\end{figure}

We first analyse model identifiability under MCAR and MAR missingness and report the KL in Figure~\ref{fig:tier1_grid_kl_distance}a. 
For nonlinear SCMs, MissCNF attains the lowest symmetric KL divergence in 23 of 24 settings of tasks, mechanisms and missing rates.  In contrast, on linear SCMs, MICE achieves the lowest KL in 23 of 24 settings, whereas MissCNF remains competitive, ranking among the top two in 22 of 24 settings
This advantage is not limited to comparisons with non-causal imputers: even O-MIRACLE does not outperform MissCNF overall. At the same time, the strong performance of MICE on the linear benchmarks shows that impute-then-fit can be highly effective when the underlying mechanisms are simple. We investigate this dependence on nonlinearity more directly in Appendix~\ref{app:full_results_nonlinearity} and confirm that the performance of MICE decreased rapidly while MissCNF remains stable as the degree of nonlinearity increases. 
The results also illustrate the sensitivity of listwise deletion to the missingness mechanism. Across all tasks and missing rates, its average KL deteriorates by $14.0\times$ from MCAR to MAR, compared with a $2.0\times$ increase for MissCNF. This is consistent with complete-case analysis becoming biased when the probability that a row is fully observed depends on observed variables (Section~\ref{sec:listwise-deletion}).
Figure~\ref{fig:tier1_grid_kl_distance}b reports $\text{RMSE}_{\text{CF}}$ on the nonlinear SCMs. The results follow the KL results: MissCNF has the lowest error in 20 of 24 settings. Results on linear SCMs and for $\text{RMSE}_{\text{ATE}}$ are in Appendix~\ref{app:full_results_exp1}. They show the same pattern, but the gap between methods is smaller for the ATE.

\vspace{-0.1cm}
\subsection{Robustness to Weaker Recovery Conditions}
\label{sec:weaker-recovery}

In this section, we test the recovery conditions (full-pattern positivity, causal-family positivity) developed in Section~\ref{sec:mar-consistent}. 
Then we move to MNAR missingness, where the observed-data distribution no longer generally identifies the full-data distribution.

\vspace{-0.1cm}
\subsubsection{From Causal-Family Positivity to Non-Identification}
\label{sec:tier2-tier3}

In contrast to the previous experiment, where full-pattern positivity holds, we first violate it by retaining causal-family positivity in the absence of any fully observed samples. We use Chain and Fork SCMs with two observation patterns $O_1$ and $O_2$ in each that share an always-observed node $x_{\text{cut}}$ (Appendix~\ref{app:missingness_mar_variant}). Which pattern is observed depends only on $x_{\text{cut}}$, so the mechanism is MAR. Under \emph{Family positivity}, the pattern probability ($O_1$ vs. $O_2$) is a sigmoid of the standardized $x_{\text{cut}}$, so both patterns occur on the whole support, every causal family is covered, and $p^*$ is identifiable. Under \emph{Family positivity violated}, we additionally never observe $O_1$ when $x_{\text{cut}}$ is below its median $\tau$, so $p^*$ is not identifiable on $Z_{O_1} = \{\mathbf{x} : x_{\text{cut}} < \tau\}$. As a \emph{full positivity} reference, we use the previous MAR mechanism at a 50\% missing rate, which matches the number of observed entries. Figure~\ref{fig:tier2_tier3_kl} reports the KL; full results are in Appendix~\ref{app:full_results_tier2_vs_tier3}.

\begin{figure}[htbp]
    \centering
    \includegraphics[width=\linewidth]{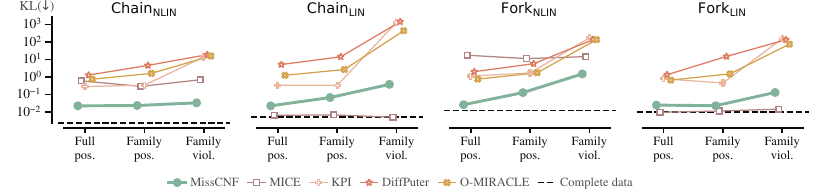}
   \caption{KL to the ground-truth distribution on Chain and Fork graphs (linear and nonlinear) under three regimes: \textit{Full positivity}, \textit{Family positivity}, and \textit{Family positivity violated}. The dashed line marks the model trained on complete data.
    }
    \vspace{-0.1cm}
    \label{fig:tier2_tier3_kl}
\end{figure}

Under \textit{Family positivity}, MissCNF matches the full-positivity reference on Chain$_{\text{NLIN}}$ (0.024 vs.\ 0.023) and Fork$_{\text{LIN}}$ (0.023 vs.\ 0.025). On Chain$_{\text{LIN}}$ and Fork$_{\text{NLIN}}$, its KL increases (0.068 vs.\ 0.023 and 0.13 vs.\ 0.027), but it stays below KPI, DiffPuter, and O-MIRACLE on every task. The gap can be explained by how much of the marginal likelihood requires Monte Carlo marginalization. In Chain, coverage holds requires $p(x_2, x_3)$, which integrates over $x_1$, on half of the data. The full-positivity reference requires $p(x_1, x_3)$, which integrates over $x_2$, only on a quarter of the data.

Under \textit{Family positivity violated}, the KL of MissCNF rises overall as expected. We further locate the error in Figure~\ref{fig:tier3_boundary} (Appendix~\ref{app:full_results_tier2_vs_tier3}) and confirm the error lies inside $Z_{O_1}$ and stays close to zero outside it, which matches our claim in Theorem~\ref{thm:mar-consistency}. 

\subsubsection{MNAR Missingness}
\label{sec:mnar}

\begin{table}[htbp]
\centering
\scriptsize
\setlength{\tabcolsep}{1.5pt}
\resizebox{\linewidth}{!}{%
\begin{tabular}{lcccccccc}
\toprule
Method & Chain$_{\text{LIN}}$ & Fork$_{\text{LIN}}$ & Collider$_{\text{LIN}}$ & Triangle$_{\text{LIN}}$ & Chain$_{\text{NLIN}}$ & Fork$_{\text{NLIN}}$ & Collider$_{\text{NLIN}}$ & Triangle$_{\text{NLIN}}$ \\
\midrule
MICE & \textbf{0.058}{\tiny$\pm$0.046} & \textbf{0.030}{\tiny$\pm$0.010} & \textbf{0.110}{\tiny$\pm$0.080} & \textbf{0.014}{\tiny$\pm$0.010} & 0.931{\tiny$\pm$0.279} & 19.116{\tiny$\pm$0.236} & \underline{0.251}{\tiny$\pm$0.146} & 16.060{\tiny$\pm$0.589} \\
KPI & 0.685{\tiny$\pm$0.220} & 0.997{\tiny$\pm$0.158} & 0.731{\tiny$\pm$0.349} & 0.650{\tiny$\pm$0.268} & \underline{0.687}{\tiny$\pm$0.301} & \underline{1.574}{\tiny$\pm$0.136} & 0.785{\tiny$\pm$0.399} & 5.596{\tiny$\pm$0.778} \\
O-MIRACLE & 3.808{\tiny$\pm$4.190} & 1.011{\tiny$\pm$0.386} & 0.956{\tiny$\pm$0.399} & 3.936{\tiny$\pm$4.829} & 2.749{\tiny$\pm$2.780} & 1.700{\tiny$\pm$0.423} & 0.921{\tiny$\pm$0.372} & \underline{3.839}{\tiny$\pm$2.393} \\
\textbf{MissCNF} & \underline{0.083}{\tiny$\pm$0.052} & \underline{0.048}{\tiny$\pm$0.025} & \underline{0.117}{\tiny$\pm$0.094} & \underline{0.031}{\tiny$\pm$0.016} & \textbf{0.121}{\tiny$\pm$0.042} & \textbf{0.066}{\tiny$\pm$0.030} & \textbf{0.177}{\tiny$\pm$0.127} & \textbf{0.288}{\tiny$\pm$0.280} \\
\bottomrule
\end{tabular}}
\caption{KL to the ground-truth distribution under self-masking MNAR at 60\% missingness. Mean $\pm$ std over 5 seeds. Lower is better. Bold = best, underline = second-best per task.}
\label{tab:mnar60_kl}
\end{table}
Finally, we consider self-masking MNAR missingness at missing rates of $30\%$ and $60\%$, where the probability that a variable is missing may depend on its own value, making the full-data distribution not generally identified from the observed data. Table~\ref{tab:mnar60_kl} reports KL at $60\%$; full results are in Appendix~\ref{app:full_results_mnar}. Despite the absence of a general identification guarantee, the linear--nonlinear crossover observed under MCAR and MAR remains clear. On the nonlinear SCMs, MissCNF achieves the lowest KL in all 8 combinations of task and missing rate. On the linear SCMs, MICE is best in all 8 settings, and MissCNF is second-best in each. MICE degrades sharply on nonlinear Fork and Triangle (KL above 16 at $60\%$), consistent with the misspecification failure of non-causal-aware imputers.
The experiment shows that MissCNF remains empirically robust to these MNAR perturbations although no guarantee holds in this regime.
\subsection{Monte Carlo Approximation and Computational Cost}
\label{sec:mc-computation}

\begin{wrapfigure}{r}{0.5\linewidth}
    \centering
    \vspace{-0.45cm}
    \includegraphics[width=\linewidth]{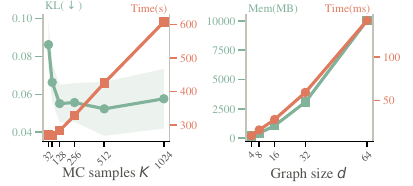}
    \caption{Scaling of MC samples and compute.}
    \label{fig:tiermc_ablation}
    \vspace{-0.05cm}
\end{wrapfigure}

We study the effect of the number of Monte Carlo samples used to estimate the observed-data marginal likelihood on $\text{Triangle}_{\text{NLIN}}$ under MAR and a missing rate of 60\%. 
The running time is nearly constant up to $K=128$, where KL plateaus, and grows linearly in $K$ beyond that (Figure~\ref{fig:tiermc_ablation}, complete results in Appendix~\ref{app:full_results_mc_ablation}). 
The linear growth in $K$ is consistent with the $\mathcal{O}(KdL)$ per-step cost of Section~\ref{sec:method}. Peak memory reaches $10$\,GB at $d=64$, which fits well on a single GPU.

\section{Conclusions and Limitations}
\label{sec:ending}

We introduced MissCNF, which trains causal normalizing flows directly on incomplete data. Unlike listwise deletion and impute-then-fit pipelines, MissCNF maximizes the observed-data likelihood of every sample without constructing any completed dataset. This can be done thanks to the autoregressive structure of CNFs: missing variables are integrated out by Monte Carlo, and only those in the ancestral closure of the observed set require sampling at all.

We analyse two concepts which are often conflated: ignorability and identifiability. Ignorability under MAR establishes that the observed-data likelihood is the right objective to maximize. The central contribution of our work is establishing the conditions for identifiability, which is where causal structure pays off. We show that recovering the full joint requires only that each variable be observed across its support together with its parents in some pattern, a condition that can hold even when no record in the dataset is ever complete. Where this condition fails, our theory identifies which region is unrecoverable.

Empirically, MissCNF performs best where causal normalizing flows are motivated in the first place. On nonlinear SCMs it consistently achieves the lowest symmetric KL divergence, including under self-masking MNAR, where no identification guarantee applies. On linear SCMs, however, a strong imputer can recover the relevant conditionals accurately.

Finally, because the objective relies only on the autoregressive structure, our method can be easily extended to CNF extensions such as DeCaFlow~\citep{almodovar2025decaflow}, which accounts for latent confounding, and TSCNF~\citep{yeom2026estimating}, which handles longitudinal data.

\textbf{Limitations.} We evaluate on synthetic data since the metrics require a known ground-truth joint distribution and known interventional and counterfactual outcomes. Validation on real-world data will be an important next step. The Monte Carlo likelihood estimate becomes expensive as the number of missing ancestral variables grows. Finally, MissCNF inherits the assumptions of CNFs: a known causal graph and no unobserved confounding. Combining MissCNF with DeCaFlow is a natural way to remove the second assumption, and we leave this to future work.

\section*{Acknowledgments}
This project was supported by the Diabetes Center Berne and strategic funding of the medical faculty of the University of Bern. Calculations were performed on UBELIX, the HPC cluster at the University of Bern.

\clearpage

\bibliography{iclr2026_conference}
\bibliographystyle{iclr2026_conference}

\appendix
\newpage
\begin{center}
    \hrule 
    \color{indigo}
    \startcontents[sections]\vbox{\vspace{4mm}\sc \LARGE Learning Causal Normalizing Flows from Incomplete Data via Observed-Data Likelihood \\\sc\small \textbf{Additional Material}} \vspace{5mm} \hrule height .5pt
    \printcontents[sections]{l}{0}{\setcounter{tocdepth}{2}}
\end{center}
\newpage
\section{Proofs}
\label{app:proofs}

\subsection{Marginal Likelihood and Training Objective}
\label{app:proofs-method}

\begin{tcolorbox}
\textbf{Proposition~\ref{prop:marginal-mc}}
The computation of $p_{\theta,O}(\mathbf{x}_O)$ involves only the variables in the ancestral closure $A = \mathrm{anc}(O)$; equivalently, only the missing ancestral coordinates $M'=A\setminus O$ need to be integrated out:
$$p_{\theta,O}(\mathbf{x}_O) = \mathbb E_{\mathbf{u}_{M'} \sim \prod_{i\in M'} p(u_i)}\big[g(\mathbf{u}_{M'})\big],
\qquad
g(\mathbf{u}_{M'}; \mathbf{x}_O) := \prod_{i\in O} p(u_i)\left|\frac{\partial\tau}{\partial x_i}\right|,
$$
where, for each $i \in A$ processed in topological order, $u_i = \tau(x_i; h_i(\mathbf{x}_{\mathrm{pa}(i)}))$ if $i \in O$, and $x_i = \tau^{-1}(u_i; h_i(\mathbf{x}_{\mathrm{pa}(i)}))$ if $i \in M'$.
\end{tcolorbox}
\begin{proof}[Proof of Proposition~\ref{prop:marginal-mc}]
Let $B := V \setminus A$ denote the barren nodes. We first show $\int \prod_{i\in B} p_\theta(x_i\mid \mathbf{x}_{\mathrm{pa}(i)})\, d\mathbf{x}_B = 1$ for any fixed $\mathbf{x}_{V\setminus B}$. 

No node in $B$ has a child in $A$: if $i\in B$ had a child $j\in A=\mathrm{anc}(O)$, then $i$ would itself be an ancestor of some node in $O$, so $i\in A$, contradicting $i\in B$. Hence every child of a node in $B$ is itself in $B$. Processing $B$ in reverse topological order (children before parents), when we reach node $i \in B$, every factor $p_\theta(x_j \mid \mathbf{x}_{\mathrm{pa}(j)})$ where $i$ is a parent of $j$ has already been integrated out, so $x_i$ appears only in its own factor:
$$
\int p_\theta(x_i\mid \mathbf{x}_{\mathrm{pa}(i)})\, dx_i=\int p(u_i)\,du_i=1,
$$
regardless of $\mathbf{x}_{\mathrm{pa}(i)}$. Repeating over all of $B$ removes the whole product.

Splitting the product over $V=A\sqcup B$ and integrating out $x_B$ first (the integrand is a nonnegative density, so the order of integration is free),
\begin{align*}
p_{\theta,O}(\mathbf{x}_{O})
&=\int  \left[\prod_{i\in A} p_\theta(x_i\mid \mathbf{x}_{\mathrm{pa}(i)})\right] \left(\int \prod_{i\in B} p_\theta(x_i\mid \mathbf{x}_{\mathrm{pa}(i)})\, d\mathbf{x}_B\right) d\mathbf{x}_{M'}\\
&= \int \prod_{i\in A} p_\theta(x_i\mid \mathbf{x}_{\mathrm{pa}(i)})\; d\mathbf{x}_{M'}.
\end{align*}

It remains to rewrite the integral over $\mathbf{x}_{M'}$ as an expectation over $\mathbf{u}_{M'}$. Since $\mathrm{pa}(i)\subseteq A$ for every $i\in A$, we process $A$ in a topological order $\prec$, so that when node $i$ is reached, every parent $\mathbf{x}_{\mathrm{pa}(i)}$ is already resolved either as observed data, or as a value already substituted via an earlier node's $u_j$, then:

\begin{itemize}
    \item If $i\in M'$, $x_i$ is an integration variable; since $\tau(\cdot\,;h_i(\mathbf{x}_{\mathrm{pa}(i)}))$ is a bijection, the change of variables $x_i = \tau^{-1}(u_i;h_i(\mathbf{x}_{\mathrm{pa}(i)}))$ cancels the density and measure Jacobians exactly, $p_\theta(x_i\mid \mathbf{x}_{\mathrm{pa}(i)})\,dx_i=p(u_i)\,du_i$.
    \item Otherwise, $i\in O$, $x_i$ is instead fixed data, so we simply re-express its density via the forward change-of-variables formula, $p_\theta(x_i\mid \mathbf{x}_{\mathrm{pa}(i)}) = p(u_i)\left|\partial \tau/\partial x_i\right|$. 
\end{itemize}

Applying this to every $i\in A$ in topological order and reassembling the product gives
$$
p_{\theta,O}(\mathbf{x}_O)=\int 
\underbrace{\prod_{i\in O} p(u_i)\left|\frac{\partial\tau}{\partial x_i}\right|}_{g(\mathbf{u}_{M'}; \mathbf{x}_O)}
\;\;
\prod_{i\in M'} p(u_i)\, du_i
= \mathbb E_{\mathbf{u}_{M'} \sim \prod_{i\in M'} p(u_i)}\big[g(\mathbf{u}_{M'}; \mathbf{x}_O)\big]. \qedhere
$$
\end{proof}

\subsection{MAR: Ignorability and Identification}
\label{app:proofs-convergence}

\begin{tcolorbox}
\textbf{Lemma~\ref{lem:mar-pullout}} 
Fix a pattern $O \subseteq V$ and let $R = \mathbf 1_O$. For \emph{any} $\theta$, define
$$
q_\theta(\mathbf{x}_O, R) := \int p_\theta(\mathbf{x}_O,\mathbf{x}_{M})\, p_\mathrm{miss}(R\mid \mathbf{x}_O, \mathbf{x}_{M})\, d\mathbf{x}_M,
$$
the observed-data density that would arise if $\mathbf{x}$ were generated by the candidate model $p_\theta$ while $R$ is generated by the true, fixed mechanism $p_\mathrm{miss}$. Under MAR, $q_\theta$ simplifies for \emph{every} $\theta$:
$$
q_\theta(\mathbf{x}_O,R) = p_{\theta,O}(\mathbf{x}_O)\, p_\mathrm{miss}(R \mid \mathbf{x}_O).
$$
In particular, if the model is realizable ($p_{\theta_0} = p^\ast$ for some $\theta_0$), then $p_{obs}(\mathbf{x}_O,R) = q_{\theta_0}(\mathbf{x}_O,R) = p^\ast_O(\mathbf{x}_O)\,p_\mathrm{miss}(R\mid \mathbf{x}_O)$. 
\end{tcolorbox}
\begin{proof}[Proof of Lemma~\ref{lem:mar-pullout}]

Under MAR, $p_\mathrm{miss}(R \mid \mathbf{x}) = p_\mathrm{miss}(R \mid \mathbf{x}_O)$ does not depend on the variable of integration $\mathbf{x}_{M}$, so it may be pulled outside the integral defining $q_\theta$:
\begin{align*}
    q_\theta(\mathbf{x}_O,R) &= \int p_\theta(\mathbf{x}_{O},\mathbf{x}_{M})\,p_\mathrm{miss}(R \mid \mathbf{x}_O)\,d\mathbf{x}_{M} \\
    &= p_\mathrm{miss}(R \mid \mathbf{x}_O)\int p_\theta(\mathbf{x}_{O},\mathbf{x}_{M})\,d\mathbf{x}_{M} \\ &= p_\mathrm{miss}(R \mid \mathbf{x}_O)\,p_{\theta,O}(\mathbf{x}_{O}).
\end{align*}

For any missingness mechanism, substituting $p_\theta=p^\ast$ into the definition of $q_\theta$ gives
\begin{align*}
    q_\theta(\mathbf{x}_O,R) &= \int p^\ast(\mathbf{x}_{O},\mathbf{x}_{M})\,p_\mathrm{miss}(R \mid \mathbf{x}_{O},\mathbf{x}_{M})\,d\mathbf{x}_{M} \\
    &= \int p_{joint}(\mathbf{x}_{O},\mathbf{x}_{M},R)\,d\mathbf{x}_{M} \\ &=p_{obs}(\mathbf{x}_O,R).
\end{align*}
matching the definition of $p_{obs}$ in Section~\ref{sec:background}.

Combining both parts at $\theta=\theta_0$ with $p_{\theta_0}=p^\ast$ gives $q_{\theta_0}(\mathbf{x}_O,R)=p_{obs}(\mathbf{x}_O,R)$ and  $q_{\theta_0}(\mathbf{x}_O,R) = p^\ast_O(\mathbf{x}_O)\,p_\mathrm{miss}(R \mid \mathbf{x}_O)$, so $p_{obs}(\mathbf{x}_O,R) = p^\ast_O(\mathbf{x}_O)\,p_\mathrm{miss}(R \mid \mathbf{x}_O)$. \qedhere
\end{proof}

\begin{tcolorbox}
\textbf{Theorem~\ref{thm:mar-consistency}} 
Suppose MAR holds, and the CNF family is realizable ($p_{\theta_0} = p^\ast$ for some $\theta_0$). Define $Z_O := \{\mathbf{x}_{O} : p_\mathrm{miss}(\mathbf 1_O\mid \mathbf{x}_{O}) = 0\}$ the set of values of $\mathbf{x}_{O}$ that pattern $O$ never reveals. Then $\theta_0$ is a global minimizer of $\mathcal L(\theta)$, and \emph{every} global minimizer $\theta^\ast$ satisfies
$$
p_{\theta^\ast,O}(\mathbf{x}_{O}) = p^\ast_O(\mathbf{x}_O) \quad \text{for $p^\ast_O$-almost every\ } \mathbf{x}_{O} \notin Z_O, \quad \text{simultaneously for every } O.
$$
\end{tcolorbox}
\begin{proof}[Proof of Theorem~\ref{thm:mar-consistency}]
By definition, we have $\mathcal L(\theta) = \mathbb E_{(\mathbf{X}_{O(\mathcal{R})},\mathcal{R}) \sim p_{obs}}\big[-\log p_{\theta,O(\mathcal{R})}(\mathbf{X}_{O(\mathcal{R})})\big]$. Moreover, by Lemma~\ref{lem:mar-pullout}, $p_{obs}(\mathbf{x}_O,R) = p_{O}^\ast(\mathbf{x}_O)\,p_\mathrm{miss}(R\mid\mathbf{x}_O)$, so wherever $p_{obs}(\mathbf{x}_O,R) > 0$ we necessarily have $p_\mathrm{miss}(R\mid \mathbf{x}_O) > 0$; thus $\log p_\mathrm{miss}(R\mid\mathbf{x}_O)$ is finite $p_{obs}$-almost every $(\mathbf{x}_O,R)$. Therefore, for any $\theta$, add and subtract $\log p_\mathrm{miss}(R \mid \mathbf{x}_O)$ inside the expectation defining $\mathcal L(\theta)-\mathcal L(\theta_0)$:
\begin{align*}
\mathcal L(\theta) - \mathcal L(\theta_0)
&= \mathbb E_{p_{obs}}\Big[\log\frac{p_{\theta_0,O}(\mathbf{x}_{O})}{p_{\theta,O}(\mathbf{x}_{O})}\Big] \\
&= \mathbb E_{p_{obs}}\Big[\log\frac{p_{\theta_0,O}(\mathbf{x}_{O})\,p_\mathrm{miss}(R \mid \mathbf{x}_O)}{p_{\theta,O}(\mathbf{x}_{O})\,p_\mathrm{miss}(R \mid \mathbf{x}_O)}\Big] \\
&= \mathbb E_{p_{obs}}\Big[\log\frac{q_{\theta_0}(\mathbf{x}_O,R)}{q_\theta(\mathbf{x}_O,R)}\Big] \\
&= \mathbb E_{p_{obs}}\Big[\log\frac{p_{obs}(\mathbf{x}_O,R)}{q_\theta(\mathbf{x}_O,R)}\Big] \\
&= \mathrm{KL}\big(p_{obs}\,\big\|\,q_\theta\big) \;\ge\; 0,
\end{align*}
using $q_{\theta_0}=p_{obs}$ by Lemma~\ref{lem:mar-pullout} in the second-to-last step. Since $\mathcal L(\theta)\ge\mathcal L(\theta_0)$ for every $\theta$, $\theta_0$ is a global minimizer. Equality $\mathcal L(\theta^\ast)=\mathcal L(\theta_0)$ holds iff the KL term vanishes, i.e.\ $q_{\theta^\ast}(\mathbf{x}_O,R) = p_{obs}(\mathbf{x}_O,R)$ for $p_{obs}$-almost every\ $(\mathbf{x}_O,R)$, i.e.
$$
p_{\theta^\ast,O}(\mathbf{x}_{O})\,p_\mathrm{miss}(R \mid \mathbf{x}_O) = p^\ast_O(\mathbf{x}_O)\,p_\mathrm{miss}(R \mid \mathbf{x}_O) \qquad \text{for $p_{obs}$-almost every\ } (\mathbf{x}_O,R).
$$
Wherever $p_{obs}(\mathbf{x}_O,R)>0$ we have $p_\mathrm{miss}(R \mid \mathbf{x}_O)>0$, so dividing gives $p_{\theta^\ast,O}(\mathbf{x}_{O})=p^\ast_O(\mathbf{x}_O)$ for $p_{obs}$-almost every\ $(\mathbf{x}_O,R)$, hence for $p^\ast_O$-almost every\ $\mathbf{x}_{O}$. Since $Z_O$ is exactly the complement of $\{\mathbf{x}_{O}:p_\mathrm{miss}(\mathbf 1_O\mid \mathbf{x}_{O})>0\}$, this is equivalent to $p_{\theta^\ast,O}(\mathbf{x}_{O})=p^\ast_O(\mathbf{x}_O)$ for $p^\ast_O$-almost every\ $\mathbf{x}_{O}\notin Z_O$, simultaneously for every $O$. 

\end{proof}

\begin{tcolorbox}
\textbf{Corollary~\ref{cor:coverage}}
Suppose $p_{\theta^\ast}$ and $p^\ast$ both factorize according to $G$ (Section~\ref{sec:background}). Under causal-family overlap (Definition~\ref{def:positivity}) and $p^\ast_{\mathrm{pa}(i)}(\mathbf{x}_{\mathrm{pa}(i)})>0$ on the relevant domain for each $i$, every global minimizer $\theta^\ast$ of $\mathcal L$ satisfies $p_{\theta^\ast}(\mathbf{x}) = p^\ast(\mathbf{x})$ for $p^\ast$-almost every $\mathbf{x}$.
\end{tcolorbox}
\begin{proof}[Proof of Corollary~\ref{cor:coverage}]
Fix $i$ and let $F_i:=\{i\}\cup\mathrm{pa}(i)\subseteq o_i$. By Theorem~\ref{thm:mar-consistency}, $p_{\theta^\ast,o_i}(\mathbf{x}_{o_i}) = p^\ast_{o_i}(\mathbf{x}_{o_i})$ for $p^\ast_{o_i}$-almost every\ $\mathbf{x}_{o_i}\notin Z_{o_i}$; by hypothesis $p^\ast_{o_i}(Z_{o_i})=0$, so this equality holds for $p^\ast_{o_i}$-almost every\ $\mathbf{x}_{o_i}$. Both sides are nonnegative densities in $\mathbf{x}_{o_i}$, so marginalizing over $\mathbf{x}_{o_i\setminus F_i}$ preserves the almost every\ equality, giving
$$
p_{\theta^\ast,F_i}(\mathbf{x}_{F_i}) = p^\ast_{F_i}(\mathbf{x}_{F_i}) \qquad \text{for $p^\ast_{F_i}$-almost every\ } \mathbf{x}_{F_i}.
$$
Marginalizing once more over $x_i$ gives
$$
p_{\theta^\ast,\mathrm{pa}(i)}(\mathbf{x}_{\mathrm{pa}(i)}) = p^\ast_{\mathrm{pa}(i)}(\mathbf{x}_{\mathrm{pa}(i)}) \qquad \text{for $p^\ast_{\mathrm{pa}(i)}$-almost every\ } \mathbf{x}_{\mathrm{pa}(i)}.
$$
Hence, wherever $p^\ast_{\mathrm{pa}(i)}(\mathbf{x}_{\mathrm{pa}(i)})>0$,
$$
p_{\theta^\ast}(x_i\mid \mathbf{x}_{\mathrm{pa}(i)}) = \frac{p_{\theta^\ast,F_i}(\mathbf{x}_{F_i})}{p_{\theta^\ast,\mathrm{pa}(i)}(\mathbf{x}_{\mathrm{pa}(i)})} = \frac{p^\ast_{F_i}(\mathbf{x}_{F_i})}{p^\ast_{\mathrm{pa}(i)}(\mathbf{x}_{\mathrm{pa}(i)})} = p^\ast(x_i\mid \mathbf{x}_{\mathrm{pa}(i)})
$$
for $p^\ast_{F_i}$-almost every $\mathbf{x}_{F_i}$. Let $N_i\subseteq\mathcal X_{F_i}$ denote the exceptional set, where $\mathcal X_{F_i}$ denote the sample space of $\mathbf{X}_{F_i}$, so $p^\ast_{F_i}(N_i)=0$.

To combine these $d$ statements for $d$ variables, lift each $N_i$ to the full sample space:
$$
\widetilde N_i := \{\mathbf{x}\in\mathcal X : \mathbf{x}_{F_i}\in N_i\}.
$$
Since $p^\ast_{F_i}(N_i)=0$, the definition of the marginal gives $p^\ast(\widetilde N_i) = p^\ast_{F_i}(N_i) = 0$. As $V$ is finite, $\widetilde N := \bigcup_{i=1}^d \widetilde N_i$ is also $p^\ast$-null. Off $\widetilde N$, every conditional equality holds simultaneously, so, using that both $p_{\theta^\ast}$ and $p^\ast$ factorize according to $G$,
$$
p_{\theta^\ast}(\mathbf{x}) = \prod_{i=1}^d p_{\theta^\ast}(x_i\mid \mathbf{x}_{\mathrm{pa}(i)}) = \prod_{i=1}^d p^\ast(x_i\mid \mathbf{x}_{\mathrm{pa}(i)}) = p^\ast(\mathbf{x})
$$
for $p^\ast$-almost every\ $\mathbf{x}$.
\end{proof}

\subsection{Recovery of Imputation Conditionals}
\label{app:proofs-imputation}

\begin{tcolorbox}
\begin{corollary}[Recovery of arbitrary imputation conditionals]
\label{cor:conditional-recovery}
Suppose the hypotheses of Corollary~\ref{cor:coverage} hold. Then every global minimizer $\theta^\ast$ of $\mathcal L$ satisfies
$$
p_{\theta^\ast}(\mathbf{x}_B\mid \mathbf{x}_A)
=
p^\ast(\mathbf{x}_B\mid \mathbf{x}_A)
$$
for any partition $V=A\sqcup B$, for $p^\ast_A$-almost every $\mathbf{x}_A$ for which the conditional is defined.
\end{corollary}
\end{tcolorbox}

\begin{proof}[Proof of Corollary~\ref{cor:conditional-recovery}]
By Corollary~\ref{cor:coverage},
$$
p_{\theta^\ast}(\mathbf{x})=p^\ast(\mathbf{x})
$$
for $p^\ast$-almost every $\mathbf{x}$. Fix any partition $V=A\sqcup B$. Marginalizing over $\mathbf{x}_B$ gives
$$
p_{\theta^\ast,A}(\mathbf{x}_A)
=
p^\ast_A(\mathbf{x}_A)
$$
for $p^\ast_A$-almost every $\mathbf{x}_A$. Wherever this common marginal is positive,
$$
p_{\theta^\ast}(\mathbf{x}_B\mid \mathbf{x}_A)
=
\frac{p_{\theta^\ast}(\mathbf{x}_A,\mathbf{x}_B)}
     {p_{\theta^\ast,A}(\mathbf{x}_A)}
=
\frac{p^\ast(\mathbf{x}_A,\mathbf{x}_B)}
     {p^\ast_A(\mathbf{x}_A)}
=
p^\ast(\mathbf{x}_B\mid \mathbf{x}_A).
$$
\end{proof}

\subsection{Discussion of Imputation-Based Alternatives}
\label{app:discussion-imputation}

\paragraph{Single imputation and Multiple imputation.}
Single imputation replaces a conditional distribution by a single point estimate. Even the oracle conditional mean $\mu(\mathbf{x}_O)=\mathbb E_{p^\ast}[\mathbf{x}_M\mid \mathbf{x}_O]$ loses variation. The completed data can have the correct conditional mean and still the wrong joint distribution. This is a classical motivation for multiple imputation \citep{rubin1987multiple}.

A stochastic imputer such as MICE or a joint generative imputer can generate $m>1$ completed datasets. Classical multiple imputation then fits the downstream analysis separately to each completed dataset and combines the resulting estimands and uncertainty estimates using Rubin-style rules \citep{rubin1987multiple}. In our setting, however, the downstream object is an entire generative model rather than a finite-dimensional estimand. Generating $m$ completed datasets would therefore produce
$$
p_{\theta^{(1)}},\ldots,p_{\theta^{(m)}}.
$$
To the best of our knowledge, there is no standard Rubin-style rule for combining these learned densities or their parameters. A natural predictive construction is the equally weighted mixture
$$
\bar p(\mathbf{x})
=
\frac{1}{m}\sum_{j=1}^m p_{\theta^{(j)}}(\mathbf{x}),
$$
equivalently sampling an index $J$ uniformly and then sampling from $p_{\theta^{(J)}}$. This is a valid mixture distribution, but it is not the classical Rubin combination procedure and requires training, storing, and evaluating $m$ downstream CNFs. Our main experiments therefore use $m=1$ for all imputation baselines
We leave the consideration of multiple imputation to future work.

\paragraph{Compatibility of fully conditional specification.}
The form of the imputation model matters independently of whether one or multiple completed datasets are generated. Fully conditional specification methods such as MICE specify a collection of conditional models, one for each partially observed variable. These conditionals need not, in general, correspond to any single coherent joint distribution; this is the classical compatibility problem for chained-equation imputation ~\citep{bartlett2015multiple}. When compatibility does hold, the resulting Gibbs-style procedure can often be interpreted as targeting a joint law \citep{vanbuuren2018flexible}. Joint-model imputers, including diffusion-based or other generative imputers, do not suffer from this particular FCS incompatibility because they parameterize a joint distribution directly.

\paragraph{Congeniality between imputation and analysis models.}
When imputation and downstream learning remain two separately fitted stages, classical multiple-imputation theory distinguishes whether the imputation and analysis procedures are \emph{congenial}: roughly, whether their probabilistic assumptions and target quantities can be embedded in a common coherent model \citep{meng1994multiple,bartlett2015multiple}. Lack of congeniality can invalidate downstream inference even when each stage appears reasonable in isolation. We do not rely on congeniality as a universal failure theorem for imputation methods; rather, it is an additional requirement introduced by fitting $q$ and the downstream model separately.

\section{Experimental details}
\label{app:exp}

\subsection{Datasets}
\label{app:datasets}

\subsubsection{Main experiment}
We test our method on eight SCMs: four graphs (Chain, Triangle, Collider, and Fork), each in a linear and a nonlinear version. We take all SCMs from \citet{javaloy2023causal} except $\text{Collider}_{\text{NLIN}}$, which they do not include. Therefore, we take it from \citet{pablo2022vaca}. All exogenous variables follow $\mathcal{N}(0,1)$. Each SCM comes with a list of intervention variables, which is used to compute the metrics (Appendix~\ref{app:metrics}). We present their formulations below and their causal graphs in Figure~\ref{figapp:causal_graphs}. 

\textbf{$\text{Chain}_{\text{LIN}}$}:
\begin{align*}
    f_1(u_1) &=  u_1 \\
    f_2(x_1, u_2) &=  10 \cdot x_1 - u_2  \\
    f_3(x_2, u_3) &= 0.25 \cdot x_2 + 2 \cdot  u_3
\end{align*}

\textbf{$\text{Chain}_{\text{NLIN}}$}:
\begin{align*}
    f_1(u_1) &=  u_1 \\
    f_2(x_1, u_2) &=  e^{ x_1/2}  +  u_2/4  \\
    f_3(x_2, u_3) &= \frac{(x_2 - 5)^3}{15}  + u_3
\end{align*}

\textbf{$\text{Collider}_{\text{LIN}}$}:
\begin{align*}
    f_1(u_1) &=  u_1 \\
    f_2(u_2) &=2 -  u_2  \\
    f_3(x_1, x_2, u_3) &= 0.25 \cdot x_2 - 0.5 \cdot x_1 + 0.5 \cdot u_3
\end{align*}

\textbf{$\text{Collider}_{\text{NLIN}}$}:
\begin{align*}
    f_1(u_1) &= u_1 \\
    f_2(u_2) &= u_2  \\
    f_3(x_1, x_2, u_3) &= 0.25 \cdot x_2^2 + 0.05 \cdot x_1 + 0.5 \cdot u_3
\end{align*}

\textbf{$\text{Fork}_{\text{LIN}}$}:
\begin{align*}
    f_1(u_1) &=  u_1 \\
    f_2(u_2) &= 2 - u_2 \\
    f_3(x_1, x_2, u_3) &= 0.25 \cdot x_2 - 1.5 \cdot x_1 + 0.5 \cdot u_3\\
    f_4(x_3, u_4) &=   x_3 +  0.25 \cdot u_4
\end{align*}

\textbf{$\text{Fork}_{\text{NLIN}}$}:
\begin{align*}
    f_1(u_1) &=  u_1 \\
    f_2(u_2) &= u_2 \\
    f_3(x_1, x_2, u_3) &=  \frac{4}{1 + e^{-x_1 - x_2}} - x_2^2 +  0.5 \cdot u_3\\
    f_4(x_3, u_4) &=   \frac{20}{1 + e^{0.5\cdot x_3^2 - x_3 }} + u_4
\end{align*}

\textbf{$\text{Triangle}_{\text{LIN}}$}:
\begin{align*}
    f_1(u_1) &=  u_1 + 1 \\
    f_2(x_1, u_2) &=  10 \cdot x_1  -  u_2  \\
    f_3(x_1, x_2, u_3) &=  0.5 \cdot x_2 + x_1  + u_3
\end{align*}

\textbf{$\text{Triangle}_{\text{NLIN}}$}:
\begin{align*}
    f_1(u_1) &=  u_1 + 1 \\
    f_2(x_1, u_2) &=  2 \cdot x_1^2  +  u_2  \\
    f_3(x_1, x_2, u_3) &=  \frac{20}{1 + e^{-x_2^2 + x_1}} + u_3
\end{align*}

\begin{figure}[h]
  \centering
  \begin{subfigure}[b]{0.32\textwidth}
  \centering
     \includegraphics[width=\textwidth]{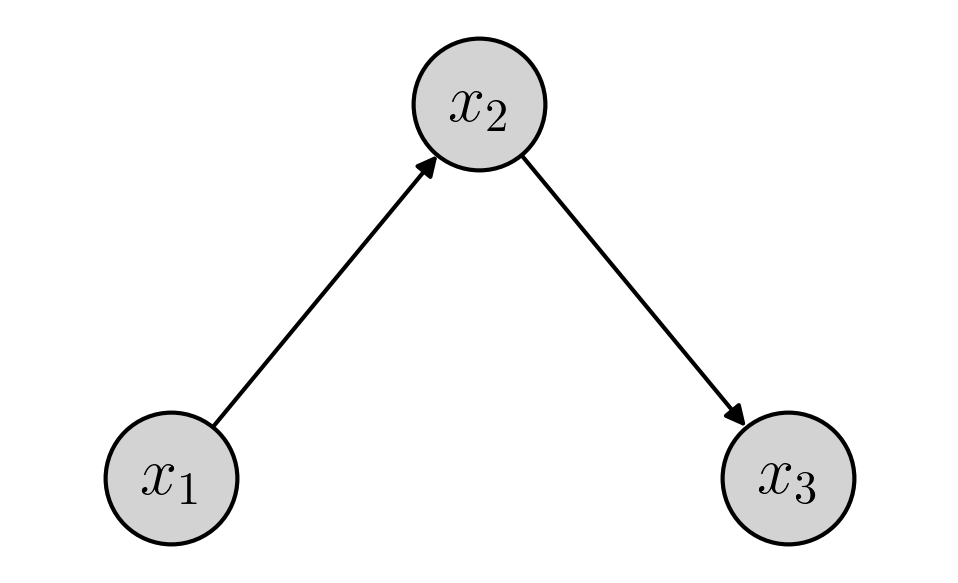}
     \caption{Chain}
  \end{subfigure}
  \begin{subfigure}[b]{0.32\textwidth}
  \centering
     \includegraphics[width=\textwidth]{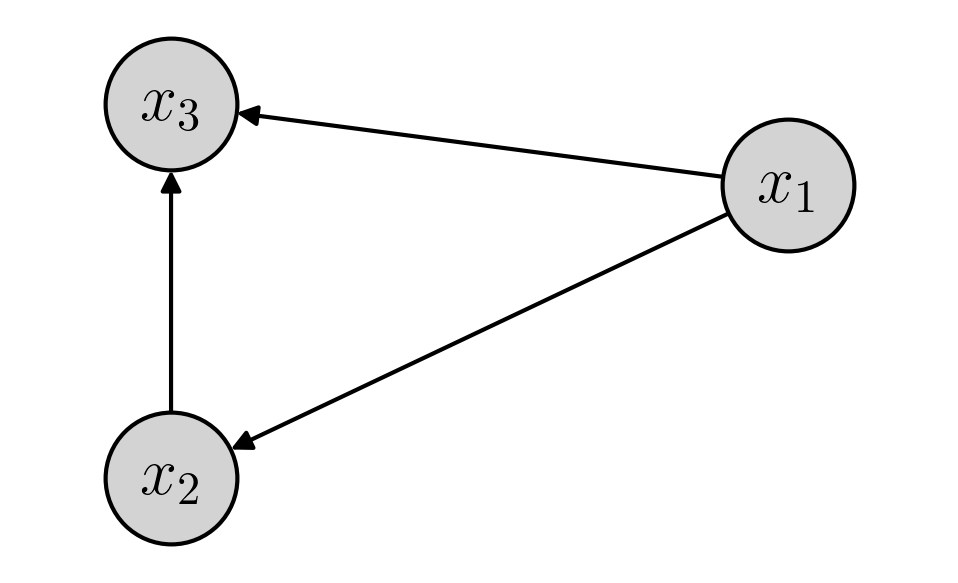}
     \caption{Triangle}
  \end{subfigure}\\
  \begin{subfigure}[b]{0.32\textwidth}
  \centering
     \includegraphics[width=\textwidth]{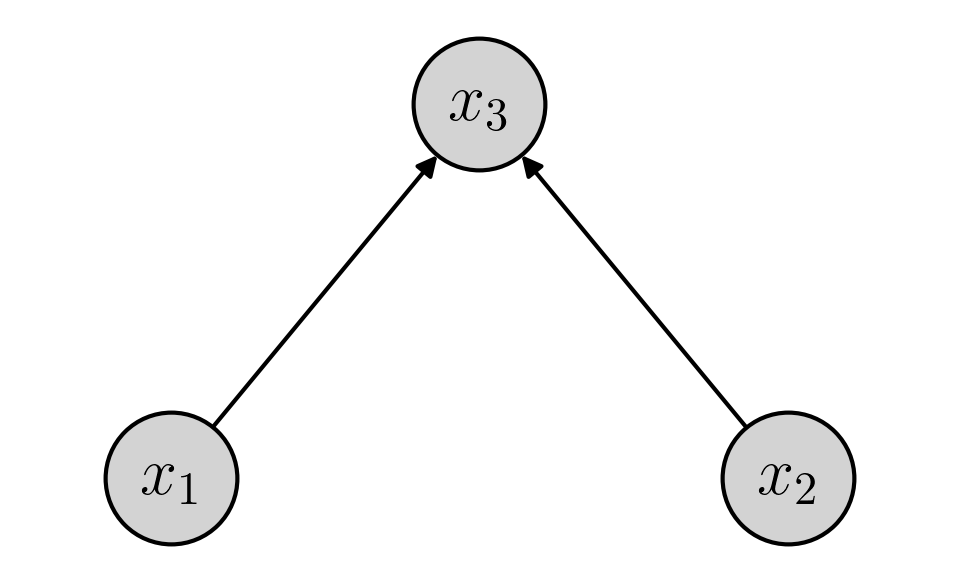}
     \caption{Collider}
  \end{subfigure}
   \begin{subfigure}[b]{0.32\textwidth}
  \centering
     \includegraphics[width=\textwidth]{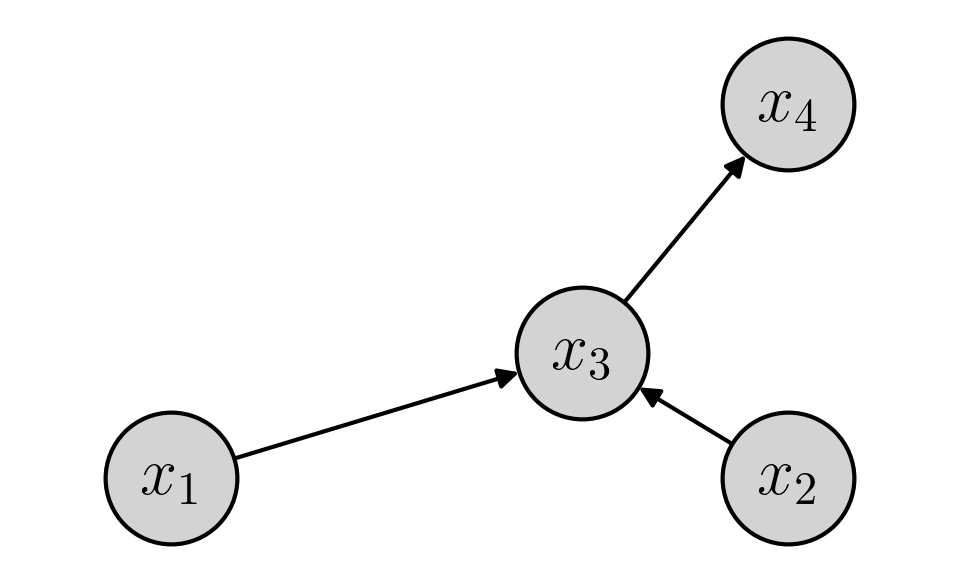}
     \caption{Fork}
  \end{subfigure}
    \caption{Causal graph of the different SCMs considered in Section~\ref{sec:experiments} ~\citep{javaloy2023causal}.}. 
    \label{figapp:causal_graphs}
\end{figure}

\subsubsection{Controlled nonlinearity study}
\label{app:dataset_nonlinearity}
We modify Chain and Collider SCMs by introducing $\alpha$ to control the nonlinearity of the SCMs:

\textbf{$\text{Chain}_\alpha$}:
\begin{align*}
    f_1(u_1) &= u_1 \\
    f_2(x_1, u_2) &= 0.5x_1 + 0.5u_2  \\
    f_3(x_2, u_3) &= h_\alpha(x_2)+0.5u_3,
\end{align*}

\textbf{$\text{Collider}_\alpha$}:
\begin{align*}
    f_1(u_1) &= u_1 \\
    f_2(u_2) &= u_2  \\
    f_3(x_1, x_2, u_3) &= 0.5h_\alpha(x_1)+0.5h_\alpha(x_2)+0.5u_3,
\end{align*}

where
\begin{equation*}
    h_\alpha(x)
    =
    \frac{
        (1-\alpha)x
        +
        \alpha (x^2-1)/\sqrt{2}
    }{
        \sqrt{(1-\alpha)^2+\alpha^2}
    },
    \qquad \alpha\in[0,1].
\end{equation*}
At $\alpha=0$, $h_\alpha(x)=x$ and the SCMs are exactly linear, while at $\alpha=1$, $h_\alpha(x)=(x^2-1)/\sqrt{2}$ and the relationship is purely quadratic. We evaluate $\alpha\in\{0,0.25,0.5,0.75,1\}$. The results of $\text{Chain}_\alpha$ and $\text{Collider}_\alpha$ are in Appendix~\ref{app:full_results_nonlinearity}.

\subsubsection{Larger graph}
\label{app:dataset_large_graph}
We test scalability on $\text{Chain}_d$, an extension of the Chain graph to $d$ nodes.
The SCM has the following form:

\textbf{$\text{Chain}_d$}:
\begin{align*}
f_1(u_1) & = u_1 \\
f_{i+1}(x_i,u_i) &= g_i(x_i)+u_i \quad \text{for} \quad i \in \{i,...,d-1\}
\end{align*}

where we design $g_i$ as a random spline function, following \cite{geffner2024deep}. We place 9 equally spaced knots on $[-3,3]$ (8 intervals), sample the knot values from $\mathcal{N}(0,1)$, and fit a natural cubic spline $g_{i}$ using Scipy \citep{2020SciPy-NMeth}. Input $x_i$ is clipped to $[-3,3]$ to avoid extrapolation and 
\begin{equation*}
    g_i(x_i) = g_{i}(\text{clip}(x_i)).
\end{equation*}

\subsection{Missingness mechanisms}
\label{app:missingness}
All the missing mask is generated based on complete data without any missing entries.

\subsubsection{Main experiment (Full-pattern positivity)}
Root variables are always observed. Therefore, missingness is applied only to each non-root variable $x_i$ with a target rate $r \in \{30\%, 60\%\, 90\%\}$ for MCAR, MAR, and $r \in \{30\%, 60\%\}$ for Self-masking MNAR.

\paragraph{MCAR}
Each entry of $x_i$ is missing independently with probability $p_{\text{miss}}(x_i) = r$.

\paragraph{MAR}
The missingness of $x_i$ depends on a conditioning set $\text{cond}(i)$, which contains the non-root parents of $x_i$. If $x_i$ has none, $\text{cond}(i)$ contains all other non-root variables. Therefore, $\text{cond}(i)$ is always observed, which satisfies the definition of MAR. We set
\begin{equation*}
    p_{\text{miss}}(x_i) = \sigma\Big(\sum_{j \in \text{cond}(i)} w_j x_j + b_i\Big),
\end{equation*}
where $\sigma$ is the sigmoid function and $w_j \sim \mathcal{U}[0.1, 1.1]$. We find $b_i$ by bisection so that the mean of $p_{\text{miss}}(x_i)$ over the data is within $10^{-3}$ of $r$. Therefore, the missing probability of each non-root node is independent.

\paragraph{Self-masking MNAR}
The missingness of $x_i$ depends on its own value:
\begin{equation}
    p_{\text{miss}}(x_i) = \sigma(w_i x_i + b_i),
\end{equation}
with $w_i \sim \mathcal{U}[0.1, 1.1]$ and $b_i$ found by bisection as above.

\subsubsection{Pattern-based MAR}
\label{app:missingness_mar_variant}
For the experiment in Section~\ref{sec:weaker-recovery}, we use Chain and Fork with two observation patterns, $o_1$ and $o_2$. Each sample is observed under exactly one pattern. Let $x_{\text{cut}}$ be the node shared by both patterns, and let $\bar{x}_{\text{cut}} = (x_{\text{cut}} - \mu)/s$, where $\mu$ and $s$ are the mean and standard deviation of $x_{\text{cut}}$ under $p^*$.

\begin{itemize}
    \item \textbf{Chain} ($x_1 \to x_2 \to x_3$): $o_1 = \{x_1, x_2\}$, $o_2 = \{x_2, x_3\}$, and $x_{\text{cut}} = x_2$.
    \item \textbf{Fork} ($x_1, x_2 \to x_3 \to x_4$): $o_1 = \{x_1, x_2, x_3\}$, $o_2 = \{x_3, x_4\}$, and $x_{\text{cut}} = x_3$.
\end{itemize}

\paragraph{Coverage holds.}
We set $p(o_1 \mid \mathbf{x}) = \sigma(\bar{x}_{\text{cut}})$ and $p(o_2 \mid \mathbf{x}) = 1 - \sigma(\bar{x}_{\text{cut}})$. Since $\sigma$ takes values in $(0,1)$, both patterns have positive probability at every value of $x_{\text{cut}}$. In Chain, the causal families $\{x_1\}$ and $\{x_1, x_2\}$ are contained in $o_1$, and $\{x_2, x_3\}$ equals $o_2$. In Fork, the families $\{x_1\}$, $\{x_2\}$ and $\{x_1, x_2, x_3\}$ are contained in $o_1$, and $\{x_3, x_4\}$ equals $o_2$. Causal-family coverage therefore holds, and $p^*$ is identifiable.

\paragraph{Coverage violated.}
Let $\tau$ be the median of $x_{\text{cut}}$. We set
\begin{equation*}
    p(o_1 \mid \mathbf{x}) = \sigma(\bar{x}_{\text{cut}})\,\mathbf{1}\{x_{\text{cut}} \ge \tau\},
    \qquad
    p(o_2 \mid \mathbf{x}) = 1 - p(o_1 \mid \mathbf{x}).
\end{equation*}
Pattern $o_1$ is never observed on $Z_{o_1} = \{\mathbf{x} : x_{\text{cut}} < \tau\}$, and $p^*(Z_{o_1}) = 1/2$. The families contained only in $o_1$ lose coverage on this region.

\subsubsection{Computational scalability}
For $\text{Chain}_d$, we use an alternating mask over consecutive nodes. Each sample follows one of two patterns, chosen with probability 0.5: [missing, observed, missing, observed, \dots] or [observed, missing, observed, missing, \dots].

\subsection{Imputation Baselines}
\label{app:baselines}


\paragraph{Mean imputation} replaces each missing value with the empirical mean of its variable over the observed entries.

\paragraph{MICE} \citep{vanbuuren2011mice} imputes each incomplete variable with a regression model conditioned on the other variables. It cycles through the variables for several rounds and draws each missing value from the model's predictive distribution instead of using the point prediction. We use the implementation in HyperImpute \citep{Jarrett2022HyperImpute}, where each per-variable model is a Bayesian ridge regression, and average the imputations over several runs.

\paragraph{MissForest} \citep{stekhoven2012missforest} also cycles through the variables like MICE, but uses random forests. It first fills missing values with the column means. Then, for each variable with missing values, it trains a random forest on the rows where that variable is observed, using the other variables as inputs, and predicts the missing entries. It repeats until the change between successive imputed datasets stops decreasing. Unlike MICE, it uses the forest's prediction directly with no random draw, so it returns a single imputation.

\paragraph{DiffPuter} \citep{zhang2025diffputer} follows the EM algorithm and alternates between an M-step and an E-step. In the M-step, it trains a diffusion model on the current completed data to learn $p(x)$. In the E-step, it samples $x_{\text{mis}}$ given $x_{\text{obs}}$ from the learned model: the observed entries are replaced by noised versions of their values, and the reverse process is run on the missing entries. The average of $N = 10$ samples becomes the new imputation.

\paragraph{KPI} \citep{wang2025kpi} imputes iteratively with kernel ridge regression using a weighted mix of Gaussian kernels with different bandwidths. In each step, it picks a random feature as the target and uses the others as inputs. It takes two batches of rows, fits the kernel ridge regression on the first batch in closed form, and uses it to predict the target on the second batch. The prediction error is then used to update, by gradient descent, both the imputed values (treated as learnable parameters) and the kernel mixing weights.

\paragraph{MIRACLE} \citep{kyono2021miracle} refines an initial imputation iteratively. Given the current completed data, a network with two sub-networks outputs a refined imputation ($\mathbb{R}^d \to \mathbb{R}^d$) and the missingness probabilities ($\mathbb{R}^d \to [0,1]^{d_S}$), where $d$ is the number of variables and $d_S$ is the number of variables with at least one missing value. The m-graph over the variables and the missingness indicators is encoded in the input-layer weights. Besides reconstructing the observed features and predicting the missingness indicators, the network is trained with an acyclicity penalty on this graph and a moment regularizer. The regularizer matches an inverse-propensity-weighted estimate of each feature's mean, computed with the predicted missingness probabilities, to the mean of the imputation network's predictions.

\subsection{Implementation}
\label{app:implementation}

\subsubsection{Imputation baselines}

\paragraph{Mean, MICE, and MissForest.}
We use the corresponding plugins of the \texttt{hyperimpute} library \citep{Jarrett2022HyperImpute}. Mean imputation uses \texttt{SimpleImputer} from \texttt{scikit-learn} \citep{scikit-learn}. MICE uses \texttt{IterativeImputer} from \texttt{scikit-learn} with a \texttt{BayesianRidge} estimator and posterior sampling ($\texttt{max\_iter}=100$, one imputation, mean initialization, ascending imputation order). MissForest uses random forest regressors ($\texttt{n\_estimators}=10$, $\texttt{max\_iter}=100$).

\paragraph{DiffPuter.}
We use the authors' code at \url{https://github.com/hengruizhang98/DiffPuter} and follow their EM training loop and configuration: hidden width $1024$, $50$ diffusion steps, $20$ RePaint resampling steps, batch size $4096$, learning rate $10^{-4}$, and up to $10{,}001$ epochs per M-step with early stopping (patience $500$) and a plateau learning-rate schedule. However, we use $5$ EM rounds and $10$ imputation trials instead of $10$ and $20$, to reduce computational cost. \citet{zhang2025diffputer} report in their ablation that $5$ rounds and $10$ trials are sufficient.

\paragraph{KPI.}
We use the authors' code at \url{https://github.com/FMLYD/kpi} with a configuration based on their default settings and running scripts: $500$ epochs, $2$ batch pairs per step, batch size $128$, early-stopping patience $10$, ridge $\lambda=1$, kernel bandwidths $\{0.01, 0.1, 5, 10, 100, 1000, 10^{4}\}$, and learning rate $10^{-2}$.

\paragraph{O-MIRACLE.}
We use the MIRACLE authors' code at \url{https://github.com/vanderschaarlab/MIRACLE}. MIRACLE normally learns a causal graph over an augmented node set (the features plus one missingness indicator per partially observed feature), using a causal regularizer. Instead, we provide the true structure through its per-node input mask and call it \textbf{O-MIRACLE} to make the difference clear. The parents of each feature come from the SCM adjacency, and the parents of each indicator come from the conditioning set of the true missingness mechanism (Appendix~\ref{app:missingness}). We drop the causal regularizer and leave all other components unchanged. This gives O-MIRACLE information of both the causal graph and missing mechanisms. The remaining settings follow the authors' example script: hidden width $32$, $\texttt{reg\_lambda}=\texttt{reg\_m}=0.1$, window $10$, $400$ refinement steps, learning rate $10^{-4}$, and batch size $32$.

\subsubsection{Causal normalizing flows}
We use the causal normalizing flow implementation from the \texttt{causal-flows} library \citep{javaloy2023causal}, available at \url{https://github.com/adrianjav/causal-flows}. The flow has a single masked autoregressive layer with an affine transform per dimension. A hyper-network maps $x$ to a shift $\mu_i$ and a log-scale $s_i$ for every dimension $i$, and the flow computes
\begin{equation*}
    u_i = \mu_i + \exp(s_i)\, x_i .
\end{equation*}
The hyper-network is a masked MLP with input size $d$, two hidden layers of $64$ units with ReLU activations, and output size $2d$. Its masks follow the graph adjacency, so $\mu_i$ and $s_i$ depend only on the parents of $x_i$. The base distribution is a standard normal. The number of autoregressive passes is set to the diameter of the causal graph, which is the smallest number consistent with its depth. All of these are the library defaults.
We train with AdamW (learning rate $10^{-3}$, weight decay $0$), batch size $4096$, and $1000$ epochs. The learning rate follows \texttt{ReduceLROnPlateau} (mode \texttt{min}, factor $0.95$, patience $60$). Both the optimizer and the schedule follow the authors' implementation at \url{https://github.com/psanch21/causal-flows}. We use $N=512$ Monte Carlo samples in all our experiments of MissCNF, unless stated otherwise. Each run uses a single NVIDIA RTX 4090 GPU (24\,GB).

\subsection{Metrics}
\label{app:metrics}
We use the evaluation metrics of \citet{javaloy2023causal}, and add a conditional KL distance for the experiment in Section~\ref{sec:tier2-tier3}. All metrics are computed against the true SCM, and the test data are complete with $n$ samples.

\subsubsection{Symmetric Kullback--Leibler divergence}
We report the symmetric \textbf{KL} divergence (KL)
\begin{equation*}
    D_J(p^\ast, p_\theta) = \mathrm{KL}(p^\ast \parallel p_\theta) + \mathrm{KL}(p_\theta \parallel p^\ast).
\end{equation*}
Both $p^\ast$ and $p_\theta$ have closed-form densities, so we estimate each term by Monte Carlo without importance weighting:
\begin{align*}
    \mathrm{KL}(p^\ast \parallel p_\theta) &\approx \frac{1}{n}\sum_{i=1}^n \big[\log p^\ast(x_i) - \log p_\theta(x_i)\big], && x_i \sim p^\ast \ \text{(test set)}, \\
    \mathrm{KL}(p_\theta \parallel p^\ast) &\approx \frac{1}{n}\sum_{i=1}^n \big[\log p_\theta(\tilde x_i) - \log p^\ast(\tilde x_i)\big], && \tilde x_i \sim p_\theta \ \text{(learned flow)}.
\end{align*}
Each of the four log-density terms is clamped from below at $-10^4$ to avoid infinite or NaN values.

\subsubsection{Average treatment effect (ATE)}
Each SCM has a set of intervention variables:
\begin{itemize}
    \item Chain: $\{x_1, x_2\}$
    \item Triangle: $\{x_1, x_2\}$
    \item Collider: $\{x_2\}$
    \item Fork: $\{x_2, x_3\}$
\end{itemize}
For each intervention variable $x_i$, we draw $5000$ fresh samples from the true SCM and compute the $25$th, $50$th, and $75$th percentiles $p_{25}, p_{50}, p_{75}$ of $x_i$. This gives three pairs $(a,b)$: $(p_{25}, p_{50})$, $(p_{25}, p_{75})$, and $(p_{50}, p_{75})$. For each pair, we draw $10{,}000$ samples under $\mathrm{do}(x_i{=}a)$ and $10{,}000$ under $\mathrm{do}(x_i{=}b)$, and estimate the ATE by Monte Carlo:
\begin{equation*}
    \widehat{\mathrm{ATE}}_{p^\ast} = \mathbb{E}\big[X \mid \mathrm{do}(x_i{=}b)\big] - \mathbb{E}\big[X \mid \mathrm{do}(x_i{=}a)\big].
\end{equation*}
$\widehat{\mathrm{ATE}}_{p_\theta}$ is computed in the same way from the learned flow. Both are vectors in $\mathbb{R}^d$. The error for one pair is
\begin{equation*}
    \textbf{$\text{RMSE}_{\text{ATE}}$}(i,a,b) = \big\lVert \widehat{\mathrm{ATE}}_{p^\ast} - \widehat{\mathrm{ATE}}_{p_\theta} \big\rVert_2,
\end{equation*}
and we report its average over all variables and pairs.

\subsubsection{Counterfactual (CF) prediction}
For each intervention variable $x_i$ and each value $a \in \{p_{25}, p_{50}, p_{75}\}$ of its percentiles, we compute the counterfactual of every test sample under $\mathrm{do}(x_i{=}a)$, once with the true SCM and once with the flow and $X^{\text{cf}}_{p_\theta}$. The error is
\begin{equation*}
    \textbf{$\text{RMSE}_{\text{CF}}$}(i,a) = \frac{1}{n}\sum_{k=1}^{n} \big\lVert \mathbf{x}^{\text{cf}}_{p^\ast,k} - \mathbf{x}^{\text{cf}}_{p_\theta,k} \big\rVert_2,
\end{equation*}
where $\mathbf{x}^{\text{cf}}_{\cdot,k}$ is the counterfactual of the $k$-th of $n=2500$ test samples. We report the average over all variables and values of $a$.
\newpage
\subsubsection{Local KL error}
\label{app:local-kl}

To locate where along $x_{\text{cut}}$ the model error lies, we partition the range of $x_{\text{cut}}$ into $J$ bins $B_1, \dots, B_J$ and write $\ell(\mathbf{x}) = \log p^*(\mathbf{x}) - \log p_\theta(\mathbf{x})$. The local error in bin $B_j$ is the mean of $\ell$ over samples that fall in the bin,

\begin{equation*}
    D_j = \mathbb{E}_{p^*}\big[\ell(\mathbf{X}) \mid x_{\text{cut}} \in B_j\big].
\end{equation*}

Unlike a contribution to the global KL, $D_j$ is not weighted by the probability of the bin, so bins with little mass, such as the tails of $x_{\text{cut}}$, are not shrunk toward zero. Let $P_j$ and $Q_j$ be the probabilities of $B_j$ under $p^*$ and $p_\theta$, and let $p^*_j$ and $p_{\theta,j}$ be the two distributions conditioned on $B_j$. Since $p^*_j(\mathbf{x}) = p^*(\mathbf{x})/P_j$ and $p_{\theta,j}(\mathbf{x}) = p_\theta(\mathbf{x})/Q_j$ on $B_j$,
\begin{equation*}
    D_j = \mathrm{KL}\big(p^*_j \,\|\, p_{\theta,j}\big) + \log \frac{P_j}{Q_j}.
\end{equation*}
The first term measures how well the model fits the shape of the distribution inside the bin, and the second whether it assigns the bin the right total mass. $D_j$ can be negative, but only when $p_\theta$ assigns more mass to $B_j$ than $p^*$ does. A correct model gives $D_j = 0$ in every bin, so values below zero do not indicate a better fit. Weighting by bin probability recovers the global KL, $\sum_{j} P_j D_j = \mathrm{KL}(p^* \,\|\, p_\theta)$, when the bins cover the whole range.

We use $J = 16$ bins of equal width between the 1\% and 99\% quantiles of $x_{\text{cut}}$ in the test set. We estimate $D_j$ by $\hat D_j = \frac{1}{m_j}\sum_{\mathbf{x} \in B_j} \ell(\mathbf{x})$, where the sum is over the $m_j$ complete test samples with $x_{\text{cut}} \in B_j$. Since $m_j$ is smaller near the edges of the range, $\hat D_j$ is noisier there.

\newpage
\section{Full results}
\label{app:full_results}

\subsection{Main result: Performance under standard missingness (Section~\ref{sec:exp1})}
\label{app:full_results_exp1}

\begin{table}[H]
\centering
\scriptsize
\setlength{\tabcolsep}{2.5pt}
\resizebox{\textwidth}{!}{%
\begin{tabular}{lcccccccc}
\toprule
Method & Chain$_{\text{LIN}}$ & Fork$_{\text{LIN}}$ & Collider$_{\text{LIN}}$ & Triangle$_{\text{LIN}}$ & Chain$_{\text{NLIN}}$ & Fork$_{\text{NLIN}}$ & Collider$_{\text{NLIN}}$ & Triangle$_{\text{NLIN}}$ \\
\midrule
\textit{Full data (reference)} & \textit{0.005{\tiny$\pm$0.003}} & \textit{0.010{\tiny$\pm$0.004}} & \textit{0.003{\tiny$\pm$0.002}} & \textit{0.007{\tiny$\pm$0.005}} & \textit{0.002{\tiny$\pm$0.002}} & \textit{0.013{\tiny$\pm$0.003}} & \textit{0.004{\tiny$\pm$0.001}} & \textit{0.015{\tiny$\pm$0.007}} \\
\midrule
\multicolumn{9}{l}{\textit{MCAR, 30\%}} \\
Listwise deletion & 0.011{\tiny$\pm$0.005} & 0.018{\tiny$\pm$0.006} & \underline{0.003}{\tiny$\pm$0.003} & 0.013{\tiny$\pm$0.007} & \underline{0.005}{\tiny$\pm$0.004} & \underline{0.029}{\tiny$\pm$0.004} & \textbf{0.004}{\tiny$\pm$0.003} & \underline{0.029}{\tiny$\pm$0.014} \\
Mean & 13.642{\tiny$\pm$0.373} & 5.750{\tiny$\pm$0.176} & 0.158{\tiny$\pm$0.011} & 19.878{\tiny$\pm$0.948} & 0.963{\tiny$\pm$0.041} & 4.051{\tiny$\pm$0.132} & 0.078{\tiny$\pm$0.013} & 3.793{\tiny$\pm$0.071} \\
MissForest & 0.688{\tiny$\pm$0.075} & 0.582{\tiny$\pm$0.091} & 0.081{\tiny$\pm$0.008} & 0.266{\tiny$\pm$0.032} & 0.166{\tiny$\pm$0.010} & 1.830{\tiny$\pm$0.127} & 0.075{\tiny$\pm$0.013} & 0.845{\tiny$\pm$0.062} \\
MICE & \textbf{0.006}{\tiny$\pm$0.005} & \textbf{0.009}{\tiny$\pm$0.002} & \textbf{0.003}{\tiny$\pm$0.002} & \underline{0.010}{\tiny$\pm$0.007} & 0.177{\tiny$\pm$0.018} & 9.056{\tiny$\pm$0.181} & 0.016{\tiny$\pm$0.002} & 8.357{\tiny$\pm$0.287} \\
KPI & 0.070{\tiny$\pm$0.006} & 0.365{\tiny$\pm$0.016} & 0.059{\tiny$\pm$0.008} & 0.085{\tiny$\pm$0.008} & 0.066{\tiny$\pm$0.008} & 0.520{\tiny$\pm$0.026} & 0.057{\tiny$\pm$0.008} & 2.765{\tiny$\pm$0.073} \\
DiffPuter & 0.119{\tiny$\pm$0.009} & 0.218{\tiny$\pm$0.036} & 0.064{\tiny$\pm$0.011} & 0.156{\tiny$\pm$0.010} & 0.141{\tiny$\pm$0.012} & 0.156{\tiny$\pm$0.029} & 0.061{\tiny$\pm$0.012} & 0.987{\tiny$\pm$0.080} \\
O-MIRACLE & 0.138{\tiny$\pm$0.008} & 0.155{\tiny$\pm$0.014} & 0.076{\tiny$\pm$0.013} & 0.118{\tiny$\pm$0.006} & 0.096{\tiny$\pm$0.004} & 0.241{\tiny$\pm$0.022} & 0.077{\tiny$\pm$0.013} & 0.820{\tiny$\pm$0.090} \\
\textbf{MissCNF (ours)} & \underline{0.008}{\tiny$\pm$0.004} & \underline{0.012}{\tiny$\pm$0.005} & 0.003{\tiny$\pm$0.002} & \textbf{0.009}{\tiny$\pm$0.005} & \textbf{0.004}{\tiny$\pm$0.001} & \textbf{0.016}{\tiny$\pm$0.006} & \underline{0.005}{\tiny$\pm$0.001} & \textbf{0.019}{\tiny$\pm$0.006} \\
\addlinespace
\multicolumn{9}{l}{\textit{MCAR, 60\%}} \\
Listwise deletion & 0.050{\tiny$\pm$0.042} & 0.065{\tiny$\pm$0.021} & \underline{0.008}{\tiny$\pm$0.004} & 0.037{\tiny$\pm$0.014} & \underline{0.037}{\tiny$\pm$0.033} & \underline{0.092}{\tiny$\pm$0.012} & \underline{0.009}{\tiny$\pm$0.003} & \underline{0.098}{\tiny$\pm$0.009} \\
Mean & 28.783{\tiny$\pm$0.474} & 8.912{\tiny$\pm$0.060} & 0.689{\tiny$\pm$0.056} & 37.877{\tiny$\pm$2.459} & 2.696{\tiny$\pm$0.069} & 9.027{\tiny$\pm$0.224} & 0.490{\tiny$\pm$0.044} & 12.014{\tiny$\pm$3.386} \\
MissForest & 2.116{\tiny$\pm$0.197} & 1.127{\tiny$\pm$0.155} & 0.505{\tiny$\pm$0.041} & 1.057{\tiny$\pm$0.105} & 1.204{\tiny$\pm$0.086} & 4.383{\tiny$\pm$0.167} & 0.473{\tiny$\pm$0.047} & 2.586{\tiny$\pm$0.301} \\
MICE & \textbf{0.004}{\tiny$\pm$0.003} & \textbf{0.010}{\tiny$\pm$0.004} & \textbf{0.004}{\tiny$\pm$0.002} & \textbf{0.007}{\tiny$\pm$0.005} & 0.488{\tiny$\pm$0.027} & 17.612{\tiny$\pm$0.368} & 0.048{\tiny$\pm$0.004} & 15.482{\tiny$\pm$0.472} \\
KPI & 0.345{\tiny$\pm$0.009} & 0.794{\tiny$\pm$0.029} & 0.342{\tiny$\pm$0.031} & 0.373{\tiny$\pm$0.012} & 0.228{\tiny$\pm$0.017} & 1.129{\tiny$\pm$0.034} & 0.327{\tiny$\pm$0.045} & 4.388{\tiny$\pm$0.185} \\
DiffPuter & 1.024{\tiny$\pm$0.082} & 1.270{\tiny$\pm$0.088} & 0.410{\tiny$\pm$0.063} & 1.097{\tiny$\pm$0.132} & 0.966{\tiny$\pm$0.135} & 1.333{\tiny$\pm$0.227} & 0.410{\tiny$\pm$0.059} & 2.784{\tiny$\pm$0.318} \\
O-MIRACLE & 0.908{\tiny$\pm$0.034} & 0.494{\tiny$\pm$0.020} & 0.485{\tiny$\pm$0.056} & 0.781{\tiny$\pm$0.027} & 0.666{\tiny$\pm$0.019} & 0.729{\tiny$\pm$0.022} & 0.478{\tiny$\pm$0.047} & 1.979{\tiny$\pm$0.121} \\
\textbf{MissCNF (ours)} & \underline{0.009}{\tiny$\pm$0.006} & \underline{0.017}{\tiny$\pm$0.006} & 0.008{\tiny$\pm$0.004} & \underline{0.016}{\tiny$\pm$0.006} & \textbf{0.006}{\tiny$\pm$0.005} & \textbf{0.029}{\tiny$\pm$0.006} & \textbf{0.007}{\tiny$\pm$0.001} & \textbf{0.032}{\tiny$\pm$0.003} \\
\addlinespace
\multicolumn{9}{l}{\textit{MCAR, 90\%}} \\
Listwise deletion & 1.242{\tiny$\pm$1.693} & 0.548{\tiny$\pm$0.145} & 0.022{\tiny$\pm$0.004} & 0.828{\tiny$\pm$0.614} & \underline{0.292}{\tiny$\pm$0.056} & \underline{1.576}{\tiny$\pm$0.998} & \underline{0.034}{\tiny$\pm$0.006} & \underline{3.262}{\tiny$\pm$1.489} \\
Mean & 52.425{\tiny$\pm$1.445} & 14.273{\tiny$\pm$0.372} & 4.619{\tiny$\pm$0.393} & 55.761{\tiny$\pm$1.344} & 9.648{\tiny$\pm$1.007} & 23.242{\tiny$\pm$2.710} & 4.156{\tiny$\pm$0.278} & 20.136{\tiny$\pm$1.173} \\
MissForest & 13.484{\tiny$\pm$1.503} & 4.037{\tiny$\pm$0.367} & 4.592{\tiny$\pm$0.274} & 6.727{\tiny$\pm$0.704} & 14.147{\tiny$\pm$3.702} & 9.635{\tiny$\pm$0.375} & 4.283{\tiny$\pm$0.383} & 19.151{\tiny$\pm$7.998} \\
MICE & \textbf{0.005}{\tiny$\pm$0.003} & \textbf{0.025}{\tiny$\pm$0.012} & \textbf{0.006}{\tiny$\pm$0.005} & \textbf{0.016}{\tiny$\pm$0.011} & 0.956{\tiny$\pm$0.125} & 24.716{\tiny$\pm$0.285} & 0.101{\tiny$\pm$0.002} & 21.659{\tiny$\pm$0.612} \\
KPI & 1.331{\tiny$\pm$0.067} & 1.976{\tiny$\pm$0.134} & 2.389{\tiny$\pm$0.143} & 1.408{\tiny$\pm$0.040} & 0.872{\tiny$\pm$0.067} & 2.814{\tiny$\pm$0.353} & 2.622{\tiny$\pm$0.222} & 5.428{\tiny$\pm$0.266} \\
DiffPuter & 26.614{\tiny$\pm$3.665} & 6.734{\tiny$\pm$0.501} & 3.751{\tiny$\pm$0.213} & 87.161{\tiny$\pm$61.871} & 6.873{\tiny$\pm$0.802} & 16.467{\tiny$\pm$1.562} & 3.674{\tiny$\pm$0.325} & 18.378{\tiny$\pm$7.317} \\
O-MIRACLE & 8.285{\tiny$\pm$0.371} & 4.581{\tiny$\pm$0.111} & 4.245{\tiny$\pm$0.426} & 7.567{\tiny$\pm$0.480} & 6.569{\tiny$\pm$0.088} & 6.460{\tiny$\pm$0.553} & 4.311{\tiny$\pm$0.419} & 8.825{\tiny$\pm$0.811} \\
\textbf{MissCNF (ours)} & \underline{0.042}{\tiny$\pm$0.020} & \underline{0.103}{\tiny$\pm$0.084} & \underline{0.019}{\tiny$\pm$0.009} & \underline{0.162}{\tiny$\pm$0.034} & \textbf{0.033}{\tiny$\pm$0.003} & \textbf{0.210}{\tiny$\pm$0.181} & \textbf{0.015}{\tiny$\pm$0.001} & \textbf{0.221}{\tiny$\pm$0.034} \\
\addlinespace
\multicolumn{9}{l}{\textit{MAR, 30\%}} \\
Listwise deletion & 0.122{\tiny$\pm$0.021} & 0.229{\tiny$\pm$0.095} & 0.065{\tiny$\pm$0.032} & 0.126{\tiny$\pm$0.026} & 0.124{\tiny$\pm$0.024} & \underline{0.219}{\tiny$\pm$0.087} & 0.070{\tiny$\pm$0.038} & \underline{0.144}{\tiny$\pm$0.024} \\
Mean & 19.358{\tiny$\pm$4.200} & 5.570{\tiny$\pm$0.396} & 0.303{\tiny$\pm$0.095} & 26.278{\tiny$\pm$4.850} & 1.823{\tiny$\pm$0.433} & 3.769{\tiny$\pm$0.187} & 0.178{\tiny$\pm$0.082} & 4.448{\tiny$\pm$0.340} \\
MissForest & 1.506{\tiny$\pm$0.729} & 0.752{\tiny$\pm$0.105} & 0.220{\tiny$\pm$0.101} & 0.819{\tiny$\pm$0.472} & 0.397{\tiny$\pm$0.182} & 2.791{\tiny$\pm$0.328} & 0.170{\tiny$\pm$0.080} & 1.447{\tiny$\pm$0.327} \\
MICE & \textbf{0.006}{\tiny$\pm$0.004} & \textbf{0.009}{\tiny$\pm$0.004} & \textbf{0.002}{\tiny$\pm$0.002} & \textbf{0.007}{\tiny$\pm$0.004} & 0.306{\tiny$\pm$0.121} & 10.785{\tiny$\pm$0.873} & \underline{0.044}{\tiny$\pm$0.011} & 10.417{\tiny$\pm$1.726} \\
KPI & 0.114{\tiny$\pm$0.029} & 0.391{\tiny$\pm$0.037} & 0.139{\tiny$\pm$0.055} & 0.123{\tiny$\pm$0.023} & \underline{0.104}{\tiny$\pm$0.036} & 0.653{\tiny$\pm$0.062} & 0.118{\tiny$\pm$0.041} & 2.938{\tiny$\pm$0.115} \\
DiffPuter & 1.020{\tiny$\pm$0.966} & 0.368{\tiny$\pm$0.140} & 0.168{\tiny$\pm$0.077} & 1.073{\tiny$\pm$0.923} & 0.401{\tiny$\pm$0.241} & 0.474{\tiny$\pm$0.207} & 0.147{\tiny$\pm$0.071} & 1.371{\tiny$\pm$0.205} \\
O-MIRACLE & 0.194{\tiny$\pm$0.020} & 0.211{\tiny$\pm$0.041} & 0.189{\tiny$\pm$0.086} & 0.166{\tiny$\pm$0.019} & 0.196{\tiny$\pm$0.042} & 0.318{\tiny$\pm$0.051} & 0.150{\tiny$\pm$0.059} & 1.027{\tiny$\pm$0.155} \\
\textbf{MissCNF (ours)} & \underline{0.014}{\tiny$\pm$0.008} & \underline{0.015}{\tiny$\pm$0.008} & \underline{0.003}{\tiny$\pm$0.002} & \underline{0.014}{\tiny$\pm$0.007} & \textbf{0.016}{\tiny$\pm$0.021} & \textbf{0.025}{\tiny$\pm$0.010} & \textbf{0.007}{\tiny$\pm$0.002} & \textbf{0.028}{\tiny$\pm$0.010} \\
\addlinespace
\multicolumn{9}{l}{\textit{MAR, 60\%}} \\
Listwise deletion & 0.381{\tiny$\pm$0.065} & 0.739{\tiny$\pm$0.298} & 0.249{\tiny$\pm$0.117} & 0.393{\tiny$\pm$0.072} & \underline{0.435}{\tiny$\pm$0.127} & \underline{0.793}{\tiny$\pm$0.293} & 0.254{\tiny$\pm$0.140} & \underline{0.532}{\tiny$\pm$0.166} \\
Mean & 40.447{\tiny$\pm$7.807} & 9.197{\tiny$\pm$0.272} & 1.863{\tiny$\pm$1.027} & 48.858{\tiny$\pm$7.987} & 4.240{\tiny$\pm$0.891} & 9.013{\tiny$\pm$0.570} & 6.925{\tiny$\pm$13.396} & 9.477{\tiny$\pm$0.569} \\
MissForest & 3.883{\tiny$\pm$1.328} & 1.861{\tiny$\pm$0.509} & 1.641{\tiny$\pm$0.952} & 2.997{\tiny$\pm$1.325} & 2.343{\tiny$\pm$0.874} & 5.227{\tiny$\pm$0.566} & 1.592{\tiny$\pm$1.571} & 4.450{\tiny$\pm$0.820} \\
MICE & \textbf{0.006}{\tiny$\pm$0.003} & \textbf{0.009}{\tiny$\pm$0.002} & \textbf{0.004}{\tiny$\pm$0.004} & \textbf{0.009}{\tiny$\pm$0.006} & 0.751{\tiny$\pm$0.256} & 21.217{\tiny$\pm$1.883} & \underline{0.101}{\tiny$\pm$0.026} & 19.927{\tiny$\pm$2.473} \\
KPI & 0.604{\tiny$\pm$0.270} & 1.075{\tiny$\pm$0.235} & 0.897{\tiny$\pm$0.413} & 0.655{\tiny$\pm$0.285} & 0.450{\tiny$\pm$0.201} & 1.642{\tiny$\pm$0.435} & 1.090{\tiny$\pm$0.833} & 4.855{\tiny$\pm$0.529} \\
DiffPuter & 11.181{\tiny$\pm$6.655} & 2.458{\tiny$\pm$0.906} & 1.296{\tiny$\pm$0.687} & 12.377{\tiny$\pm$9.267} & 2.212{\tiny$\pm$1.015} & 4.268{\tiny$\pm$2.237} & 1.738{\tiny$\pm$2.116} & 4.350{\tiny$\pm$1.234} \\
O-MIRACLE & 3.690{\tiny$\pm$4.407} & 1.326{\tiny$\pm$0.772} & 1.263{\tiny$\pm$0.671} & 3.578{\tiny$\pm$4.510} & 1.613{\tiny$\pm$0.954} & 1.435{\tiny$\pm$0.726} & 1.198{\tiny$\pm$0.688} & 2.878{\tiny$\pm$0.615} \\
\textbf{MissCNF (ours)} & \underline{0.029}{\tiny$\pm$0.017} & \underline{0.033}{\tiny$\pm$0.013} & \underline{0.008}{\tiny$\pm$0.002} & \underline{0.034}{\tiny$\pm$0.016} & \textbf{0.025}{\tiny$\pm$0.022} & \textbf{0.032}{\tiny$\pm$0.011} & \textbf{0.011}{\tiny$\pm$0.001} & \textbf{0.052}{\tiny$\pm$0.032} \\
\addlinespace
\multicolumn{9}{l}{\textit{MAR, 90\%}} \\
Listwise deletion & 14.365{\tiny$\pm$10.153} & 54.749{\tiny$\pm$105.659} & 1.300{\tiny$\pm$1.779} & 7.028{\tiny$\pm$4.523} & 3.622{\tiny$\pm$2.310} & 7.401{\tiny$\pm$2.831} & 0.772{\tiny$\pm$0.678} & 22.316{\tiny$\pm$27.533} \\
Mean & 324.773{\tiny$\pm$337.569} & 37.922{\tiny$\pm$28.441} & 33.047{\tiny$\pm$37.771} & 209.427{\tiny$\pm$259.684} & 125.221{\tiny$\pm$194.035} & 42.405{\tiny$\pm$19.754} & 28.089{\tiny$\pm$31.147} & 70.806{\tiny$\pm$94.904} \\
MissForest & 18.373{\tiny$\pm$3.435} & 14.614{\tiny$\pm$9.701} & 15.488{\tiny$\pm$11.749} & 18.650{\tiny$\pm$9.691} & 18.934{\tiny$\pm$1.970} & 14.088{\tiny$\pm$4.933} & 24.085{\tiny$\pm$21.509} & 69.188{\tiny$\pm$27.696} \\
MICE & \textbf{0.009}{\tiny$\pm$0.009} & \textbf{0.022}{\tiny$\pm$0.010} & \textbf{0.007}{\tiny$\pm$0.005} & \textbf{0.025}{\tiny$\pm$0.023} & \underline{1.272}{\tiny$\pm$0.352} & 34.704{\tiny$\pm$5.341} & \underline{0.210}{\tiny$\pm$0.064} & 27.042{\tiny$\pm$4.654} \\
KPI & 2.153{\tiny$\pm$0.658} & 4.550{\tiny$\pm$2.090} & 7.601{\tiny$\pm$5.810} & 2.262{\tiny$\pm$0.652} & 1.463{\tiny$\pm$0.377} & \underline{5.347}{\tiny$\pm$2.216} & 10.629{\tiny$\pm$5.841} & \underline{7.058}{\tiny$\pm$0.758} \\
DiffPuter & 91.241{\tiny$\pm$71.097} & 23.082{\tiny$\pm$16.464} & 14.365{\tiny$\pm$10.648} & 161.602{\tiny$\pm$128.402} & 28.363{\tiny$\pm$26.479} & 32.768{\tiny$\pm$21.681} & 40.917{\tiny$\pm$61.437} & 55.624{\tiny$\pm$43.512} \\
O-MIRACLE & 41.293{\tiny$\pm$28.772} & 23.727{\tiny$\pm$20.749} & 12.479{\tiny$\pm$9.174} & 41.247{\tiny$\pm$27.602} & 15.721{\tiny$\pm$5.929} & 22.174{\tiny$\pm$9.637} & 13.622{\tiny$\pm$8.549} & 16.275{\tiny$\pm$6.147} \\
\textbf{MissCNF (ours)} & \underline{0.059}{\tiny$\pm$0.019} & \underline{0.108}{\tiny$\pm$0.052} & \underline{0.023}{\tiny$\pm$0.005} & \underline{0.243}{\tiny$\pm$0.132} & \textbf{0.073}{\tiny$\pm$0.027} & \textbf{0.348}{\tiny$\pm$0.114} & \textbf{0.031}{\tiny$\pm$0.017} & \textbf{0.834}{\tiny$\pm$0.989} \\
\bottomrule
\end{tabular}}
\caption{KL across all 8 SCMs, both missingness mechanisms and three missingness rates. Values are mean $\pm$ std over seeds. Lower is better. Bold = best, underline = second-best, within each row-block (mechanism $\times$ rate); the full-data reference row (model trained with no missingness) is excluded from best/second-best marking.}
\label{tab:kl_distance}
\end{table}

\begin{table}[H]
\centering
\scriptsize
\setlength{\tabcolsep}{2.5pt}
\resizebox{\textwidth}{!}{%
\begin{tabular}{lcccccccc}
\toprule
Method & Chain$_{\text{LIN}}$ & Fork$_{\text{LIN}}$ & Collider$_{\text{LIN}}$ & Triangle$_{\text{LIN}}$ & Chain$_{\text{NLIN}}$ & Fork$_{\text{NLIN}}$ & Collider$_{\text{NLIN}}$ & Triangle$_{\text{NLIN}}$ \\
\midrule
\textit{Full data (reference)} & \textit{0.054{\tiny$\pm$0.014}} & \textit{0.033{\tiny$\pm$0.004}} & \textit{0.011{\tiny$\pm$0.001}} & \textit{0.350{\tiny$\pm$0.074}} & \textit{0.027{\tiny$\pm$0.008}} & \textit{0.063{\tiny$\pm$0.014}} & \textit{0.018{\tiny$\pm$0.006}} & \textit{0.145{\tiny$\pm$0.048}} \\
\midrule
\multicolumn{9}{l}{\textit{MCAR, 30\%}} \\
Listwise deletion & 0.089{\tiny$\pm$0.035} & 0.035{\tiny$\pm$0.006} & 0.013{\tiny$\pm$0.002} & \underline{0.333}{\tiny$\pm$0.052} & \underline{0.040}{\tiny$\pm$0.011} & 0.083{\tiny$\pm$0.030} & \textbf{0.019}{\tiny$\pm$0.007} & \underline{0.219}{\tiny$\pm$0.091} \\
Mean & 1.778{\tiny$\pm$0.043} & 0.299{\tiny$\pm$0.010} & 0.067{\tiny$\pm$0.004} & 3.809{\tiny$\pm$0.311} & 0.599{\tiny$\pm$0.016} & 0.772{\tiny$\pm$0.048} & 0.027{\tiny$\pm$0.006} & 1.354{\tiny$\pm$0.189} \\
MissForest & 0.407{\tiny$\pm$0.090} & 0.057{\tiny$\pm$0.012} & 0.012{\tiny$\pm$0.002} & 0.373{\tiny$\pm$0.289} & 0.077{\tiny$\pm$0.021} & 0.580{\tiny$\pm$0.172} & 0.028{\tiny$\pm$0.009} & 1.130{\tiny$\pm$0.473} \\
MICE & \textbf{0.070}{\tiny$\pm$0.008} & \underline{0.032}{\tiny$\pm$0.004} & 0.012{\tiny$\pm$0.003} & 0.372{\tiny$\pm$0.046} & 0.122{\tiny$\pm$0.005} & 1.187{\tiny$\pm$0.037} & 0.026{\tiny$\pm$0.007} & 2.184{\tiny$\pm$0.062} \\
KPI & 0.086{\tiny$\pm$0.018} & 0.043{\tiny$\pm$0.006} & 0.014{\tiny$\pm$0.005} & 0.924{\tiny$\pm$0.039} & 0.082{\tiny$\pm$0.009} & 0.125{\tiny$\pm$0.042} & 0.024{\tiny$\pm$0.011} & 1.503{\tiny$\pm$0.249} \\
DiffPuter & 0.078{\tiny$\pm$0.027} & \textbf{0.031}{\tiny$\pm$0.004} & 0.013{\tiny$\pm$0.003} & \textbf{0.243}{\tiny$\pm$0.143} & 0.045{\tiny$\pm$0.014} & 0.111{\tiny$\pm$0.036} & 0.020{\tiny$\pm$0.009} & 0.643{\tiny$\pm$0.373} \\
O-MIRACLE & \underline{0.074}{\tiny$\pm$0.020} & 0.035{\tiny$\pm$0.005} & \underline{0.012}{\tiny$\pm$0.003} & 0.387{\tiny$\pm$0.068} & \textbf{0.033}{\tiny$\pm$0.006} & \underline{0.075}{\tiny$\pm$0.021} & \underline{0.020}{\tiny$\pm$0.007} & 0.560{\tiny$\pm$0.120} \\
\textbf{MissCNF (ours)} & 0.093{\tiny$\pm$0.018} & 0.033{\tiny$\pm$0.003} & \textbf{0.012}{\tiny$\pm$0.002} & 0.343{\tiny$\pm$0.044} & 0.041{\tiny$\pm$0.015} & \textbf{0.069}{\tiny$\pm$0.015} & 0.022{\tiny$\pm$0.007} & \textbf{0.164}{\tiny$\pm$0.057} \\
\addlinespace
\multicolumn{9}{l}{\textit{MCAR, 60\%}} \\
Listwise deletion & 0.171{\tiny$\pm$0.023} & 0.038{\tiny$\pm$0.013} & 0.015{\tiny$\pm$0.004} & \underline{0.359}{\tiny$\pm$0.211} & 0.095{\tiny$\pm$0.043} & 0.131{\tiny$\pm$0.076} & 0.031{\tiny$\pm$0.011} & \underline{0.236}{\tiny$\pm$0.102} \\
Mean & 3.486{\tiny$\pm$0.036} & 0.556{\tiny$\pm$0.015} & 0.132{\tiny$\pm$0.006} & 5.000{\tiny$\pm$0.620} & 1.091{\tiny$\pm$0.012} & 1.383{\tiny$\pm$0.034} & 0.047{\tiny$\pm$0.004} & 2.998{\tiny$\pm$0.133} \\
MissForest & 0.369{\tiny$\pm$0.159} & 0.083{\tiny$\pm$0.019} & 0.027{\tiny$\pm$0.011} & 4.374{\tiny$\pm$4.999} & 0.164{\tiny$\pm$0.030} & 0.826{\tiny$\pm$0.142} & 0.049{\tiny$\pm$0.010} & 1.953{\tiny$\pm$1.493} \\
MICE & \textbf{0.060}{\tiny$\pm$0.024} & \underline{0.033}{\tiny$\pm$0.004} & \textbf{0.011}{\tiny$\pm$0.004} & \textbf{0.299}{\tiny$\pm$0.104} & 0.190{\tiny$\pm$0.023} & 1.826{\tiny$\pm$0.053} & 0.045{\tiny$\pm$0.006} & 3.039{\tiny$\pm$0.035} \\
KPI & 0.111{\tiny$\pm$0.024} & 0.051{\tiny$\pm$0.015} & 0.018{\tiny$\pm$0.007} & 1.518{\tiny$\pm$0.048} & 0.125{\tiny$\pm$0.031} & 0.212{\tiny$\pm$0.034} & 0.033{\tiny$\pm$0.010} & 2.625{\tiny$\pm$0.112} \\
DiffPuter & 0.374{\tiny$\pm$0.103} & 0.043{\tiny$\pm$0.011} & 0.016{\tiny$\pm$0.003} & 0.789{\tiny$\pm$0.377} & 0.088{\tiny$\pm$0.022} & 0.342{\tiny$\pm$0.064} & \underline{0.028}{\tiny$\pm$0.009} & 1.028{\tiny$\pm$0.463} \\
O-MIRACLE & \underline{0.067}{\tiny$\pm$0.012} & 0.035{\tiny$\pm$0.006} & \underline{0.013}{\tiny$\pm$0.005} & 0.403{\tiny$\pm$0.118} & \textbf{0.042}{\tiny$\pm$0.012} & \underline{0.111}{\tiny$\pm$0.017} & \textbf{0.025}{\tiny$\pm$0.015} & 0.816{\tiny$\pm$0.159} \\
\textbf{MissCNF (ours)} & 0.093{\tiny$\pm$0.022} & \textbf{0.033}{\tiny$\pm$0.004} & 0.013{\tiny$\pm$0.004} & 0.419{\tiny$\pm$0.092} & \underline{0.051}{\tiny$\pm$0.019} & \textbf{0.093}{\tiny$\pm$0.026} & 0.031{\tiny$\pm$0.010} & \textbf{0.178}{\tiny$\pm$0.057} \\
\addlinespace
\multicolumn{9}{l}{\textit{MCAR, 90\%}} \\
Listwise deletion & 0.534{\tiny$\pm$0.269} & 0.079{\tiny$\pm$0.011} & 0.031{\tiny$\pm$0.016} & \underline{1.283}{\tiny$\pm$0.424} & 0.139{\tiny$\pm$0.052} & 0.421{\tiny$\pm$0.097} & 0.064{\tiny$\pm$0.025} & \underline{1.083}{\tiny$\pm$0.281} \\
Mean & 5.118{\tiny$\pm$0.032} & 0.816{\tiny$\pm$0.004} & 0.204{\tiny$\pm$0.002} & 6.840{\tiny$\pm$0.057} & 1.465{\tiny$\pm$0.013} & 1.956{\tiny$\pm$0.044} & 0.069{\tiny$\pm$0.002} & 4.469{\tiny$\pm$0.032} \\
MissForest & 0.701{\tiny$\pm$0.120} & 0.090{\tiny$\pm$0.024} & 0.050{\tiny$\pm$0.019} & 7.394{\tiny$\pm$8.690} & 0.339{\tiny$\pm$0.071} & 2.383{\tiny$\pm$0.358} & 0.068{\tiny$\pm$0.005} & 6.561{\tiny$\pm$10.536} \\
MICE & \textbf{0.083}{\tiny$\pm$0.045} & \textbf{0.037}{\tiny$\pm$0.005} & \textbf{0.015}{\tiny$\pm$0.007} & \textbf{0.458}{\tiny$\pm$0.192} & 0.278{\tiny$\pm$0.061} & 2.173{\tiny$\pm$0.120} & 0.079{\tiny$\pm$0.006} & 3.806{\tiny$\pm$0.130} \\
KPI & 0.209{\tiny$\pm$0.096} & 0.060{\tiny$\pm$0.016} & 0.043{\tiny$\pm$0.016} & 2.089{\tiny$\pm$0.135} & 0.156{\tiny$\pm$0.075} & 0.291{\tiny$\pm$0.075} & 0.098{\tiny$\pm$0.035} & 3.247{\tiny$\pm$0.095} \\
DiffPuter & 2.837{\tiny$\pm$0.355} & 0.178{\tiny$\pm$0.042} & 0.126{\tiny$\pm$0.015} & 52.412{\tiny$\pm$63.952} & 0.523{\tiny$\pm$0.073} & 1.544{\tiny$\pm$0.069} & \underline{0.040}{\tiny$\pm$0.009} & 2.638{\tiny$\pm$0.899} \\
O-MIRACLE & 0.295{\tiny$\pm$0.124} & 0.056{\tiny$\pm$0.015} & \underline{0.023}{\tiny$\pm$0.012} & 1.576{\tiny$\pm$0.262} & \underline{0.093}{\tiny$\pm$0.025} & \underline{0.272}{\tiny$\pm$0.082} & \textbf{0.037}{\tiny$\pm$0.026} & 1.625{\tiny$\pm$0.273} \\
\textbf{MissCNF (ours)} & \underline{0.177}{\tiny$\pm$0.078} & \underline{0.045}{\tiny$\pm$0.011} & 0.024{\tiny$\pm$0.008} & 1.572{\tiny$\pm$0.128} & \textbf{0.088}{\tiny$\pm$0.023} & \textbf{0.230}{\tiny$\pm$0.056} & 0.047{\tiny$\pm$0.018} & \textbf{0.404}{\tiny$\pm$0.112} \\
\addlinespace
\multicolumn{9}{l}{\textit{MAR, 30\%}} \\
Listwise deletion & 0.075{\tiny$\pm$0.013} & \underline{0.033}{\tiny$\pm$0.004} & 0.014{\tiny$\pm$0.002} & 0.354{\tiny$\pm$0.109} & 0.039{\tiny$\pm$0.013} & 0.287{\tiny$\pm$0.108} & 0.019{\tiny$\pm$0.010} & \underline{0.494}{\tiny$\pm$0.119} \\
Mean & 1.836{\tiny$\pm$0.047} & 0.253{\tiny$\pm$0.028} & 0.052{\tiny$\pm$0.010} & 3.779{\tiny$\pm$0.392} & 0.586{\tiny$\pm$0.026} & 0.910{\tiny$\pm$0.034} & 0.024{\tiny$\pm$0.006} & 1.754{\tiny$\pm$0.113} \\
MissForest & 0.343{\tiny$\pm$0.101} & 0.051{\tiny$\pm$0.015} & 0.015{\tiny$\pm$0.003} & 0.412{\tiny$\pm$0.156} & 0.091{\tiny$\pm$0.018} & 0.857{\tiny$\pm$0.080} & 0.019{\tiny$\pm$0.007} & 0.913{\tiny$\pm$0.057} \\
MICE & \textbf{0.053}{\tiny$\pm$0.012} & \textbf{0.032}{\tiny$\pm$0.003} & \underline{0.012}{\tiny$\pm$0.003} & \textbf{0.300}{\tiny$\pm$0.102} & 0.050{\tiny$\pm$0.014} & 1.234{\tiny$\pm$0.040} & 0.031{\tiny$\pm$0.007} & 2.298{\tiny$\pm$0.072} \\
KPI & 0.071{\tiny$\pm$0.022} & 0.039{\tiny$\pm$0.007} & 0.013{\tiny$\pm$0.005} & 0.903{\tiny$\pm$0.063} & 0.100{\tiny$\pm$0.016} & 0.145{\tiny$\pm$0.020} & 0.024{\tiny$\pm$0.011} & 1.641{\tiny$\pm$0.159} \\
DiffPuter & 0.089{\tiny$\pm$0.040} & 0.038{\tiny$\pm$0.014} & 0.017{\tiny$\pm$0.008} & 0.374{\tiny$\pm$0.329} & 0.057{\tiny$\pm$0.018} & 0.157{\tiny$\pm$0.063} & 0.023{\tiny$\pm$0.010} & 0.551{\tiny$\pm$0.134} \\
O-MIRACLE & \underline{0.053}{\tiny$\pm$0.018} & 0.033{\tiny$\pm$0.003} & 0.012{\tiny$\pm$0.002} & \underline{0.345}{\tiny$\pm$0.103} & \textbf{0.033}{\tiny$\pm$0.009} & \underline{0.078}{\tiny$\pm$0.011} & \underline{0.018}{\tiny$\pm$0.004} & 0.599{\tiny$\pm$0.082} \\
\textbf{MissCNF (ours)} & 0.069{\tiny$\pm$0.020} & 0.035{\tiny$\pm$0.002} & \textbf{0.012}{\tiny$\pm$0.002} & 0.367{\tiny$\pm$0.102} & \underline{0.037}{\tiny$\pm$0.010} & \textbf{0.062}{\tiny$\pm$0.014} & \textbf{0.016}{\tiny$\pm$0.004} & \textbf{0.153}{\tiny$\pm$0.029} \\
\addlinespace
\multicolumn{9}{l}{\textit{MAR, 60\%}} \\
Listwise deletion & 0.152{\tiny$\pm$0.071} & 0.044{\tiny$\pm$0.012} & 0.015{\tiny$\pm$0.007} & \underline{0.436}{\tiny$\pm$0.210} & 0.074{\tiny$\pm$0.028} & 0.528{\tiny$\pm$0.213} & 0.031{\tiny$\pm$0.019} & 0.881{\tiny$\pm$0.180} \\
Mean & 3.768{\tiny$\pm$0.216} & 0.517{\tiny$\pm$0.030} & 0.133{\tiny$\pm$0.010} & 5.242{\tiny$\pm$0.532} & 1.117{\tiny$\pm$0.015} & 1.373{\tiny$\pm$0.043} & 0.047{\tiny$\pm$0.011} & 3.135{\tiny$\pm$0.345} \\
MissForest & 0.468{\tiny$\pm$0.207} & 0.069{\tiny$\pm$0.016} & 0.024{\tiny$\pm$0.004} & 1.470{\tiny$\pm$1.210} & 0.187{\tiny$\pm$0.050} & 1.300{\tiny$\pm$0.292} & 0.043{\tiny$\pm$0.009} & 1.866{\tiny$\pm$1.183} \\
MICE & \textbf{0.063}{\tiny$\pm$0.027} & \textbf{0.035}{\tiny$\pm$0.007} & \textbf{0.011}{\tiny$\pm$0.002} & \textbf{0.297}{\tiny$\pm$0.062} & 0.081{\tiny$\pm$0.026} & 2.240{\tiny$\pm$0.176} & 0.094{\tiny$\pm$0.026} & 3.429{\tiny$\pm$0.121} \\
KPI & 0.105{\tiny$\pm$0.042} & 0.060{\tiny$\pm$0.012} & 0.017{\tiny$\pm$0.005} & 1.489{\tiny$\pm$0.056} & 0.124{\tiny$\pm$0.017} & 0.207{\tiny$\pm$0.016} & 0.039{\tiny$\pm$0.015} & 2.658{\tiny$\pm$0.151} \\
DiffPuter & 1.078{\tiny$\pm$0.469} & 0.114{\tiny$\pm$0.062} & 0.015{\tiny$\pm$0.005} & 2.010{\tiny$\pm$1.467} & 0.255{\tiny$\pm$0.140} & 0.718{\tiny$\pm$0.264} & \underline{0.022}{\tiny$\pm$0.010} & 1.604{\tiny$\pm$0.175} \\
O-MIRACLE & 0.102{\tiny$\pm$0.068} & 0.040{\tiny$\pm$0.001} & 0.016{\tiny$\pm$0.005} & 0.451{\tiny$\pm$0.186} & \textbf{0.034}{\tiny$\pm$0.012} & \underline{0.189}{\tiny$\pm$0.100} & \textbf{0.018}{\tiny$\pm$0.010} & \underline{0.853}{\tiny$\pm$0.052} \\
\textbf{MissCNF (ours)} & \underline{0.093}{\tiny$\pm$0.034} & \underline{0.037}{\tiny$\pm$0.007} & \underline{0.014}{\tiny$\pm$0.005} & 0.492{\tiny$\pm$0.345} & \underline{0.045}{\tiny$\pm$0.010} & \textbf{0.083}{\tiny$\pm$0.024} & 0.027{\tiny$\pm$0.013} & \textbf{0.190}{\tiny$\pm$0.043} \\
\addlinespace
\multicolumn{9}{l}{\textit{MAR, 90\%}} \\
Listwise deletion & 0.704{\tiny$\pm$0.571} & 0.087{\tiny$\pm$0.047} & 0.032{\tiny$\pm$0.014} & \underline{1.340}{\tiny$\pm$0.510} & 0.242{\tiny$\pm$0.050} & 1.299{\tiny$\pm$0.449} & 0.046{\tiny$\pm$0.020} & \underline{1.871}{\tiny$\pm$0.691} \\
Mean & 5.333{\tiny$\pm$0.136} & 0.811{\tiny$\pm$0.023} & 0.206{\tiny$\pm$0.005} & 7.013{\tiny$\pm$0.103} & 1.471{\tiny$\pm$0.017} & 1.803{\tiny$\pm$0.032} & 0.075{\tiny$\pm$0.003} & 4.602{\tiny$\pm$0.055} \\
MissForest & 0.831{\tiny$\pm$0.248} & 0.133{\tiny$\pm$0.036} & 0.048{\tiny$\pm$0.028} & 18.145{\tiny$\pm$35.762} & 0.311{\tiny$\pm$0.080} & 1.879{\tiny$\pm$0.321} & 0.067{\tiny$\pm$0.007} & 18.476{\tiny$\pm$29.456} \\
MICE & \textbf{0.095}{\tiny$\pm$0.043} & \textbf{0.041}{\tiny$\pm$0.004} & \textbf{0.013}{\tiny$\pm$0.003} & \textbf{0.444}{\tiny$\pm$0.270} & 0.207{\tiny$\pm$0.037} & 3.106{\tiny$\pm$0.244} & 0.218{\tiny$\pm$0.066} & 4.427{\tiny$\pm$0.417} \\
KPI & 0.243{\tiny$\pm$0.085} & 0.075{\tiny$\pm$0.035} & 0.031{\tiny$\pm$0.012} & 2.005{\tiny$\pm$0.160} & \underline{0.152}{\tiny$\pm$0.101} & \underline{0.323}{\tiny$\pm$0.088} & 0.074{\tiny$\pm$0.020} & 3.339{\tiny$\pm$0.170} \\
DiffPuter & 4.345{\tiny$\pm$0.619} & 0.342{\tiny$\pm$0.228} & 0.117{\tiny$\pm$0.021} & 12.606{\tiny$\pm$12.593} & 0.878{\tiny$\pm$0.104} & 1.673{\tiny$\pm$0.273} & 0.068{\tiny$\pm$0.018} & 4.628{\tiny$\pm$0.731} \\
O-MIRACLE & 1.899{\tiny$\pm$0.804} & 0.205{\tiny$\pm$0.141} & 0.038{\tiny$\pm$0.029} & 3.368{\tiny$\pm$0.801} & 0.420{\tiny$\pm$0.180} & 1.638{\tiny$\pm$0.478} & \underline{0.041}{\tiny$\pm$0.030} & 2.550{\tiny$\pm$0.427} \\
\textbf{MissCNF (ours)} & \underline{0.157}{\tiny$\pm$0.126} & \underline{0.046}{\tiny$\pm$0.012} & \underline{0.028}{\tiny$\pm$0.011} & 1.557{\tiny$\pm$0.487} & \textbf{0.124}{\tiny$\pm$0.088} & \textbf{0.175}{\tiny$\pm$0.121} & \textbf{0.040}{\tiny$\pm$0.019} & \textbf{0.460}{\tiny$\pm$0.058} \\
\bottomrule
\end{tabular}}
\caption{RMSE$_{\text{ATE}}$ across all 8 SCMs, both missingness mechanisms and three missingness rates. Values are mean $\pm$ std over seeds. Lower is better. Bold = best, underline = second-best, within each row-block (mechanism $\times$ rate); the full-data reference row (model trained with no missingness) is excluded from best/second-best marking.}
\label{tab:rmse_ate}
\end{table}

\begin{table}[H]
\centering
\scriptsize
\setlength{\tabcolsep}{2.5pt}
\resizebox{\textwidth}{!}{%
\begin{tabular}{lcccccccc}
\toprule
Method & Chain$_{\text{LIN}}$ & Fork$_{\text{LIN}}$ & Collider$_{\text{LIN}}$ & Triangle$_{\text{LIN}}$ & Chain$_{\text{NLIN}}$ & Fork$_{\text{NLIN}}$ & Collider$_{\text{NLIN}}$ & Triangle$_{\text{NLIN}}$ \\
\midrule
\textit{Full data (reference)} & \textit{0.067{\tiny$\pm$0.012}} & \textit{0.022{\tiny$\pm$0.002}} & \textit{0.016{\tiny$\pm$0.001}} & \textit{0.268{\tiny$\pm$0.037}} & \textit{0.036{\tiny$\pm$0.002}} & \textit{0.101{\tiny$\pm$0.011}} & \textit{0.034{\tiny$\pm$0.002}} & \textit{0.198{\tiny$\pm$0.037}} \\
\midrule
\multicolumn{9}{l}{\textit{MCAR, 30\%}} \\
Listwise deletion & 0.091{\tiny$\pm$0.017} & 0.028{\tiny$\pm$0.003} & 0.019{\tiny$\pm$0.001} & \underline{0.287}{\tiny$\pm$0.036} & \underline{0.049}{\tiny$\pm$0.006} & \underline{0.137}{\tiny$\pm$0.013} & \underline{0.038}{\tiny$\pm$0.005} & \underline{0.287}{\tiny$\pm$0.053} \\
Mean & 1.372{\tiny$\pm$0.037} & 0.218{\tiny$\pm$0.011} & 0.066{\tiny$\pm$0.003} & 3.381{\tiny$\pm$0.242} & 0.500{\tiny$\pm$0.011} & 1.082{\tiny$\pm$0.021} & 0.071{\tiny$\pm$0.008} & 1.825{\tiny$\pm$0.064} \\
MissForest & 0.390{\tiny$\pm$0.044} & 0.073{\tiny$\pm$0.006} & 0.029{\tiny$\pm$0.005} & 0.456{\tiny$\pm$0.202} & 0.121{\tiny$\pm$0.014} & 0.971{\tiny$\pm$0.086} & 0.050{\tiny$\pm$0.009} & 1.381{\tiny$\pm$0.307} \\
MICE & \textbf{0.075}{\tiny$\pm$0.010} & \underline{0.022}{\tiny$\pm$0.001} & \textbf{0.018}{\tiny$\pm$0.003} & \textbf{0.283}{\tiny$\pm$0.027} & 0.140{\tiny$\pm$0.004} & 1.526{\tiny$\pm$0.038} & 0.063{\tiny$\pm$0.010} & 2.178{\tiny$\pm$0.042} \\
KPI & 0.087{\tiny$\pm$0.017} & 0.033{\tiny$\pm$0.004} & 0.026{\tiny$\pm$0.006} & 0.911{\tiny$\pm$0.043} & 0.108{\tiny$\pm$0.008} & 0.360{\tiny$\pm$0.019} & 0.047{\tiny$\pm$0.008} & 1.789{\tiny$\pm$0.102} \\
DiffPuter & 0.087{\tiny$\pm$0.018} & \textbf{0.022}{\tiny$\pm$0.003} & 0.019{\tiny$\pm$0.002} & 0.386{\tiny$\pm$0.175} & 0.061{\tiny$\pm$0.006} & 0.215{\tiny$\pm$0.023} & 0.040{\tiny$\pm$0.005} & 1.067{\tiny$\pm$0.211} \\
O-MIRACLE & \underline{0.077}{\tiny$\pm$0.014} & 0.022{\tiny$\pm$0.002} & \underline{0.018}{\tiny$\pm$0.002} & 0.296{\tiny$\pm$0.022} & 0.057{\tiny$\pm$0.005} & 0.239{\tiny$\pm$0.008} & \textbf{0.038}{\tiny$\pm$0.004} & 0.784{\tiny$\pm$0.056} \\
\textbf{MissCNF (ours)} & 0.094{\tiny$\pm$0.014} & 0.023{\tiny$\pm$0.002} & 0.020{\tiny$\pm$0.001} & 0.310{\tiny$\pm$0.054} & \textbf{0.046}{\tiny$\pm$0.008} & \textbf{0.112}{\tiny$\pm$0.014} & 0.038{\tiny$\pm$0.003} & \textbf{0.246}{\tiny$\pm$0.043} \\
\addlinespace
\multicolumn{9}{l}{\textit{MCAR, 60\%}} \\
Listwise deletion & 0.192{\tiny$\pm$0.029} & 0.051{\tiny$\pm$0.007} & 0.026{\tiny$\pm$0.005} & 0.354{\tiny$\pm$0.109} & 0.103{\tiny$\pm$0.025} & \underline{0.236}{\tiny$\pm$0.038} & 0.054{\tiny$\pm$0.007} & \underline{0.457}{\tiny$\pm$0.053} \\
Mean & 2.406{\tiny$\pm$0.032} & 0.389{\tiny$\pm$0.031} & 0.113{\tiny$\pm$0.002} & 4.084{\tiny$\pm$0.268} & 0.808{\tiny$\pm$0.010} & 1.710{\tiny$\pm$0.035} & 0.133{\tiny$\pm$0.008} & 2.985{\tiny$\pm$0.095} \\
MissForest & 0.584{\tiny$\pm$0.021} & 0.114{\tiny$\pm$0.009} & 0.050{\tiny$\pm$0.004} & 2.583{\tiny$\pm$2.614} & 0.241{\tiny$\pm$0.029} & 1.568{\tiny$\pm$0.046} & 0.089{\tiny$\pm$0.014} & 3.070{\tiny$\pm$1.692} \\
MICE & \textbf{0.072}{\tiny$\pm$0.008} & \textbf{0.022}{\tiny$\pm$0.002} & \textbf{0.019}{\tiny$\pm$0.004} & \textbf{0.250}{\tiny$\pm$0.056} & 0.227{\tiny$\pm$0.012} & 2.605{\tiny$\pm$0.038} & 0.129{\tiny$\pm$0.010} & 3.180{\tiny$\pm$0.052} \\
KPI & 0.122{\tiny$\pm$0.015} & 0.049{\tiny$\pm$0.004} & 0.048{\tiny$\pm$0.005} & 1.557{\tiny$\pm$0.048} & 0.163{\tiny$\pm$0.009} & 0.620{\tiny$\pm$0.014} & 0.096{\tiny$\pm$0.017} & 2.865{\tiny$\pm$0.069} \\
DiffPuter & 0.341{\tiny$\pm$0.053} & 0.043{\tiny$\pm$0.004} & 0.026{\tiny$\pm$0.004} & 1.054{\tiny$\pm$0.309} & 0.134{\tiny$\pm$0.016} & 0.603{\tiny$\pm$0.030} & 0.057{\tiny$\pm$0.014} & 2.110{\tiny$\pm$0.372} \\
O-MIRACLE & \underline{0.087}{\tiny$\pm$0.010} & \underline{0.023}{\tiny$\pm$0.002} & \underline{0.020}{\tiny$\pm$0.003} & \underline{0.324}{\tiny$\pm$0.088} & \underline{0.076}{\tiny$\pm$0.003} & 0.386{\tiny$\pm$0.012} & \textbf{0.046}{\tiny$\pm$0.007} & 1.123{\tiny$\pm$0.054} \\
\textbf{MissCNF (ours)} & 0.108{\tiny$\pm$0.015} & 0.030{\tiny$\pm$0.002} & 0.025{\tiny$\pm$0.006} & 0.371{\tiny$\pm$0.072} & \textbf{0.065}{\tiny$\pm$0.014} & \textbf{0.146}{\tiny$\pm$0.005} & \underline{0.051}{\tiny$\pm$0.006} & \textbf{0.275}{\tiny$\pm$0.045} \\
\addlinespace
\multicolumn{9}{l}{\textit{MCAR, 90\%}} \\
Listwise deletion & 0.586{\tiny$\pm$0.204} & 0.125{\tiny$\pm$0.029} & 0.044{\tiny$\pm$0.006} & \underline{1.236}{\tiny$\pm$0.336} & 0.216{\tiny$\pm$0.050} & 0.733{\tiny$\pm$0.266} & 0.098{\tiny$\pm$0.015} & \underline{1.851}{\tiny$\pm$0.423} \\
Mean & 3.545{\tiny$\pm$0.124} & 0.645{\tiny$\pm$0.022} & 0.165{\tiny$\pm$0.004} & 5.155{\tiny$\pm$0.214} & 1.256{\tiny$\pm$0.067} & 2.539{\tiny$\pm$0.173} & 0.188{\tiny$\pm$0.007} & 4.380{\tiny$\pm$0.067} \\
MissForest & 0.893{\tiny$\pm$0.076} & 0.160{\tiny$\pm$0.006} & 0.078{\tiny$\pm$0.013} & 4.031{\tiny$\pm$3.657} & 0.431{\tiny$\pm$0.084} & 2.240{\tiny$\pm$0.126} & 0.126{\tiny$\pm$0.011} & 27.034{\tiny$\pm$52.249} \\
MICE & \textbf{0.087}{\tiny$\pm$0.032} & \textbf{0.027}{\tiny$\pm$0.006} & \textbf{0.019}{\tiny$\pm$0.005} & \textbf{0.422}{\tiny$\pm$0.122} & 0.326{\tiny$\pm$0.040} & 3.386{\tiny$\pm$0.046} & 0.203{\tiny$\pm$0.006} & 4.069{\tiny$\pm$0.069} \\
KPI & 0.245{\tiny$\pm$0.028} & 0.083{\tiny$\pm$0.006} & 0.123{\tiny$\pm$0.011} & 2.298{\tiny$\pm$0.068} & 0.225{\tiny$\pm$0.034} & 0.809{\tiny$\pm$0.048} & 0.277{\tiny$\pm$0.017} & 3.466{\tiny$\pm$0.064} \\
DiffPuter & 2.352{\tiny$\pm$0.265} & 0.270{\tiny$\pm$0.046} & 0.113{\tiny$\pm$0.010} & 39.859{\tiny$\pm$38.052} & 0.850{\tiny$\pm$0.079} & 2.745{\tiny$\pm$0.266} & 0.111{\tiny$\pm$0.021} & 6.204{\tiny$\pm$1.746} \\
O-MIRACLE & 0.297{\tiny$\pm$0.086} & \underline{0.042}{\tiny$\pm$0.007} & \underline{0.033}{\tiny$\pm$0.006} & 1.522{\tiny$\pm$0.298} & \underline{0.143}{\tiny$\pm$0.010} & \underline{0.642}{\tiny$\pm$0.069} & \textbf{0.066}{\tiny$\pm$0.004} & 2.219{\tiny$\pm$0.229} \\
\textbf{MissCNF (ours)} & \underline{0.206}{\tiny$\pm$0.028} & 0.051{\tiny$\pm$0.003} & 0.040{\tiny$\pm$0.006} & 1.633{\tiny$\pm$0.120} & \textbf{0.114}{\tiny$\pm$0.012} & \textbf{0.282}{\tiny$\pm$0.045} & \underline{0.072}{\tiny$\pm$0.011} & \textbf{0.678}{\tiny$\pm$0.159} \\
\addlinespace
\multicolumn{9}{l}{\textit{MAR, 30\%}} \\
Listwise deletion & 0.098{\tiny$\pm$0.015} & 0.032{\tiny$\pm$0.005} & 0.021{\tiny$\pm$0.002} & \underline{0.307}{\tiny$\pm$0.085} & \underline{0.053}{\tiny$\pm$0.008} & \underline{0.130}{\tiny$\pm$0.016} & \textbf{0.041}{\tiny$\pm$0.006} & \underline{0.251}{\tiny$\pm$0.017} \\
Mean & 1.323{\tiny$\pm$0.014} & 0.255{\tiny$\pm$0.030} & 0.091{\tiny$\pm$0.020} & 3.170{\tiny$\pm$0.303} & 0.538{\tiny$\pm$0.034} & 1.087{\tiny$\pm$0.025} & 0.115{\tiny$\pm$0.031} & 2.025{\tiny$\pm$0.064} \\
MissForest & 0.474{\tiny$\pm$0.029} & 0.085{\tiny$\pm$0.006} & 0.055{\tiny$\pm$0.018} & 0.482{\tiny$\pm$0.072} & 0.157{\tiny$\pm$0.011} & 1.115{\tiny$\pm$0.065} & 0.096{\tiny$\pm$0.029} & 1.302{\tiny$\pm$0.057} \\
MICE & \textbf{0.064}{\tiny$\pm$0.005} & \textbf{0.021}{\tiny$\pm$0.001} & \textbf{0.017}{\tiny$\pm$0.002} & \textbf{0.279}{\tiny$\pm$0.067} & 0.134{\tiny$\pm$0.011} & 1.543{\tiny$\pm$0.034} & 0.089{\tiny$\pm$0.015} & 2.338{\tiny$\pm$0.067} \\
KPI & 0.123{\tiny$\pm$0.022} & 0.047{\tiny$\pm$0.008} & 0.049{\tiny$\pm$0.014} & 0.946{\tiny$\pm$0.077} & 0.113{\tiny$\pm$0.008} & 0.402{\tiny$\pm$0.018} & 0.095{\tiny$\pm$0.016} & 1.895{\tiny$\pm$0.056} \\
DiffPuter & 0.224{\tiny$\pm$0.043} & 0.054{\tiny$\pm$0.022} & 0.047{\tiny$\pm$0.016} & 0.576{\tiny$\pm$0.167} & 0.126{\tiny$\pm$0.025} & 0.318{\tiny$\pm$0.048} & 0.087{\tiny$\pm$0.033} & 1.293{\tiny$\pm$0.470} \\
O-MIRACLE & 0.154{\tiny$\pm$0.044} & 0.040{\tiny$\pm$0.012} & 0.048{\tiny$\pm$0.014} & 0.321{\tiny$\pm$0.064} & 0.097{\tiny$\pm$0.021} & 0.271{\tiny$\pm$0.026} & 0.085{\tiny$\pm$0.030} & 0.937{\tiny$\pm$0.130} \\
\textbf{MissCNF (ours)} & \underline{0.082}{\tiny$\pm$0.005} & \underline{0.023}{\tiny$\pm$0.001} & \underline{0.019}{\tiny$\pm$0.002} & 0.327{\tiny$\pm$0.083} & \textbf{0.047}{\tiny$\pm$0.007} & \textbf{0.111}{\tiny$\pm$0.009} & \underline{0.043}{\tiny$\pm$0.005} & \textbf{0.240}{\tiny$\pm$0.062} \\
\addlinespace
\multicolumn{9}{l}{\textit{MAR, 60\%}} \\
Listwise deletion & 0.222{\tiny$\pm$0.047} & 0.071{\tiny$\pm$0.009} & 0.039{\tiny$\pm$0.007} & \underline{0.410}{\tiny$\pm$0.122} & \underline{0.124}{\tiny$\pm$0.030} & \underline{0.339}{\tiny$\pm$0.095} & \underline{0.079}{\tiny$\pm$0.012} & \underline{0.559}{\tiny$\pm$0.064} \\
Mean & 2.474{\tiny$\pm$0.082} & 0.461{\tiny$\pm$0.054} & 0.187{\tiny$\pm$0.035} & 4.356{\tiny$\pm$0.559} & 0.893{\tiny$\pm$0.065} & 1.865{\tiny$\pm$0.093} & 0.253{\tiny$\pm$0.131} & 3.368{\tiny$\pm$0.128} \\
MissForest & 0.799{\tiny$\pm$0.147} & 0.145{\tiny$\pm$0.019} & 0.113{\tiny$\pm$0.035} & 1.277{\tiny$\pm$0.592} & 0.326{\tiny$\pm$0.032} & 1.853{\tiny$\pm$0.156} & 0.196{\tiny$\pm$0.086} & 2.624{\tiny$\pm$0.613} \\
MICE & \textbf{0.072}{\tiny$\pm$0.012} & \textbf{0.022}{\tiny$\pm$0.002} & \textbf{0.018}{\tiny$\pm$0.004} & \textbf{0.254}{\tiny$\pm$0.060} & 0.224{\tiny$\pm$0.024} & 2.677{\tiny$\pm$0.065} & 0.165{\tiny$\pm$0.026} & 3.532{\tiny$\pm$0.083} \\
KPI & 0.185{\tiny$\pm$0.028} & 0.081{\tiny$\pm$0.018} & 0.096{\tiny$\pm$0.024} & 1.570{\tiny$\pm$0.045} & 0.178{\tiny$\pm$0.010} & 0.631{\tiny$\pm$0.033} & 0.210{\tiny$\pm$0.060} & 2.958{\tiny$\pm$0.097} \\
DiffPuter & 0.907{\tiny$\pm$0.225} & 0.181{\tiny$\pm$0.060} & 0.098{\tiny$\pm$0.034} & 2.924{\tiny$\pm$2.596} & 0.388{\tiny$\pm$0.131} & 1.215{\tiny$\pm$0.366} & 0.180{\tiny$\pm$0.086} & 2.357{\tiny$\pm$0.268} \\
O-MIRACLE & 0.345{\tiny$\pm$0.084} & 0.083{\tiny$\pm$0.035} & 0.088{\tiny$\pm$0.029} & 0.585{\tiny$\pm$0.287} & 0.206{\tiny$\pm$0.042} & 0.451{\tiny$\pm$0.085} & 0.167{\tiny$\pm$0.068} & 1.387{\tiny$\pm$0.141} \\
\textbf{MissCNF (ours)} & \underline{0.111}{\tiny$\pm$0.009} & \underline{0.033}{\tiny$\pm$0.006} & \underline{0.026}{\tiny$\pm$0.005} & 0.465{\tiny$\pm$0.284} & \textbf{0.062}{\tiny$\pm$0.012} & \textbf{0.147}{\tiny$\pm$0.004} & \textbf{0.052}{\tiny$\pm$0.006} & \textbf{0.325}{\tiny$\pm$0.055} \\
\addlinespace
\multicolumn{9}{l}{\textit{MAR, 90\%}} \\
Listwise deletion & 2.604{\tiny$\pm$1.905} & 0.343{\tiny$\pm$0.339} & 0.134{\tiny$\pm$0.063} & 1.714{\tiny$\pm$0.555} & 0.590{\tiny$\pm$0.175} & 1.364{\tiny$\pm$0.153} & 0.305{\tiny$\pm$0.128} & \underline{2.619}{\tiny$\pm$0.498} \\
Mean & 10.793{\tiny$\pm$10.221} & 0.915{\tiny$\pm$0.207} & 0.297{\tiny$\pm$0.079} & 12.767{\tiny$\pm$16.181} & 3.394{\tiny$\pm$4.000} & 3.271{\tiny$\pm$0.465} & 0.405{\tiny$\pm$0.179} & 5.483{\tiny$\pm$1.581} \\
MissForest & 1.160{\tiny$\pm$0.220} & 0.231{\tiny$\pm$0.049} & 0.160{\tiny$\pm$0.038} & 12.064{\tiny$\pm$21.237} & 0.478{\tiny$\pm$0.038} & 2.316{\tiny$\pm$0.174} & 0.357{\tiny$\pm$0.150} & 30.084{\tiny$\pm$23.658} \\
MICE & \textbf{0.089}{\tiny$\pm$0.017} & \textbf{0.026}{\tiny$\pm$0.002} & \textbf{0.018}{\tiny$\pm$0.004} & \textbf{0.475}{\tiny$\pm$0.219} & 0.320{\tiny$\pm$0.034} & 3.679{\tiny$\pm$0.093} & 0.287{\tiny$\pm$0.049} & 4.571{\tiny$\pm$0.210} \\
KPI & 0.287{\tiny$\pm$0.038} & 0.131{\tiny$\pm$0.022} & 0.178{\tiny$\pm$0.030} & 2.138{\tiny$\pm$0.097} & \underline{0.237}{\tiny$\pm$0.039} & \underline{0.896}{\tiny$\pm$0.087} & 0.428{\tiny$\pm$0.101} & 3.646{\tiny$\pm$0.052} \\
DiffPuter & 3.786{\tiny$\pm$1.512} & 0.593{\tiny$\pm$0.273} & 0.212{\tiny$\pm$0.059} & 10.625{\tiny$\pm$6.849} & 1.227{\tiny$\pm$0.412} & 2.891{\tiny$\pm$0.440} & 0.424{\tiny$\pm$0.286} & 32.592{\tiny$\pm$43.503} \\
O-MIRACLE & 1.567{\tiny$\pm$0.632} & 0.324{\tiny$\pm$0.233} & 0.139{\tiny$\pm$0.055} & 2.872{\tiny$\pm$0.645} & 0.684{\tiny$\pm$0.150} & 1.823{\tiny$\pm$0.387} & \underline{0.281}{\tiny$\pm$0.109} & 2.987{\tiny$\pm$0.368} \\
\textbf{MissCNF (ours)} & \underline{0.206}{\tiny$\pm$0.046} & \underline{0.058}{\tiny$\pm$0.008} & \underline{0.041}{\tiny$\pm$0.007} & \underline{1.683}{\tiny$\pm$0.478} & \textbf{0.141}{\tiny$\pm$0.044} & \textbf{0.362}{\tiny$\pm$0.065} & \textbf{0.073}{\tiny$\pm$0.012} & \textbf{0.860}{\tiny$\pm$0.253} \\
\bottomrule
\end{tabular}}
\caption{RMSE$_{\text{CF}}$ across all 8 SCMs, both missingness mechanisms and three missingness rates. Values are mean $\pm$ std over seeds. Lower is better. Bold = best, underline = second-best, within each row-block (mechanism $\times$ rate); the full-data reference row (model trained with no missingness) is excluded from best/second-best marking.}
\label{tab:rmse_cf}
\end{table}

\subsection{Robustness to Weaker Recovery Conditions (Section~\ref{sec:weaker-recovery})}
\label{app:full_results_weaker_recovery}

\subsubsection{From Causal-Family Coverage to Non-Identification (Section~\ref{sec:tier2-tier3})}
\label{app:full_results_tier2_vs_tier3}

\begin{table}[H]
\centering
\scriptsize
\begin{tabular}{llccccc}
\toprule
Task & Setting & MissCNF & MICE & O-MIRACLE & KPI & DiffPuter \\
\midrule
\multirow{3}{*}{Chain$_{\text{LIN}}$} & Full pos. & 0.023 $\pm$ 0.010 & \textbf{0.007 $\pm$ 0.005} & 1.27 $\pm$ 1.13 & 0.345 $\pm$ 0.102 & 5.20 $\pm$ 3.81 \\
 & Cov.\ hold & 0.068 $\pm$ 0.023 & \textbf{0.007 $\pm$ 0.005} & 2.70 $\pm$ 0.113 & 0.338 $\pm$ 0.026 & 14.4 $\pm$ 0.438 \\
 & Cov.\ violated & 0.385 $\pm$ 0.116 & \textbf{0.005 $\pm$ 0.004} & 439 $\pm$ 277 & 1309 $\pm$ 1082 & 1434 $\pm$ 116 \\
\midrule
\multirow{3}{*}{Chain$_{\text{NLIN}}$} & Full pos. & \textbf{0.023 $\pm$ 0.026} & 0.610 $\pm$ 0.213 & 0.759 $\pm$ 0.284 & 0.290 $\pm$ 0.129 & 1.30 $\pm$ 0.688 \\
 & Cov.\ hold & \textbf{0.024 $\pm$ 0.014} & 0.296 $\pm$ 0.021 & 1.63 $\pm$ 0.381 & 0.345 $\pm$ 0.047 & 4.64 $\pm$ 4.71 \\
 & Cov.\ violated & \textbf{0.034 $\pm$ 0.017} & 0.703 $\pm$ 0.042 & 15.9 $\pm$ 0.530 & 14.0 $\pm$ 2.70 & 19.0 $\pm$ 1.61 \\
\midrule
\multirow{3}{*}{Fork$_{\text{LIN}}$} & Full pos. & 0.025 $\pm$ 0.010 & \textbf{0.010 $\pm$ 0.004} & 0.670 $\pm$ 0.305 & 0.802 $\pm$ 0.114 & 1.36 $\pm$ 0.758 \\
 & Cov.\ hold & 0.023 $\pm$ 0.007 & \textbf{0.012 $\pm$ 0.003} & 1.54 $\pm$ 0.081 & 0.463 $\pm$ 0.023 & 15.0 $\pm$ 1.12 \\
 & Cov.\ violated & 0.134 $\pm$ 0.041 & \textbf{0.015 $\pm$ 0.005} & 74.9 $\pm$ 11.8 & 158 $\pm$ 53.5 & 134 $\pm$ 79.7 \\
\midrule
\multirow{3}{*}{Fork$_{\text{NLIN}}$} & Full pos. & \textbf{0.027 $\pm$ 0.008} & 17.6 $\pm$ 1.38 & 0.769 $\pm$ 0.240 & 1.15 $\pm$ 0.130 & 2.03 $\pm$ 0.981 \\
 & Cov.\ hold & \textbf{0.130 $\pm$ 0.071} & 11.4 $\pm$ 0.355 & 1.76 $\pm$ 0.050 & 1.72 $\pm$ 0.444 & 5.80 $\pm$ 0.184 \\
 & Cov.\ violated & \textbf{1.52 $\pm$ 1.25} & 14.7 $\pm$ 1.20 & 136 $\pm$ 15.8 & 173 $\pm$ 66.9 & 135 $\pm$ 15.3 \\
\bottomrule
\end{tabular}
\caption{KL across all 4 SCMs under 3 settings. Values are mean $\pm$ std over seeds. Lower is better. Bold = best,  within each row.}
\label{tab:tier2-tier3-kl}
\end{table}

\begin{table}[H]
\centering
\scriptsize
\begin{tabular}{llccccc}
\toprule
Task & Setting & MissCNF & MICE & O-MIRACLE & KPI & DiffPuter \\
\midrule
\multirow{3}{*}{Chain$_{\text{LIN}}$} & Full pos. & 0.078 $\pm$ 0.035 & \textbf{0.072 $\pm$ 0.021} & 0.083 $\pm$ 0.029 & 0.102 $\pm$ 0.036 & 0.504 $\pm$ 0.234 \\
 & Cov.\ hold & 0.117 $\pm$ 0.055 & \textbf{0.057 $\pm$ 0.030} & 0.111 $\pm$ 0.044 & 0.148 $\pm$ 0.065 & 1.44 $\pm$ 0.196 \\
 & Cov.\ violated & 0.157 $\pm$ 0.057 & \textbf{0.064 $\pm$ 0.026} & 2.01 $\pm$ 0.096 & 29.1 $\pm$ 43.7 & 3.51 $\pm$ 0.248 \\
\midrule
\multirow{3}{*}{Chain$_{\text{NLIN}}$} & Full pos. & 0.041 $\pm$ 0.011 & 0.069 $\pm$ 0.007 & \textbf{0.034 $\pm$ 0.006} & 0.134 $\pm$ 0.019 & 0.165 $\pm$ 0.197 \\
 & Cov.\ hold & \textbf{0.041 $\pm$ 0.011} & 0.062 $\pm$ 0.013 & 0.067 $\pm$ 0.014 & 0.079 $\pm$ 0.005 & 0.537 $\pm$ 0.017 \\
 & Cov.\ violated & \textbf{0.042 $\pm$ 0.018} & 0.176 $\pm$ 0.017 & 0.470 $\pm$ 0.021 & 0.468 $\pm$ 0.020 & 0.985 $\pm$ 0.028 \\
\midrule
\multirow{3}{*}{Fork$_{\text{LIN}}$} & Full pos. & 0.037 $\pm$ 0.002 & 0.035 $\pm$ 0.004 & \textbf{0.034 $\pm$ 0.005} & 0.052 $\pm$ 0.007 & 0.073 $\pm$ 0.042 \\
 & Cov.\ hold & 0.034 $\pm$ 0.002 & 0.034 $\pm$ 0.004 & 0.052 $\pm$ 0.008 & \textbf{0.021 $\pm$ 0.003} & 0.584 $\pm$ 0.049 \\
 & Cov.\ violated & 0.042 $\pm$ 0.007 & \textbf{0.032 $\pm$ 0.003} & 0.051 $\pm$ 0.010 & 0.084 $\pm$ 0.018 & 0.526 $\pm$ 0.042 \\
\midrule
\multirow{3}{*}{Fork$_{\text{NLIN}}$} & Full pos. & \textbf{0.068 $\pm$ 0.014} & 1.94 $\pm$ 0.094 & 0.120 $\pm$ 0.050 & 0.211 $\pm$ 0.014 & 0.580 $\pm$ 0.218 \\
 & Cov.\ hold & \textbf{0.063 $\pm$ 0.018} & 1.97 $\pm$ 0.093 & 0.429 $\pm$ 0.040 & 0.686 $\pm$ 0.099 & 0.899 $\pm$ 0.046 \\
 & Cov.\ violated & 0.366 $\pm$ 0.195 & 1.49 $\pm$ 0.072 & 0.731 $\pm$ 0.030 & \textbf{0.345 $\pm$ 0.077} & 0.438 $\pm$ 0.078 \\
\bottomrule
\end{tabular}
\caption{RMSE$_{\text{ATE}}$ across all 4 SCMs under 3 settings. Values are mean $\pm$ std over seeds. Lower is better. Bold = best,  within each row.}
\label{tab:tier2-tier3-ate}
\end{table}

\begin{table}[H]
\centering
\scriptsize
\begin{tabular}{llccccc}
\toprule
Task & Setting & MissCNF & MICE & O-MIRACLE & KPI & DiffPuter \\
\midrule
\multirow{3}{*}{Chain$_{\text{LIN}}$} & Full pos. & 0.104 $\pm$ 0.006 & \textbf{0.075 $\pm$ 0.007} & 0.264 $\pm$ 0.058 & 0.166 $\pm$ 0.028 & 0.569 $\pm$ 0.129 \\
 & Cov.\ hold & 0.133 $\pm$ 0.029 & \textbf{0.068 $\pm$ 0.010} & 0.472 $\pm$ 0.009 & 0.277 $\pm$ 0.014 & 1.40 $\pm$ 0.071 \\
 & Cov.\ violated & 0.200 $\pm$ 0.030 & \textbf{0.076 $\pm$ 0.011} & 5.69 $\pm$ 1.62 & 33.4 $\pm$ 18.8 & 92.3 $\pm$ 71.5 \\
\midrule
\multirow{3}{*}{Chain$_{\text{NLIN}}$} & Full pos. & \textbf{0.057 $\pm$ 0.004} & 0.207 $\pm$ 0.017 & 0.163 $\pm$ 0.033 & 0.163 $\pm$ 0.008 & 0.267 $\pm$ 0.129 \\
 & Cov.\ hold & \textbf{0.057 $\pm$ 0.010} & 0.204 $\pm$ 0.016 & 0.258 $\pm$ 0.012 & 0.159 $\pm$ 0.012 & 0.707 $\pm$ 0.086 \\
 & Cov.\ violated & \textbf{0.058 $\pm$ 0.011} & 0.296 $\pm$ 0.033 & 0.965 $\pm$ 0.034 & 0.813 $\pm$ 0.034 & 1.52 $\pm$ 0.135 \\
\midrule
\multirow{3}{*}{Fork$_{\text{LIN}}$} & Full pos. & 0.029 $\pm$ 0.003 & \textbf{0.022 $\pm$ 0.002} & 0.063 $\pm$ 0.023 & 0.069 $\pm$ 0.016 & 0.121 $\pm$ 0.045 \\
 & Cov.\ hold & 0.030 $\pm$ 0.003 & \textbf{0.022 $\pm$ 0.002} & 0.068 $\pm$ 0.006 & 0.062 $\pm$ 0.003 & 0.369 $\pm$ 0.011 \\
 & Cov.\ violated & 0.064 $\pm$ 0.005 & \textbf{0.020 $\pm$ 0.003} & 0.180 $\pm$ 0.015 & 0.184 $\pm$ 0.044 & 0.478 $\pm$ 0.053 \\
\midrule
\multirow{3}{*}{Fork$_{\text{NLIN}}$} & Full pos. & \textbf{0.129 $\pm$ 0.010} & 2.32 $\pm$ 0.038 & 0.366 $\pm$ 0.040 & 0.562 $\pm$ 0.030 & 0.824 $\pm$ 0.248 \\
 & Cov.\ hold & \textbf{0.153 $\pm$ 0.017} & 1.97 $\pm$ 0.040 & 0.264 $\pm$ 0.018 & 0.424 $\pm$ 0.102 & 1.49 $\pm$ 0.011 \\
 & Cov.\ violated & \textbf{0.407 $\pm$ 0.156} & 2.05 $\pm$ 0.067 & 1.93 $\pm$ 0.022 & 2.04 $\pm$ 0.193 & 2.39 $\pm$ 0.067 \\
\bottomrule
\end{tabular}
\caption{RMSE$_{\text{CF}}$ across all 4 SCMs under 3 settings. Values are mean $\pm$ std over seeds. Lower is better. Bold = best,  within each row.}
\label{tab:tier2-tier3-cf}
\end{table}

\begin{figure}[H]
    \centering
    \includegraphics[width=1.0\linewidth]
    {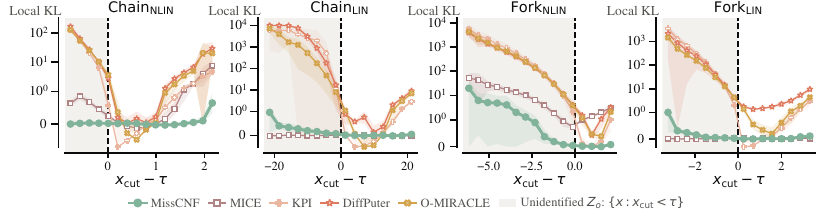}
    \caption{
    Local KL error (defined in Appendix~\ref{app:local-kl}) as a function of the distance to the identification boundary. The dashed line denotes the boundary $x_{\mathrm{rel}}=0$. Negative values correspond to the revealed side, while positive values lie inside the non-identified region $Z_o$. Curves show mean $\pm$ std over seeds.
    }
    \label{fig:tier3_boundary}
\end{figure}

\subsubsection{MNAR Missingness (Section~\ref{sec:mnar})}
\label{app:full_results_mnar}
\begin{table}[H]
\centering
\scriptsize
\setlength{\tabcolsep}{2.5pt}
\begin{tabular}{lcccccccc}
\toprule
Method & Chain$_{\text{LIN}}$ & Fork$_{\text{LIN}}$ & Collider$_{\text{LIN}}$ & Triangle$_{\text{LIN}}$ & Chain$_{\text{NLIN}}$ & Fork$_{\text{NLIN}}$ & Collider$_{\text{NLIN}}$ & Triangle$_{\text{NLIN}}$ \\
\midrule
\multicolumn{9}{l}{\textit{MNAR (self-masking), 30\%}} \\
MICE & \textbf{0.020}{\tiny$\pm$0.012} & \textbf{0.013}{\tiny$\pm$0.004} & \textbf{0.035}{\tiny$\pm$0.026} & \textbf{0.008}{\tiny$\pm$0.006} & 0.446{\tiny$\pm$0.204} & 9.367{\tiny$\pm$0.115} & \underline{0.087}{\tiny$\pm$0.050} & 9.131{\tiny$\pm$0.391} \\
KPI & 0.132{\tiny$\pm$0.026} & 0.383{\tiny$\pm$0.031} & 0.131{\tiny$\pm$0.056} & 0.121{\tiny$\pm$0.024} & \underline{0.261}{\tiny$\pm$0.149} & 0.678{\tiny$\pm$0.047} & 0.149{\tiny$\pm$0.070} & 2.951{\tiny$\pm$0.308} \\
O-MIRACLE & 0.210{\tiny$\pm$0.030} & 0.204{\tiny$\pm$0.063} & 0.152{\tiny$\pm$0.062} & 0.168{\tiny$\pm$0.025} & 0.312{\tiny$\pm$0.167} & \underline{0.426}{\tiny$\pm$0.070} & 0.171{\tiny$\pm$0.075} & \underline{0.921}{\tiny$\pm$0.212} \\
\textbf{MissCNF (ours)} & \underline{0.028}{\tiny$\pm$0.014} & \underline{0.021}{\tiny$\pm$0.014} & \underline{0.037}{\tiny$\pm$0.025} & \underline{0.014}{\tiny$\pm$0.008} & \textbf{0.054}{\tiny$\pm$0.033} & \textbf{0.022}{\tiny$\pm$0.005} & \textbf{0.060}{\tiny$\pm$0.040} & \textbf{0.078}{\tiny$\pm$0.043} \\
\addlinespace
\multicolumn{9}{l}{\textit{MNAR (self-masking), 60\%}} \\
MICE & \textbf{0.058}{\tiny$\pm$0.046} & \textbf{0.030}{\tiny$\pm$0.010} & \textbf{0.110}{\tiny$\pm$0.080} & \textbf{0.014}{\tiny$\pm$0.010} & 0.931{\tiny$\pm$0.279} & 19.116{\tiny$\pm$0.236} & \underline{0.251}{\tiny$\pm$0.146} & 16.060{\tiny$\pm$0.589} \\
KPI & 0.685{\tiny$\pm$0.220} & 0.997{\tiny$\pm$0.158} & 0.731{\tiny$\pm$0.349} & 0.650{\tiny$\pm$0.268} & \underline{0.687}{\tiny$\pm$0.301} & \underline{1.574}{\tiny$\pm$0.136} & 0.785{\tiny$\pm$0.399} & 5.596{\tiny$\pm$0.778} \\
O-MIRACLE & 3.808{\tiny$\pm$4.190} & 1.011{\tiny$\pm$0.386} & 0.956{\tiny$\pm$0.399} & 3.936{\tiny$\pm$4.829} & 2.749{\tiny$\pm$2.780} & 1.700{\tiny$\pm$0.423} & 0.921{\tiny$\pm$0.372} & \underline{3.839}{\tiny$\pm$2.393} \\
\textbf{MissCNF (ours)} & \underline{0.083}{\tiny$\pm$0.052} & \underline{0.048}{\tiny$\pm$0.025} & \underline{0.117}{\tiny$\pm$0.094} & \underline{0.031}{\tiny$\pm$0.016} & \textbf{0.121}{\tiny$\pm$0.042} & \textbf{0.066}{\tiny$\pm$0.030} & \textbf{0.177}{\tiny$\pm$0.127} & \textbf{0.288}{\tiny$\pm$0.280} \\
\addlinespace
\bottomrule
\end{tabular}
\caption{KL across all 8 SCMs under self-masking MNAR at two missing rates. Values are mean $\pm$ std over seeds. Lower is better. Bold = best, underline = second-best, within each row-block (missing rate).}
\label{tab:mnar_kl_distance_mean}
\end{table}

\begin{table}[H]
\centering
\scriptsize
\setlength{\tabcolsep}{2.5pt}
\begin{tabular}{lcccccccc}
\toprule
Method & Chain$_{\text{LIN}}$ & Fork$_{\text{LIN}}$ & Collider$_{\text{LIN}}$ & Triangle$_{\text{LIN}}$ & Chain$_{\text{NLIN}}$ & Fork$_{\text{NLIN}}$ & Collider$_{\text{NLIN}}$ & Triangle$_{\text{NLIN}}$ \\
\midrule
\multicolumn{9}{l}{\textit{MNAR (self-masking), 30\%}} \\
MICE & \underline{0.086}{\tiny$\pm$0.033} & \textbf{0.034}{\tiny$\pm$0.003} & \textbf{0.014}{\tiny$\pm$0.006} & \textbf{0.312}{\tiny$\pm$0.076} & 0.077{\tiny$\pm$0.040} & 1.385{\tiny$\pm$0.032} & 0.029{\tiny$\pm$0.006} & 1.698{\tiny$\pm$0.101} \\
KPI & 0.100{\tiny$\pm$0.045} & 0.046{\tiny$\pm$0.005} & 0.015{\tiny$\pm$0.005} & 0.872{\tiny$\pm$0.058} & 0.113{\tiny$\pm$0.014} & 0.192{\tiny$\pm$0.081} & \textbf{0.024}{\tiny$\pm$0.008} & 1.304{\tiny$\pm$0.274} \\
O-MIRACLE & \textbf{0.083}{\tiny$\pm$0.045} & \underline{0.035}{\tiny$\pm$0.004} & \underline{0.014}{\tiny$\pm$0.006} & \underline{0.349}{\tiny$\pm$0.097} & \underline{0.054}{\tiny$\pm$0.024} & \underline{0.163}{\tiny$\pm$0.062} & \underline{0.025}{\tiny$\pm$0.004} & \underline{0.477}{\tiny$\pm$0.116} \\
\textbf{MissCNF (ours)} & 0.086{\tiny$\pm$0.037} & 0.035{\tiny$\pm$0.004} & 0.015{\tiny$\pm$0.006} & 0.359{\tiny$\pm$0.117} & \textbf{0.049}{\tiny$\pm$0.015} & \textbf{0.074}{\tiny$\pm$0.032} & 0.025{\tiny$\pm$0.012} & \textbf{0.216}{\tiny$\pm$0.092} \\
\addlinespace
\multicolumn{9}{l}{\textit{MNAR (self-masking), 60\%}} \\
MICE & \textbf{0.098}{\tiny$\pm$0.070} & \textbf{0.035}{\tiny$\pm$0.004} & \underline{0.018}{\tiny$\pm$0.007} & \textbf{0.241}{\tiny$\pm$0.090} & 0.097{\tiny$\pm$0.025} & 2.068{\tiny$\pm$0.078} & 0.049{\tiny$\pm$0.007} & 2.185{\tiny$\pm$0.100} \\
KPI & 0.125{\tiny$\pm$0.048} & 0.055{\tiny$\pm$0.005} & 0.022{\tiny$\pm$0.014} & 1.490{\tiny$\pm$0.033} & 0.137{\tiny$\pm$0.018} & \underline{0.286}{\tiny$\pm$0.108} & 0.033{\tiny$\pm$0.013} & 2.833{\tiny$\pm$0.146} \\
O-MIRACLE & 0.140{\tiny$\pm$0.062} & 0.046{\tiny$\pm$0.010} & \textbf{0.016}{\tiny$\pm$0.010} & 0.498{\tiny$\pm$0.223} & \textbf{0.052}{\tiny$\pm$0.026} & 0.417{\tiny$\pm$0.196} & \underline{0.027}{\tiny$\pm$0.006} & \underline{0.770}{\tiny$\pm$0.203} \\
\textbf{MissCNF (ours)} & \underline{0.106}{\tiny$\pm$0.064} & \underline{0.042}{\tiny$\pm$0.013} & 0.019{\tiny$\pm$0.016} & \underline{0.440}{\tiny$\pm$0.296} & \underline{0.076}{\tiny$\pm$0.026} & \textbf{0.136}{\tiny$\pm$0.033} & \textbf{0.025}{\tiny$\pm$0.010} & \textbf{0.283}{\tiny$\pm$0.073} \\
\addlinespace
\bottomrule
\end{tabular}
\caption{RMSE$_{\text{ATE}}$ across all 8 SCMs under self-masking MNAR at two missing rates. Values are mean $\pm$ std over seeds. Lower is better. Bold = best, underline = second-best, within each row-block (missing rate).}
\label{tab:mnar_rmse_ate_mean_mean}
\end{table}

\begin{table}[H]
\centering
\scriptsize
\setlength{\tabcolsep}{2.5pt}
\begin{tabular}{lcccccccc}
\toprule
Method & Chain$_{\text{LIN}}$ & Fork$_{\text{LIN}}$ & Collider$_{\text{LIN}}$ & Triangle$_{\text{LIN}}$ & Chain$_{\text{NLIN}}$ & Fork$_{\text{NLIN}}$ & Collider$_{\text{NLIN}}$ & Triangle$_{\text{NLIN}}$ \\
\midrule
\multicolumn{9}{l}{\textit{MNAR (self-masking), 30\%}} \\
MICE & \textbf{0.090}{\tiny$\pm$0.033} & \textbf{0.022}{\tiny$\pm$0.001} & \textbf{0.019}{\tiny$\pm$0.004} & \textbf{0.268}{\tiny$\pm$0.053} & 0.149{\tiny$\pm$0.009} & 1.654{\tiny$\pm$0.027} & 0.091{\tiny$\pm$0.015} & 1.928{\tiny$\pm$0.057} \\
KPI & 0.129{\tiny$\pm$0.040} & 0.039{\tiny$\pm$0.001} & 0.031{\tiny$\pm$0.007} & 0.906{\tiny$\pm$0.082} & 0.132{\tiny$\pm$0.008} & 0.414{\tiny$\pm$0.037} & 0.066{\tiny$\pm$0.020} & 1.634{\tiny$\pm$0.078} \\
O-MIRACLE & 0.146{\tiny$\pm$0.048} & 0.029{\tiny$\pm$0.004} & 0.025{\tiny$\pm$0.004} & 0.330{\tiny$\pm$0.056} & \underline{0.104}{\tiny$\pm$0.019} & \underline{0.325}{\tiny$\pm$0.031} & \underline{0.053}{\tiny$\pm$0.011} & \underline{0.920}{\tiny$\pm$0.190} \\
\textbf{MissCNF (ours)} & \underline{0.101}{\tiny$\pm$0.034} & \underline{0.025}{\tiny$\pm$0.001} & \underline{0.023}{\tiny$\pm$0.003} & \underline{0.328}{\tiny$\pm$0.096} & \textbf{0.059}{\tiny$\pm$0.003} & \textbf{0.113}{\tiny$\pm$0.013} & \textbf{0.045}{\tiny$\pm$0.009} & \textbf{0.317}{\tiny$\pm$0.086} \\
\addlinespace
\multicolumn{9}{l}{\textit{MNAR (self-masking), 60\%}} \\
MICE & \textbf{0.107}{\tiny$\pm$0.058} & \textbf{0.022}{\tiny$\pm$0.002} & \textbf{0.022}{\tiny$\pm$0.003} & \textbf{0.240}{\tiny$\pm$0.072} & 0.242{\tiny$\pm$0.020} & 2.738{\tiny$\pm$0.040} & 0.149{\tiny$\pm$0.017} & 2.866{\tiny$\pm$0.042} \\
KPI & 0.193{\tiny$\pm$0.052} & 0.055{\tiny$\pm$0.003} & 0.062{\tiny$\pm$0.009} & 1.564{\tiny$\pm$0.032} & \underline{0.199}{\tiny$\pm$0.022} & 0.634{\tiny$\pm$0.049} & 0.140{\tiny$\pm$0.029} & 2.829{\tiny$\pm$0.161} \\
O-MIRACLE & 0.326{\tiny$\pm$0.081} & 0.043{\tiny$\pm$0.010} & 0.039{\tiny$\pm$0.014} & 0.632{\tiny$\pm$0.307} & 0.200{\tiny$\pm$0.033} & \underline{0.567}{\tiny$\pm$0.069} & \underline{0.075}{\tiny$\pm$0.030} & \underline{1.314}{\tiny$\pm$0.120} \\
\textbf{MissCNF (ours)} & \underline{0.134}{\tiny$\pm$0.051} & \underline{0.032}{\tiny$\pm$0.003} & \underline{0.031}{\tiny$\pm$0.011} & \underline{0.437}{\tiny$\pm$0.244} & \textbf{0.082}{\tiny$\pm$0.015} & \textbf{0.178}{\tiny$\pm$0.032} & \textbf{0.052}{\tiny$\pm$0.010} & \textbf{0.390}{\tiny$\pm$0.054} \\
\addlinespace
\bottomrule
\end{tabular}
\caption{RMSE$_{\text{CF}}$ across all 8 SCMs under self-masking MNAR at two missing rates. Values are mean $\pm$ std over seeds. Lower is better. Bold = best, underline = second-best, within each row-block (missing rate).}
\label{tab:mnar_rmse_cf_mean_mean}
\end{table}

\subsection{Monte Carlo Approximation and Computational Cost (Section~\ref{sec:mc-computation})}
\label{app:full_results_mc_ablation}

\begin{table}[H]
\centering
\scriptsize
\setlength{\tabcolsep}{2.5pt}
\begin{tabular}{rcccccccc}
\toprule
 & \multicolumn{4}{c}{Fork$_{\text{NLIN}}$} & \multicolumn{4}{c}{Triangle$_{\text{NLIN}}$} \\
\cmidrule(lr){2-5}\cmidrule(lr){6-9}
$K$ & KL & RMSE$_{\text{ATE}}$ & RMSE$_{\text{CF}}$ & Time (s) & KL & RMSE$_{\text{ATE}}$ & RMSE$_{\text{CF}}$ & Time (s) \\
\midrule
32 & 0.054{\tiny$\pm$0.015} & 0.078{\tiny$\pm$0.026} & 0.170{\tiny$\pm$0.012} & 272{\tiny$\pm$3} & 0.086{\tiny$\pm$0.028} & 0.248{\tiny$\pm$0.118} & 0.433{\tiny$\pm$0.102} & 272{\tiny$\pm$6} \\
64 & 0.041{\tiny$\pm$0.015} & 0.085{\tiny$\pm$0.028} & 0.157{\tiny$\pm$0.006} & 273{\tiny$\pm$4} & 0.066{\tiny$\pm$0.023} & 0.190{\tiny$\pm$0.086} & 0.343{\tiny$\pm$0.073} & 270{\tiny$\pm$4} \\
128 & 0.036{\tiny$\pm$0.015} & 0.083{\tiny$\pm$0.024} & 0.154{\tiny$\pm$0.006} & 284{\tiny$\pm$4} & 0.055{\tiny$\pm$0.022} & 0.251{\tiny$\pm$0.093} & 0.374{\tiny$\pm$0.116} & 284{\tiny$\pm$4} \\
256 & 0.032{\tiny$\pm$0.012} & 0.076{\tiny$\pm$0.024} & 0.149{\tiny$\pm$0.005} & 328{\tiny$\pm$4} & 0.056{\tiny$\pm$0.023} & 0.193{\tiny$\pm$0.053} & 0.322{\tiny$\pm$0.019} & 329{\tiny$\pm$4} \\
512 & 0.032{\tiny$\pm$0.011} & 0.083{\tiny$\pm$0.024} & 0.147{\tiny$\pm$0.004} & 416{\tiny$\pm$4} & 0.052{\tiny$\pm$0.032} & 0.190{\tiny$\pm$0.043} & 0.325{\tiny$\pm$0.055} & 426{\tiny$\pm$13} \\
1024 & 0.034{\tiny$\pm$0.011} & 0.082{\tiny$\pm$0.021} & 0.150{\tiny$\pm$0.005} & 602{\tiny$\pm$21} & 0.058{\tiny$\pm$0.035} & 0.181{\tiny$\pm$0.059} & 0.328{\tiny$\pm$0.045} & 607{\tiny$\pm$9} \\
\bottomrule
\end{tabular}
\caption{Monte Carlo ablation on Fork$_{\text{NLIN}}$ and Triangle$_{\text{NLIN}}$. $K$ is the number of Monte Carlo samples used to estimate the observed-data marginal likelihood. Values are mean $\pm$ std over seeds. Time is total training time in seconds.}
\label{tab:mc_ablation_full}
\end{table}

\section{Additional experiments}
\label{app:add_exp}

\subsection{Controlled Nonlinearity Study (mentioned in Section~\ref{sec:exp1})}
\label{app:full_results_nonlinearity}
We construct a controlled SCM in which the degree of nonlinearity can be varied continuously while keeping the causal graph and noise distribution fixed.
We use the $\text{Chain}_\alpha$ and $\text{Collider}_\alpha$ defined in Appendix~\ref{app:dataset_nonlinearity} where $\alpha$ controls the nonlinearity of the SCM. At $\alpha=0$, the SCM is exactly linear, while at $\alpha=1$, it is purely quadratic. We evaluate $\alpha\in\{0,0.25,0.5,0.75,1\}$.
We consider MCAR and MAR with missing rates of $30\%$ and $60\%$, using the same training setup and baselines as in the previous experiment, except we only run for 3 seeds. We exclude Mean imputation and MissForest due to their poor performance. 

\begin{figure}[H]
    \centering
    \includegraphics[]{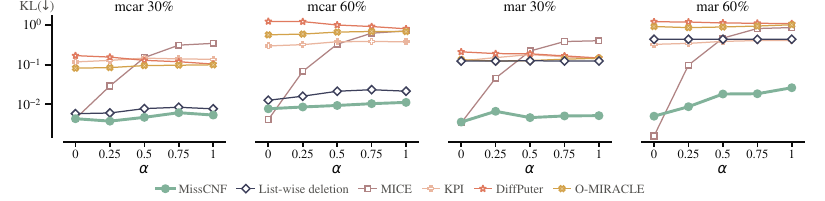}
    \caption{KL to the ground-truth distribution as $\alpha$ increases, on $\text{Chain}_\alpha$ under MCAR and MAR at 30\% and 60\% missing rates (mean ± std over seeds). $\alpha$ interpolates the $X_2 \rightarrow X_3$ mechanism from purely linear ($\alpha=0$) to purely quadratic ($\alpha=1$).}
    \label{appfig:tieralpha_chain-alpha_kl_distance}
\end{figure}

Figure~\ref{appfig:tieralpha_chain-alpha_kl_distance} shows the KL divergence, while full metrics are reported in Tables~\ref{tab:alpha_chain_alpha_kl_distance_mean}, ~\ref{tab:alpha_chain_alpha_rmse_ate_mean_mean} and ~\ref{tab:alpha_chain_alpha_rmse_cf_mean_mean} (Appendix~\ref{app:full_results_nonlinearity}). At $\alpha=0$ (fully linear), MICE matches or outperforms MissCNF. Once nonlinearity is introduced, its KL divergence increases rapidly, whereas MissCNF remains stable. Under MAR with $30\%$ missingness, the ratio between the KL divergence of MICE and MissCNF increases from $0.95\times$ at $\alpha=0$ to $6.9\times$, $48.0\times$, $75.6\times$, and $78.0\times$ at $\alpha=0.25,0.5,0.75,$ and $1$, respectively. The same qualitative behavior is observed under MCAR and at the higher missing rate. This controlled experiment therefore further confirms the linear--nonlinear behaviour in the previous benchmark. The same crossover appears in ATE and RMSE$_{\text{CF}}$, though the gap is smaller (up to $3.2\times$ and $6.9\times$, respectively, at $\alpha=1$) since these estimands depend only on specific functionals of the joint distribution rather than the full quantity. The remaining baselines (KPI, DiffPuter, O-MIRACLE) make no linearity assumption, and their KL divergence stays comparatively flat across $\alpha$, albeit at a worse level. We show full results in the tables below, along with a similar study on $\text{Collider}_\alpha$.

\begin{table}[H]
\centering
\scriptsize
\setlength{\tabcolsep}{4pt}
\begin{tabular}{lccccc}
\toprule
Method & $\alpha=0.00$ & $\alpha=0.25$ & $\alpha=0.50$ & $\alpha=0.75$ & $\alpha=1.00$ \\
\midrule
\multicolumn{6}{l}{\textit{MCAR, 30\%}} \\
Listwise deletion & 0.006{\tiny$\pm$0.002} & \underline{0.006}{\tiny$\pm$0.001} & \underline{0.008}{\tiny$\pm$0.002} & \underline{0.008}{\tiny$\pm$0.003} & \underline{0.008}{\tiny$\pm$0.001} \\
MICE & \underline{0.004}{\tiny$\pm$0.001} & 0.029{\tiny$\pm$0.004} & 0.151{\tiny$\pm$0.013} & 0.312{\tiny$\pm$0.021} & 0.345{\tiny$\pm$0.021} \\
KPI & 0.118{\tiny$\pm$0.012} & 0.128{\tiny$\pm$0.014} & 0.146{\tiny$\pm$0.014} & 0.140{\tiny$\pm$0.014} & 0.137{\tiny$\pm$0.016} \\
DiffPuter & 0.169{\tiny$\pm$0.007} & 0.156{\tiny$\pm$0.006} & 0.129{\tiny$\pm$0.020} & 0.119{\tiny$\pm$0.015} & 0.104{\tiny$\pm$0.012} \\
O-MIRACLE & 0.082{\tiny$\pm$0.004} & 0.084{\tiny$\pm$0.006} & 0.095{\tiny$\pm$0.011} & 0.097{\tiny$\pm$0.011} & 0.099{\tiny$\pm$0.008} \\
\textbf{MissCNF (ours)} & \textbf{0.004}{\tiny$\pm$0.001} & \textbf{0.004}{\tiny$\pm$0.001} & \textbf{0.005}{\tiny$\pm$0.002} & \textbf{0.006}{\tiny$\pm$0.003} & \textbf{0.005}{\tiny$\pm$0.002} \\
\addlinespace
\multicolumn{6}{l}{\textit{MCAR, 60\%}} \\
Listwise deletion & 0.013{\tiny$\pm$0.001} & \underline{0.016}{\tiny$\pm$0.001} & \underline{0.021}{\tiny$\pm$0.004} & \underline{0.023}{\tiny$\pm$0.002} & \underline{0.021}{\tiny$\pm$0.004} \\
MICE & \textbf{0.004}{\tiny$\pm$0.002} & 0.068{\tiny$\pm$0.003} & 0.330{\tiny$\pm$0.020} & 0.624{\tiny$\pm$0.023} & 0.714{\tiny$\pm$0.020} \\
KPI & 0.297{\tiny$\pm$0.025} & 0.321{\tiny$\pm$0.022} & 0.380{\tiny$\pm$0.027} & 0.386{\tiny$\pm$0.016} & 0.378{\tiny$\pm$0.009} \\
DiffPuter & 1.242{\tiny$\pm$0.148} & 1.224{\tiny$\pm$0.096} & 1.006{\tiny$\pm$0.024} & 0.920{\tiny$\pm$0.096} & 0.806{\tiny$\pm$0.043} \\
O-MIRACLE & 0.567{\tiny$\pm$0.015} & 0.593{\tiny$\pm$0.016} & 0.666{\tiny$\pm$0.041} & 0.689{\tiny$\pm$0.042} & 0.686{\tiny$\pm$0.026} \\
\textbf{MissCNF (ours)} & \underline{0.008}{\tiny$\pm$0.002} & \textbf{0.009}{\tiny$\pm$0.002} & \textbf{0.009}{\tiny$\pm$0.002} & \textbf{0.010}{\tiny$\pm$0.002} & \textbf{0.011}{\tiny$\pm$0.002} \\
\addlinespace

\multicolumn{6}{l}{\textit{MAR, 30\%}} \\
Listwise deletion & 0.123{\tiny$\pm$0.023} & 0.124{\tiny$\pm$0.026} & \underline{0.125}{\tiny$\pm$0.024} & \underline{0.124}{\tiny$\pm$0.025} & \underline{0.124}{\tiny$\pm$0.025} \\
MICE & \textbf{0.003}{\tiny$\pm$0.002} & \underline{0.046}{\tiny$\pm$0.003} & 0.221{\tiny$\pm$0.006} & 0.384{\tiny$\pm$0.006} & 0.403{\tiny$\pm$0.014} \\
KPI & 0.125{\tiny$\pm$0.004} & 0.151{\tiny$\pm$0.004} & 0.180{\tiny$\pm$0.006} & 0.158{\tiny$\pm$0.009} & 0.140{\tiny$\pm$0.008} \\
DiffPuter & 0.211{\tiny$\pm$0.043} & 0.191{\tiny$\pm$0.047} & 0.188{\tiny$\pm$0.044} & 0.167{\tiny$\pm$0.028} & 0.147{\tiny$\pm$0.031} \\
O-MIRACLE & 0.134{\tiny$\pm$0.032} & 0.124{\tiny$\pm$0.035} & 0.127{\tiny$\pm$0.034} & 0.138{\tiny$\pm$0.039} & 0.145{\tiny$\pm$0.036} \\
\textbf{MissCNF (ours)} & \underline{0.004}{\tiny$\pm$0.001} & \textbf{0.007}{\tiny$\pm$0.002} & \textbf{0.005}{\tiny$\pm$0.001} & \textbf{0.005}{\tiny$\pm$0.002} & \textbf{0.005}{\tiny$\pm$0.001} \\
\addlinespace
\multicolumn{6}{l}{\textit{MAR, 60\%}} \\
Listwise deletion & 0.437{\tiny$\pm$0.166} & 0.437{\tiny$\pm$0.167} & 0.438{\tiny$\pm$0.166} & 0.436{\tiny$\pm$0.168} & 0.439{\tiny$\pm$0.166} \\
MICE & \textbf{0.002}{\tiny$\pm$0.001} & \underline{0.098}{\tiny$\pm$0.007} & 0.469{\tiny$\pm$0.030} & 0.819{\tiny$\pm$0.054} & 0.865{\tiny$\pm$0.053} \\
KPI & 0.326{\tiny$\pm$0.009} & 0.345{\tiny$\pm$0.012} & \underline{0.388}{\tiny$\pm$0.016} & \underline{0.406}{\tiny$\pm$0.024} & \underline{0.411}{\tiny$\pm$0.031} \\
DiffPuter & 1.212{\tiny$\pm$0.135} & 1.197{\tiny$\pm$0.125} & 1.129{\tiny$\pm$0.257} & 1.114{\tiny$\pm$0.194} & 1.085{\tiny$\pm$0.116} \\
O-MIRACLE & 0.920{\tiny$\pm$0.148} & 0.863{\tiny$\pm$0.156} & 0.896{\tiny$\pm$0.169} & 0.946{\tiny$\pm$0.194} & 1.020{\tiny$\pm$0.178} \\
\textbf{MissCNF (ours)} & \underline{0.005}{\tiny$\pm$0.002} & \textbf{0.009}{\tiny$\pm$0.002} & \textbf{0.018}{\tiny$\pm$0.015} & \textbf{0.018}{\tiny$\pm$0.011} & \textbf{0.026}{\tiny$\pm$0.022} \\
\addlinespace
\bottomrule
\end{tabular}
\caption{KL vs.\ nonlinearity strength $\alpha$ on the $\text{Chain}_\alpha$ graph, both missingness mechanisms and two missing rates. Values are mean $\pm$ std over seeds. Lower is better. Bold = best, underline = second-best, within each row-block (mechanism $\times$ rate).}
\label{tab:alpha_chain_alpha_kl_distance_mean}
\end{table}

\begin{table}[H]
\centering
\scriptsize
\setlength{\tabcolsep}{4pt}
\begin{tabular}{lccccc}
\toprule
Method & $\alpha=0.00$ & $\alpha=0.25$ & $\alpha=0.50$ & $\alpha=0.75$ & $\alpha=1.00$ \\
\midrule
\multicolumn{6}{l}{\textit{MCAR, 30\%}} \\
Listwise deletion & 0.027{\tiny$\pm$0.006} & 0.027{\tiny$\pm$0.008} & \underline{0.025}{\tiny$\pm$0.011} & \underline{0.022}{\tiny$\pm$0.009} & \underline{0.023}{\tiny$\pm$0.006} \\
MICE & \textbf{0.021}{\tiny$\pm$0.006} & \underline{0.020}{\tiny$\pm$0.003} & 0.030{\tiny$\pm$0.007} & 0.041{\tiny$\pm$0.006} & 0.048{\tiny$\pm$0.007} \\
KPI & 0.086{\tiny$\pm$0.016} & 0.083{\tiny$\pm$0.011} & 0.064{\tiny$\pm$0.006} & 0.038{\tiny$\pm$0.008} & 0.029{\tiny$\pm$0.008} \\
DiffPuter & 0.039{\tiny$\pm$0.013} & 0.038{\tiny$\pm$0.011} & 0.042{\tiny$\pm$0.015} & 0.038{\tiny$\pm$0.009} & 0.037{\tiny$\pm$0.006} \\
O-MIRACLE & \underline{0.021}{\tiny$\pm$0.005} & \textbf{0.020}{\tiny$\pm$0.005} & \textbf{0.019}{\tiny$\pm$0.005} & \textbf{0.018}{\tiny$\pm$0.005} & \textbf{0.019}{\tiny$\pm$0.003} \\
\textbf{MissCNF (ours)} & 0.030{\tiny$\pm$0.007} & 0.028{\tiny$\pm$0.007} & 0.026{\tiny$\pm$0.010} & 0.027{\tiny$\pm$0.008} & 0.027{\tiny$\pm$0.006} \\
\addlinespace
\multicolumn{6}{l}{\textit{MCAR, 60\%}} \\
Listwise deletion & 0.044{\tiny$\pm$0.002} & 0.045{\tiny$\pm$0.001} & 0.047{\tiny$\pm$0.003} & 0.050{\tiny$\pm$0.007} & 0.049{\tiny$\pm$0.006} \\
MICE & \underline{0.023}{\tiny$\pm$0.004} & \underline{0.029}{\tiny$\pm$0.007} & 0.046{\tiny$\pm$0.004} & 0.061{\tiny$\pm$0.003} & 0.070{\tiny$\pm$0.002} \\
KPI & 0.109{\tiny$\pm$0.024} & 0.105{\tiny$\pm$0.021} & 0.080{\tiny$\pm$0.015} & 0.045{\tiny$\pm$0.013} & 0.036{\tiny$\pm$0.008} \\
DiffPuter & 0.050{\tiny$\pm$0.018} & 0.057{\tiny$\pm$0.028} & 0.063{\tiny$\pm$0.011} & 0.071{\tiny$\pm$0.022} & 0.063{\tiny$\pm$0.007} \\
O-MIRACLE & \textbf{0.017}{\tiny$\pm$0.001} & \textbf{0.019}{\tiny$\pm$0.001} & \textbf{0.019}{\tiny$\pm$0.004} & \textbf{0.018}{\tiny$\pm$0.003} & \textbf{0.018}{\tiny$\pm$0.004} \\
\textbf{MissCNF (ours)} & 0.034{\tiny$\pm$0.005} & 0.035{\tiny$\pm$0.005} & \underline{0.029}{\tiny$\pm$0.008} & \underline{0.029}{\tiny$\pm$0.006} & \underline{0.026}{\tiny$\pm$0.006} \\
\addlinespace

\multicolumn{6}{l}{\textit{MAR, 30\%}} \\
Listwise deletion & \underline{0.021}{\tiny$\pm$0.007} & \textbf{0.020}{\tiny$\pm$0.004} & \underline{0.020}{\tiny$\pm$0.002} & \underline{0.020}{\tiny$\pm$0.004} & \underline{0.020}{\tiny$\pm$0.005} \\
MICE & 0.026{\tiny$\pm$0.005} & 0.023{\tiny$\pm$0.010} & 0.037{\tiny$\pm$0.010} & 0.055{\tiny$\pm$0.011} & 0.053{\tiny$\pm$0.013} \\
KPI & 0.084{\tiny$\pm$0.014} & 0.075{\tiny$\pm$0.013} & 0.055{\tiny$\pm$0.015} & 0.035{\tiny$\pm$0.010} & 0.033{\tiny$\pm$0.010} \\
DiffPuter & 0.043{\tiny$\pm$0.003} & 0.048{\tiny$\pm$0.019} & 0.071{\tiny$\pm$0.040} & 0.069{\tiny$\pm$0.022} & 0.067{\tiny$\pm$0.021} \\
O-MIRACLE & \textbf{0.019}{\tiny$\pm$0.004} & \underline{0.020}{\tiny$\pm$0.007} & 0.021{\tiny$\pm$0.008} & 0.021{\tiny$\pm$0.008} & 0.020{\tiny$\pm$0.007} \\
\textbf{MissCNF (ours)} & 0.022{\tiny$\pm$0.006} & 0.024{\tiny$\pm$0.009} & \textbf{0.017}{\tiny$\pm$0.003} & \textbf{0.016}{\tiny$\pm$0.002} & \textbf{0.016}{\tiny$\pm$0.001} \\
\addlinespace
\multicolumn{6}{l}{\textit{MAR, 60\%}} \\
Listwise deletion & 0.041{\tiny$\pm$0.007} & 0.040{\tiny$\pm$0.009} & 0.038{\tiny$\pm$0.009} & 0.037{\tiny$\pm$0.005} & 0.032{\tiny$\pm$0.007} \\
MICE & \textbf{0.016}{\tiny$\pm$0.006} & 0.040{\tiny$\pm$0.002} & 0.097{\tiny$\pm$0.014} & 0.137{\tiny$\pm$0.018} & 0.139{\tiny$\pm$0.014} \\
KPI & 0.110{\tiny$\pm$0.010} & 0.102{\tiny$\pm$0.013} & 0.074{\tiny$\pm$0.011} & 0.046{\tiny$\pm$0.013} & 0.046{\tiny$\pm$0.006} \\
DiffPuter & 0.094{\tiny$\pm$0.047} & 0.102{\tiny$\pm$0.069} & 0.113{\tiny$\pm$0.069} & 0.094{\tiny$\pm$0.067} & 0.110{\tiny$\pm$0.057} \\
O-MIRACLE & \underline{0.020}{\tiny$\pm$0.008} & \textbf{0.022}{\tiny$\pm$0.010} & \textbf{0.022}{\tiny$\pm$0.008} & \textbf{0.021}{\tiny$\pm$0.005} & \textbf{0.022}{\tiny$\pm$0.006} \\
\textbf{MissCNF (ours)} & 0.027{\tiny$\pm$0.012} & \underline{0.029}{\tiny$\pm$0.016} & \underline{0.032}{\tiny$\pm$0.008} & \underline{0.028}{\tiny$\pm$0.006} & \underline{0.025}{\tiny$\pm$0.007} \\
\addlinespace
\bottomrule
\end{tabular}
\caption{RMSE$_{\text{ATE}}$ vs.\ nonlinearity strength $\alpha$ on the $\text{Chain}_\alpha$ graph, both missingness mechanisms and two missing rates. Values are mean $\pm$ std over seeds. Lower is better. Bold = best, underline = second-best, within each row-block (mechanism $\times$ rate).}
\label{tab:alpha_chain_alpha_rmse_ate_mean_mean}
\end{table}

\begin{table}[H]
\centering
\scriptsize
\setlength{\tabcolsep}{4pt}
\begin{tabular}{lccccc}
\toprule
Method & $\alpha=0.00$ & $\alpha=0.25$ & $\alpha=0.50$ & $\alpha=0.75$ & $\alpha=1.00$ \\
\midrule
\multicolumn{6}{l}{\textit{MCAR, 30\%}} \\
Listwise deletion & 0.024{\tiny$\pm$0.002} & \textbf{0.026}{\tiny$\pm$0.002} & \textbf{0.027}{\tiny$\pm$0.004} & \textbf{0.027}{\tiny$\pm$0.004} & \textbf{0.026}{\tiny$\pm$0.003} \\
MICE & \textbf{0.020}{\tiny$\pm$0.003} & 0.041{\tiny$\pm$0.001} & 0.085{\tiny$\pm$0.002} & 0.119{\tiny$\pm$0.002} & 0.124{\tiny$\pm$0.002} \\
KPI & 0.085{\tiny$\pm$0.009} & 0.083{\tiny$\pm$0.005} & 0.083{\tiny$\pm$0.002} & 0.080{\tiny$\pm$0.003} & 0.077{\tiny$\pm$0.006} \\
DiffPuter & 0.041{\tiny$\pm$0.010} & 0.042{\tiny$\pm$0.007} & 0.046{\tiny$\pm$0.003} & 0.042{\tiny$\pm$0.007} & 0.041{\tiny$\pm$0.006} \\
O-MIRACLE & \underline{0.023}{\tiny$\pm$0.001} & 0.029{\tiny$\pm$0.002} & 0.038{\tiny$\pm$0.003} & 0.033{\tiny$\pm$0.002} & 0.034{\tiny$\pm$0.006} \\
\textbf{MissCNF (ours)} & 0.027{\tiny$\pm$0.005} & \underline{0.027}{\tiny$\pm$0.003} & \underline{0.028}{\tiny$\pm$0.006} & \underline{0.029}{\tiny$\pm$0.005} & \underline{0.030}{\tiny$\pm$0.004} \\
\addlinespace
\multicolumn{6}{l}{\textit{MCAR, 60\%}} \\
Listwise deletion & 0.044{\tiny$\pm$0.004} & 0.047{\tiny$\pm$0.002} & \underline{0.052}{\tiny$\pm$0.002} & \underline{0.052}{\tiny$\pm$0.006} & \underline{0.053}{\tiny$\pm$0.004} \\
MICE & \textbf{0.022}{\tiny$\pm$0.002} & 0.069{\tiny$\pm$0.003} & 0.148{\tiny$\pm$0.005} & 0.203{\tiny$\pm$0.005} & 0.216{\tiny$\pm$0.005} \\
KPI & 0.125{\tiny$\pm$0.015} & 0.121{\tiny$\pm$0.012} & 0.122{\tiny$\pm$0.009} & 0.120{\tiny$\pm$0.012} & 0.116{\tiny$\pm$0.012} \\
DiffPuter & 0.081{\tiny$\pm$0.007} & 0.089{\tiny$\pm$0.007} & 0.097{\tiny$\pm$0.006} & 0.097{\tiny$\pm$0.018} & 0.084{\tiny$\pm$0.006} \\
O-MIRACLE & \underline{0.023}{\tiny$\pm$0.005} & \underline{0.039}{\tiny$\pm$0.006} & 0.058{\tiny$\pm$0.003} & 0.054{\tiny$\pm$0.004} & 0.053{\tiny$\pm$0.008} \\
\textbf{MissCNF (ours)} & 0.035{\tiny$\pm$0.008} & \textbf{0.037}{\tiny$\pm$0.007} & \textbf{0.037}{\tiny$\pm$0.009} & \textbf{0.037}{\tiny$\pm$0.005} & \textbf{0.034}{\tiny$\pm$0.007} \\
\addlinespace

\multicolumn{6}{l}{\textit{MAR, 30\%}} \\
Listwise deletion & 0.023{\tiny$\pm$0.002} & \textbf{0.024}{\tiny$\pm$0.002} & \underline{0.025}{\tiny$\pm$0.003} & \underline{0.024}{\tiny$\pm$0.004} & \underline{0.025}{\tiny$\pm$0.003} \\
MICE & \textbf{0.022}{\tiny$\pm$0.002} & 0.045{\tiny$\pm$0.004} & 0.091{\tiny$\pm$0.007} & 0.125{\tiny$\pm$0.007} & 0.130{\tiny$\pm$0.007} \\
KPI & 0.088{\tiny$\pm$0.008} & 0.088{\tiny$\pm$0.007} & 0.089{\tiny$\pm$0.010} & 0.082{\tiny$\pm$0.009} & 0.080{\tiny$\pm$0.007} \\
DiffPuter & 0.061{\tiny$\pm$0.015} & 0.063{\tiny$\pm$0.020} & 0.084{\tiny$\pm$0.031} & 0.082{\tiny$\pm$0.022} & 0.076{\tiny$\pm$0.023} \\
O-MIRACLE & 0.050{\tiny$\pm$0.007} & 0.046{\tiny$\pm$0.015} & 0.045{\tiny$\pm$0.022} & 0.048{\tiny$\pm$0.021} & 0.053{\tiny$\pm$0.013} \\
\textbf{MissCNF (ours)} & \underline{0.022}{\tiny$\pm$0.004} & \underline{0.025}{\tiny$\pm$0.003} & \textbf{0.021}{\tiny$\pm$0.002} & \textbf{0.022}{\tiny$\pm$0.002} & \textbf{0.021}{\tiny$\pm$0.002} \\
\addlinespace
\multicolumn{6}{l}{\textit{MAR, 60\%}} \\
Listwise deletion & 0.065{\tiny$\pm$0.023} & \underline{0.066}{\tiny$\pm$0.025} & \underline{0.068}{\tiny$\pm$0.028} & \underline{0.069}{\tiny$\pm$0.030} & \underline{0.066}{\tiny$\pm$0.026} \\
MICE & \textbf{0.016}{\tiny$\pm$0.001} & 0.080{\tiny$\pm$0.002} & 0.179{\tiny$\pm$0.008} & 0.248{\tiny$\pm$0.011} & 0.261{\tiny$\pm$0.009} \\
KPI & 0.126{\tiny$\pm$0.009} & 0.123{\tiny$\pm$0.009} & 0.122{\tiny$\pm$0.010} & 0.123{\tiny$\pm$0.011} & 0.123{\tiny$\pm$0.008} \\
DiffPuter & 0.182{\tiny$\pm$0.041} & 0.190{\tiny$\pm$0.058} & 0.202{\tiny$\pm$0.050} & 0.175{\tiny$\pm$0.048} & 0.170{\tiny$\pm$0.032} \\
O-MIRACLE & 0.106{\tiny$\pm$0.007} & 0.086{\tiny$\pm$0.015} & 0.076{\tiny$\pm$0.027} & 0.088{\tiny$\pm$0.024} & 0.107{\tiny$\pm$0.016} \\
\textbf{MissCNF (ours)} & \underline{0.028}{\tiny$\pm$0.004} & \textbf{0.029}{\tiny$\pm$0.006} & \textbf{0.036}{\tiny$\pm$0.005} & \textbf{0.035}{\tiny$\pm$0.004} & \textbf{0.038}{\tiny$\pm$0.012} \\
\addlinespace
\bottomrule
\end{tabular}
\caption{RMSE$_{\text{CF}}$ vs.\ nonlinearity strength $\alpha$ on the $\text{Chain}_\alpha$ graph, both missingness mechanisms and two missing rates. Values are mean $\pm$ std over seeds. Lower is better. Bold = best, underline = second-best, within each row-block (mechanism $\times$ rate).}
\label{tab:alpha_chain_alpha_rmse_cf_mean_mean}
\end{table}

\begin{table}[H]
\centering
\scriptsize
\setlength{\tabcolsep}{4pt}
\begin{tabular}{lccccc}
\toprule
Method & $\alpha=0.00$ & $\alpha=0.25$ & $\alpha=0.50$ & $\alpha=0.75$ & $\alpha=1.00$ \\
\midrule
\multicolumn{6}{l}{\textit{MCAR, 30\%}} \\
Listwise deletion & \underline{0.004}{\tiny$\pm$0.004} & \textbf{0.007}{\tiny$\pm$0.003} & \textbf{0.009}{\tiny$\pm$0.002} & \textbf{0.010}{\tiny$\pm$0.001} & \textbf{0.008}{\tiny$\pm$0.002} \\
MICE & \textbf{0.003}{\tiny$\pm$0.003} & 0.029{\tiny$\pm$0.004} & 0.168{\tiny$\pm$0.010} & 0.349{\tiny$\pm$0.014} & 0.398{\tiny$\pm$0.016} \\
KPI & 0.055{\tiny$\pm$0.010} & 0.056{\tiny$\pm$0.012} & 0.059{\tiny$\pm$0.010} & 0.061{\tiny$\pm$0.007} & 0.065{\tiny$\pm$0.010} \\
DiffPuter & 0.057{\tiny$\pm$0.007} & 0.060{\tiny$\pm$0.011} & 0.068{\tiny$\pm$0.008} & 0.068{\tiny$\pm$0.010} & 837.056{\tiny$\pm$1871.559} \\
O-MIRACLE & 0.071{\tiny$\pm$0.010} & 0.073{\tiny$\pm$0.011} & 0.075{\tiny$\pm$0.010} & 0.079{\tiny$\pm$0.011} & 0.079{\tiny$\pm$0.008} \\
\textbf{MissCNF (ours)} & 0.004{\tiny$\pm$0.003} & \underline{0.007}{\tiny$\pm$0.004} & \underline{0.011}{\tiny$\pm$0.007} & \underline{0.010}{\tiny$\pm$0.003} & \underline{0.009}{\tiny$\pm$0.001} \\
\addlinespace
\multicolumn{6}{l}{\textit{MCAR, 60\%}} \\
Listwise deletion & \underline{0.007}{\tiny$\pm$0.003} & \underline{0.010}{\tiny$\pm$0.004} & \underline{0.013}{\tiny$\pm$0.003} & \textbf{0.015}{\tiny$\pm$0.004} & \underline{0.016}{\tiny$\pm$0.003} \\
MICE & \textbf{0.005}{\tiny$\pm$0.003} & 0.085{\tiny$\pm$0.009} & 0.457{\tiny$\pm$0.020} & 0.875{\tiny$\pm$0.031} & 0.983{\tiny$\pm$0.032} \\
KPI & 0.321{\tiny$\pm$0.025} & 0.328{\tiny$\pm$0.035} & 0.330{\tiny$\pm$0.031} & 0.340{\tiny$\pm$0.050} & 0.334{\tiny$\pm$0.041} \\
DiffPuter & 0.399{\tiny$\pm$0.032} & 0.402{\tiny$\pm$0.028} & 0.414{\tiny$\pm$0.053} & 0.418{\tiny$\pm$0.039} & 0.417{\tiny$\pm$0.037} \\
O-MIRACLE & 0.472{\tiny$\pm$0.043} & 0.475{\tiny$\pm$0.048} & 0.481{\tiny$\pm$0.052} & 0.489{\tiny$\pm$0.049} & 0.485{\tiny$\pm$0.055} \\
\textbf{MissCNF (ours)} & 0.007{\tiny$\pm$0.005} & \textbf{0.010}{\tiny$\pm$0.005} & \textbf{0.013}{\tiny$\pm$0.005} & \underline{0.015}{\tiny$\pm$0.004} & \textbf{0.013}{\tiny$\pm$0.005} \\
\addlinespace

\multicolumn{6}{l}{\textit{MAR, 30\%}} \\
Listwise deletion & 0.068{\tiny$\pm$0.038} & \underline{0.071}{\tiny$\pm$0.036} & \underline{0.077}{\tiny$\pm$0.039} & \underline{0.078}{\tiny$\pm$0.038} & \underline{0.080}{\tiny$\pm$0.040} \\
MICE & \textbf{0.005}{\tiny$\pm$0.004} & 0.076{\tiny$\pm$0.021} & 0.334{\tiny$\pm$0.077} & 0.617{\tiny$\pm$0.117} & 0.698{\tiny$\pm$0.142} \\
KPI & 0.118{\tiny$\pm$0.049} & 0.130{\tiny$\pm$0.054} & 0.159{\tiny$\pm$0.067} & 0.210{\tiny$\pm$0.124} & 0.176{\tiny$\pm$0.083} \\
DiffPuter & 0.198{\tiny$\pm$0.134} & 0.233{\tiny$\pm$0.162} & 0.268{\tiny$\pm$0.160} & 0.234{\tiny$\pm$0.123} & 0.226{\tiny$\pm$0.136} \\
O-MIRACLE & 0.151{\tiny$\pm$0.059} & 0.149{\tiny$\pm$0.058} & 0.153{\tiny$\pm$0.057} & 0.156{\tiny$\pm$0.053} & 0.155{\tiny$\pm$0.053} \\
\textbf{MissCNF (ours)} & \underline{0.005}{\tiny$\pm$0.003} & \textbf{0.010}{\tiny$\pm$0.004} & \textbf{0.010}{\tiny$\pm$0.002} & \textbf{0.013}{\tiny$\pm$0.005} & \textbf{0.010}{\tiny$\pm$0.004} \\
\addlinespace
\multicolumn{6}{l}{\textit{MAR, 60\%}} \\
Listwise deletion & 0.258{\tiny$\pm$0.130} & 0.262{\tiny$\pm$0.131} & \underline{0.261}{\tiny$\pm$0.140} & \underline{0.264}{\tiny$\pm$0.137} & \underline{0.268}{\tiny$\pm$0.133} \\
MICE & \textbf{0.005}{\tiny$\pm$0.002} & \underline{0.175}{\tiny$\pm$0.058} & 0.787{\tiny$\pm$0.203} & 1.383{\tiny$\pm$0.317} & 1.538{\tiny$\pm$0.337} \\
KPI & 0.810{\tiny$\pm$0.456} & 1.080{\tiny$\pm$0.945} & 1.305{\tiny$\pm$1.366} & 1.414{\tiny$\pm$1.666} & 1.657{\tiny$\pm$2.176} \\
DiffPuter & 1.725{\tiny$\pm$1.762} & 1.329{\tiny$\pm$0.788} & 1.548{\tiny$\pm$1.276} & 2.516{\tiny$\pm$3.531} & 1.015{\tiny$\pm$0.361} \\
O-MIRACLE & 1.680{\tiny$\pm$1.784} & 1.663{\tiny$\pm$1.704} & 1.527{\tiny$\pm$1.370} & 1.434{\tiny$\pm$1.194} & 1.421{\tiny$\pm$1.134} \\
\textbf{MissCNF (ours)} & \underline{0.010}{\tiny$\pm$0.003} & \textbf{0.017}{\tiny$\pm$0.011} & \textbf{0.014}{\tiny$\pm$0.005} & \textbf{0.017}{\tiny$\pm$0.003} & \textbf{0.018}{\tiny$\pm$0.006} \\
\addlinespace
\bottomrule
\end{tabular}
\caption{KL vs.\ nonlinearity strength $\alpha$ on the $\text{Collider}_\alpha$ graph, both missingness mechanisms and two missing rates. Values are mean $\pm$ std over seeds. Lower is better. Bold = best, underline = second-best, within each row-block (mechanism $\times$ rate).}
\label{tab:alpha_collider_alpha_kl_distance_mean}
\end{table}

\begin{table}[H]
\centering
\scriptsize
\setlength{\tabcolsep}{4pt}
\begin{tabular}{lccccc}
\toprule
Method & $\alpha=0.00$ & $\alpha=0.25$ & $\alpha=0.50$ & $\alpha=0.75$ & $\alpha=1.00$ \\
\midrule
\multicolumn{6}{l}{\textit{MCAR, 30\%}} \\
Listwise deletion & 0.017{\tiny$\pm$0.005} & 0.015{\tiny$\pm$0.004} & 0.015{\tiny$\pm$0.003} & 0.015{\tiny$\pm$0.003} & 0.015{\tiny$\pm$0.003} \\
MICE & 0.016{\tiny$\pm$0.003} & 0.015{\tiny$\pm$0.001} & 0.020{\tiny$\pm$0.003} & 0.022{\tiny$\pm$0.005} & 0.026{\tiny$\pm$0.003} \\
KPI & \underline{0.015}{\tiny$\pm$0.005} & 0.014{\tiny$\pm$0.005} & \underline{0.014}{\tiny$\pm$0.005} & 0.016{\tiny$\pm$0.004} & 0.014{\tiny$\pm$0.004} \\
DiffPuter & 0.017{\tiny$\pm$0.005} & 0.015{\tiny$\pm$0.006} & 0.015{\tiny$\pm$0.006} & 0.020{\tiny$\pm$0.009} & 0.531{\tiny$\pm$1.153} \\
O-MIRACLE & \textbf{0.015}{\tiny$\pm$0.003} & \textbf{0.013}{\tiny$\pm$0.003} & \textbf{0.013}{\tiny$\pm$0.003} & \textbf{0.013}{\tiny$\pm$0.004} & \underline{0.013}{\tiny$\pm$0.003} \\
\textbf{MissCNF (ours)} & 0.017{\tiny$\pm$0.003} & \underline{0.013}{\tiny$\pm$0.003} & 0.015{\tiny$\pm$0.005} & \underline{0.013}{\tiny$\pm$0.003} & \textbf{0.013}{\tiny$\pm$0.004} \\
\addlinespace
\multicolumn{6}{l}{\textit{MCAR, 60\%}} \\
Listwise deletion & 0.019{\tiny$\pm$0.004} & 0.018{\tiny$\pm$0.005} & 0.019{\tiny$\pm$0.006} & 0.020{\tiny$\pm$0.006} & 0.018{\tiny$\pm$0.004} \\
MICE & \textbf{0.015}{\tiny$\pm$0.003} & 0.023{\tiny$\pm$0.007} & 0.042{\tiny$\pm$0.007} & 0.054{\tiny$\pm$0.005} & 0.056{\tiny$\pm$0.004} \\
KPI & 0.018{\tiny$\pm$0.005} & 0.020{\tiny$\pm$0.007} & 0.020{\tiny$\pm$0.006} & 0.020{\tiny$\pm$0.005} & 0.019{\tiny$\pm$0.007} \\
DiffPuter & 0.046{\tiny$\pm$0.011} & 0.036{\tiny$\pm$0.012} & 0.026{\tiny$\pm$0.005} & 0.020{\tiny$\pm$0.008} & \textbf{0.015}{\tiny$\pm$0.007} \\
O-MIRACLE & \underline{0.015}{\tiny$\pm$0.006} & \textbf{0.016}{\tiny$\pm$0.007} & \textbf{0.015}{\tiny$\pm$0.007} & \underline{0.017}{\tiny$\pm$0.007} & \underline{0.015}{\tiny$\pm$0.005} \\
\textbf{MissCNF (ours)} & 0.019{\tiny$\pm$0.005} & \underline{0.018}{\tiny$\pm$0.006} & \underline{0.018}{\tiny$\pm$0.006} & \textbf{0.016}{\tiny$\pm$0.004} & 0.015{\tiny$\pm$0.005} \\
\addlinespace

\multicolumn{6}{l}{\textit{MAR, 30\%}} \\
Listwise deletion & 0.015{\tiny$\pm$0.003} & \textbf{0.012}{\tiny$\pm$0.002} & 0.015{\tiny$\pm$0.005} & 0.016{\tiny$\pm$0.006} & 0.016{\tiny$\pm$0.006} \\
MICE & 0.016{\tiny$\pm$0.003} & 0.022{\tiny$\pm$0.005} & 0.032{\tiny$\pm$0.003} & 0.038{\tiny$\pm$0.006} & 0.038{\tiny$\pm$0.009} \\
KPI & \underline{0.014}{\tiny$\pm$0.003} & 0.014{\tiny$\pm$0.003} & 0.015{\tiny$\pm$0.005} & 0.015{\tiny$\pm$0.005} & 0.014{\tiny$\pm$0.005} \\
DiffPuter & 0.028{\tiny$\pm$0.020} & 0.030{\tiny$\pm$0.018} & 0.027{\tiny$\pm$0.011} & 0.018{\tiny$\pm$0.007} & 0.018{\tiny$\pm$0.008} \\
O-MIRACLE & \textbf{0.014}{\tiny$\pm$0.002} & \underline{0.013}{\tiny$\pm$0.001} & \underline{0.013}{\tiny$\pm$0.003} & \underline{0.013}{\tiny$\pm$0.003} & \underline{0.013}{\tiny$\pm$0.003} \\
\textbf{MissCNF (ours)} & 0.016{\tiny$\pm$0.003} & 0.013{\tiny$\pm$0.004} & \textbf{0.012}{\tiny$\pm$0.003} & \textbf{0.013}{\tiny$\pm$0.001} & \textbf{0.012}{\tiny$\pm$0.002} \\
\addlinespace
\multicolumn{6}{l}{\textit{MAR, 60\%}} \\
Listwise deletion & \underline{0.018}{\tiny$\pm$0.010} & \underline{0.016}{\tiny$\pm$0.008} & 0.021{\tiny$\pm$0.013} & 0.020{\tiny$\pm$0.012} & 0.021{\tiny$\pm$0.012} \\
MICE & \textbf{0.016}{\tiny$\pm$0.004} & 0.047{\tiny$\pm$0.010} & 0.087{\tiny$\pm$0.027} & 0.115{\tiny$\pm$0.037} & 0.116{\tiny$\pm$0.040} \\
KPI & 0.023{\tiny$\pm$0.006} & 0.022{\tiny$\pm$0.007} & \underline{0.019}{\tiny$\pm$0.007} & \underline{0.018}{\tiny$\pm$0.009} & \underline{0.019}{\tiny$\pm$0.011} \\
DiffPuter & 0.105{\tiny$\pm$0.051} & 0.096{\tiny$\pm$0.046} & 0.072{\tiny$\pm$0.035} & 0.030{\tiny$\pm$0.019} & 0.022{\tiny$\pm$0.006} \\
O-MIRACLE & 0.024{\tiny$\pm$0.019} & 0.023{\tiny$\pm$0.012} & 0.023{\tiny$\pm$0.009} & 0.021{\tiny$\pm$0.008} & 0.021{\tiny$\pm$0.008} \\
\textbf{MissCNF (ours)} & 0.019{\tiny$\pm$0.003} & \textbf{0.013}{\tiny$\pm$0.002} & \textbf{0.016}{\tiny$\pm$0.006} & \textbf{0.017}{\tiny$\pm$0.007} & \textbf{0.016}{\tiny$\pm$0.006} \\
\addlinespace
\bottomrule
\end{tabular}
\caption{RMSE$_{\text{ATE}}$ vs.\ nonlinearity strength $\alpha$ on the $\text{Collider}_\alpha$ graph, both missingness mechanisms and two missing rates. Values are mean $\pm$ std over seeds. Lower is better. Bold = best, underline = second-best, within each row-block (mechanism $\times$ rate).}
\label{tab:alpha_collider_alpha_rmse_ate_mean_mean}
\end{table}

\begin{table}[H]
\centering
\scriptsize
\setlength{\tabcolsep}{4pt}
\begin{tabular}{lccccc}
\toprule
Method & $\alpha=0.00$ & $\alpha=0.25$ & $\alpha=0.50$ & $\alpha=0.75$ & $\alpha=1.00$ \\
\midrule
\multicolumn{6}{l}{\textit{MCAR, 30\%}} \\
Listwise deletion & \underline{0.020}{\tiny$\pm$0.003} & 0.025{\tiny$\pm$0.003} & 0.028{\tiny$\pm$0.002} & \underline{0.029}{\tiny$\pm$0.002} & \underline{0.029}{\tiny$\pm$0.002} \\
MICE & \textbf{0.019}{\tiny$\pm$0.002} & 0.033{\tiny$\pm$0.004} & 0.064{\tiny$\pm$0.005} & 0.080{\tiny$\pm$0.005} & 0.086{\tiny$\pm$0.003} \\
KPI & 0.026{\tiny$\pm$0.004} & 0.028{\tiny$\pm$0.004} & 0.030{\tiny$\pm$0.004} & 0.031{\tiny$\pm$0.001} & 0.031{\tiny$\pm$0.006} \\
DiffPuter & 0.021{\tiny$\pm$0.003} & \underline{0.024}{\tiny$\pm$0.001} & \underline{0.028}{\tiny$\pm$0.004} & 0.031{\tiny$\pm$0.003} & 0.422{\tiny$\pm$0.873} \\
O-MIRACLE & 0.020{\tiny$\pm$0.002} & \textbf{0.023}{\tiny$\pm$0.003} & \textbf{0.025}{\tiny$\pm$0.002} & \textbf{0.026}{\tiny$\pm$0.003} & \textbf{0.026}{\tiny$\pm$0.003} \\
\textbf{MissCNF (ours)} & 0.020{\tiny$\pm$0.001} & 0.024{\tiny$\pm$0.001} & 0.030{\tiny$\pm$0.004} & 0.029{\tiny$\pm$0.002} & 0.029{\tiny$\pm$0.004} \\
\addlinespace
\multicolumn{6}{l}{\textit{MCAR, 60\%}} \\
Listwise deletion & 0.026{\tiny$\pm$0.003} & 0.031{\tiny$\pm$0.002} & 0.034{\tiny$\pm$0.004} & 0.036{\tiny$\pm$0.005} & 0.036{\tiny$\pm$0.004} \\
MICE & \textbf{0.020}{\tiny$\pm$0.001} & 0.061{\tiny$\pm$0.003} & 0.123{\tiny$\pm$0.007} & 0.161{\tiny$\pm$0.008} & 0.168{\tiny$\pm$0.008} \\
KPI & 0.045{\tiny$\pm$0.004} & 0.049{\tiny$\pm$0.004} & 0.046{\tiny$\pm$0.007} & 0.045{\tiny$\pm$0.004} & 0.044{\tiny$\pm$0.002} \\
DiffPuter & 0.041{\tiny$\pm$0.005} & 0.039{\tiny$\pm$0.007} & 0.037{\tiny$\pm$0.006} & 0.046{\tiny$\pm$0.005} & 0.042{\tiny$\pm$0.006} \\
O-MIRACLE & \underline{0.022}{\tiny$\pm$0.004} & \textbf{0.026}{\tiny$\pm$0.003} & \textbf{0.026}{\tiny$\pm$0.003} & \textbf{0.031}{\tiny$\pm$0.004} & \textbf{0.030}{\tiny$\pm$0.005} \\
\textbf{MissCNF (ours)} & 0.025{\tiny$\pm$0.003} & \underline{0.028}{\tiny$\pm$0.002} & \underline{0.031}{\tiny$\pm$0.004} & \underline{0.032}{\tiny$\pm$0.002} & \underline{0.032}{\tiny$\pm$0.004} \\
\addlinespace

\multicolumn{6}{l}{\textit{MAR, 30\%}} \\
Listwise deletion & \underline{0.021}{\tiny$\pm$0.002} & \textbf{0.024}{\tiny$\pm$0.002} & \underline{0.027}{\tiny$\pm$0.003} & \underline{0.029}{\tiny$\pm$0.003} & \underline{0.031}{\tiny$\pm$0.002} \\
MICE & \textbf{0.020}{\tiny$\pm$0.002} & 0.050{\tiny$\pm$0.006} & 0.088{\tiny$\pm$0.014} & 0.109{\tiny$\pm$0.018} & 0.115{\tiny$\pm$0.019} \\
KPI & 0.045{\tiny$\pm$0.012} & 0.048{\tiny$\pm$0.012} & 0.051{\tiny$\pm$0.012} & 0.052{\tiny$\pm$0.013} & 0.050{\tiny$\pm$0.011} \\
DiffPuter & 0.059{\tiny$\pm$0.026} & 0.062{\tiny$\pm$0.027} & 0.066{\tiny$\pm$0.026} & 0.061{\tiny$\pm$0.019} & 0.060{\tiny$\pm$0.021} \\
O-MIRACLE & 0.044{\tiny$\pm$0.014} & 0.046{\tiny$\pm$0.014} & 0.047{\tiny$\pm$0.013} & 0.048{\tiny$\pm$0.014} & 0.048{\tiny$\pm$0.013} \\
\textbf{MissCNF (ours)} & 0.022{\tiny$\pm$0.003} & \underline{0.025}{\tiny$\pm$0.002} & \textbf{0.027}{\tiny$\pm$0.004} & \textbf{0.028}{\tiny$\pm$0.004} & \textbf{0.030}{\tiny$\pm$0.003} \\
\addlinespace
\multicolumn{6}{l}{\textit{MAR, 60\%}} \\
Listwise deletion & 0.036{\tiny$\pm$0.002} & \underline{0.038}{\tiny$\pm$0.003} & \underline{0.039}{\tiny$\pm$0.006} & \underline{0.039}{\tiny$\pm$0.006} & \underline{0.040}{\tiny$\pm$0.002} \\
MICE & \textbf{0.018}{\tiny$\pm$0.002} & 0.080{\tiny$\pm$0.012} & 0.165{\tiny$\pm$0.027} & 0.212{\tiny$\pm$0.039} & 0.221{\tiny$\pm$0.042} \\
KPI & 0.095{\tiny$\pm$0.028} & 0.099{\tiny$\pm$0.034} & 0.097{\tiny$\pm$0.037} & 0.094{\tiny$\pm$0.038} & 0.096{\tiny$\pm$0.042} \\
DiffPuter & 0.155{\tiny$\pm$0.071} & 0.147{\tiny$\pm$0.053} & 0.139{\tiny$\pm$0.061} & 0.138{\tiny$\pm$0.069} & 0.110{\tiny$\pm$0.027} \\
O-MIRACLE & 0.092{\tiny$\pm$0.050} & 0.095{\tiny$\pm$0.049} & 0.095{\tiny$\pm$0.044} & 0.093{\tiny$\pm$0.041} & 0.093{\tiny$\pm$0.039} \\
\textbf{MissCNF (ours)} & \underline{0.029}{\tiny$\pm$0.006} & \textbf{0.032}{\tiny$\pm$0.005} & \textbf{0.032}{\tiny$\pm$0.005} & \textbf{0.037}{\tiny$\pm$0.004} & \textbf{0.036}{\tiny$\pm$0.003} \\
\addlinespace
\bottomrule
\end{tabular}
\caption{RMSE$_{\text{CF}}$ vs.\ nonlinearity strength $\alpha$ on the $\text{Collider}_\alpha$ graph, both missingness mechanisms and two missing rates. Values are mean $\pm$ std over seeds. Lower is better. Bold = best, underline = second-best, within each row-block (mechanism $\times$ rate).}
\label{tab:alpha_collider_alpha_rmse_cf_mean_mean}
\end{table}

\end{document}